\documentclass[11pt]{article}
\usepackage{geometry}
\usepackage{macros}

\usepackage[numbers]{natbib}

\title{Statistical, Robustness, and Computational\\Guarantees for Sliced Wasserstein Distances}
\author{Sloan Nietert\thanks{Department of Computer Science, Cornell University}, Ritwik Sadhu\thanks{Department of Statistics and Data Science, Cornell University}, Ziv Goldfeld\thanks{School of Electrical and Computer Engineering, Cornell University}, and Kengo Kato\footnotemark[2]}

\begin{document}

\maketitle
\allowdisplaybreaks

\begin{abstract}
Sliced Wasserstein distances preserve properties of classic Wasserstein distances while being more scalable for computation and estimation in high dimensions. %
The goal of this work is to quantify this scalability from three key aspects: (i) empirical convergence rates; (ii) robustness to data contamination; and (iii) efficient computational methods. For empirical convergence, we derive fast rates with explicit dependence of constants on dimension, subject to log-concavity of the population distributions. For robustness, we characterize minimax optimal, dimension-free robust estimation risks, and show an equivalence between robust sliced 1-Wasserstein estimation and robust mean estimation. This enables lifting statistical and algorithmic guarantees available for the latter to the sliced 1-Wasserstein setting. %
Moving on to computational aspects, we analyze the Monte Carlo estimator for the average-sliced distance, demonstrating that larger dimension can result in faster convergence of the numerical integration error. For the max-sliced distance, we focus on a subgradient-based local optimization algorithm that is frequently used in practice, albeit without formal guarantees, and establish an $O(\epsilon^{-4})$ computational complexity bound for it. Our theory is validated by numerical experiments, which altogether provide a comprehensive quantitative account of the scalability question.

\end{abstract}

\section{Introduction}

Sliced Wasserstein distances consider the average or maximum of Wasserstein distances between one-dimensional projections of the two distributions. Formally, for $1\leq p<\infty$, they are defined~as%
\begin{equation}
\SWp(\mu,\nu):= \left [\int_{\unitsph} \Wp^p(\ptheta_{\sharp} \mu, \ptheta_{\sharp} \nu)d\sigma(\theta) \right]^{1/p} \ \text{and} \ \ \ \ \MSWp(\mu,\nu):= \max_{\theta\in\unitsph} \Wp(\ptheta_{\sharp} \mu, \ptheta_{\sharp} \nu),\label{EQ:SW_MSW}    
\end{equation}
where $\ptheta_{\sharp}\mu$ is the pushforward of $\mu$ under the projection $\ptheta: x \mapsto  \theta^\intercal x$ from $\RR^d$ to $\RR$ and $\sigma$ is~the uniform distribution on the unit sphere $\smash{\unitsph}$ in $\smash{\RR^d}$. Sliced Wasserstein distances were introduced in~\cite{rabin2011wasserstein} as a means to mitigate the computational burden of evaluating classic $\Wp$, which rapidly becomes excessive as $d$ grows. Indeed, sliced distances are readily computable using the closed-form expression for $\Wp$ between distributions on $\RR$ (as the $L^p$ norm between quantile functions). Further, $\SWp$ and $\MSWp$ are metrics on $\cP_p(\RR^d)$ and generate the same topology as classic~$\Wp$~\cite{bonnotte2013unidimensional,nadjahi2019asymptotic,bayraktar2021,nadjahi2020statistical}. As such, the sliced distances have been applied to various statistical inference and machine learning tasks, including barycenter computation \cite{rabin2011wasserstein,bonneel_barycenter_2015}, generative modeling \cite{deshpande2018generative,deshpande2019max,nadjahi2019asymptotic,wu2019sliced}, autoencoders~\cite{kolouri2018sliced}, differential privacy \cite{rakotomamonjy2021differentially}, Bayesian computation \cite{nadjahi2020approximate} and topological data analysis~\cite{carriere2017sliced}.

\subsection{Statistical, Robustness, and Computational Aspects of Sliced Distances}

In practice, the sliced Wasserstein distances in \eqref{EQ:SW_MSW} must be approximated from two aspects: (i)~empirically estimate the population measures $\mu$ and $\nu$, and (ii) employ numerical integration or optimization methods to compute the average- or max-sliced distances, respectively. While these approximations are implemented in all but every application of sliced distances, formal guarantees concerning their accuracy are partial or even missing. For the estimation error, the question boils down to quantifying the rate at which $\SWp(\hat \mu_n,\mu)$ and $\MSWp(\hat \mu_n,\mu)$ decay to 0, where $\hat \mu_n$ is the empirical distribution of $n$ independent observations from $\mu$.\footnote{The two-sample setting, which concerns the convergence $\SWp(\hat \mu_n,\hat \nu_n)$ and $\MSWp(\hat \mu_n, \hat \nu_n)$ towards the corresponding distance between the population measures, is also of interest.} These rates are known to adapt to the low-dimensionality of the projected distribution, but previously derived rates do not seem to be sharp \cite{lin2021projection}, rely on high-level assumptions that may be hard to verify in practice \cite{niles2019estimation}, or hide dimension-dependent constants whose characterization is crucial for understanding the scalability of sliced distances \cite{nadjahi2020statistical}. More recently, \cite{manole2022minimax} showed that near-parametric rates (i.e., up to polylogarithmic factors) are achievable for the average-sliced $p$-Wasserstein distance in the two-sample case, under the alternative ($\mu\neq\nu$). %
Limit distributions for sliced Wasserstein distances were studied in \cite{manole2022minimax,goldfeld2022statistical,xu2022central,xi2022distributional}, but these results inherently neglect constants and dependence on dimension.

Concerning robust estimation, while these aspects were studied for classic Wasserstein distances \cite{balaji2020,nath2020,mukherjee2021,khang2021,staerman21,nietert2022robust}, they were not considered under sliced $\Wp$. Improvement in~robustness to outliers due to projection-averaging was demonstrated for the Cram\'er-von Mises statistic in the context of multivariate two-sample testing \cite{kim2020robust}. It therefore stands to reason that  similar gains would emerge for Wasserstein distances, which is especially appealing since robust estimation of classic $\Wp$ in high dimensions is hard. Indeed, \cite{nietert2022robust} showed that when an $\eps$-fraction of data is contaminated, $\Wp$ admits worst-case estimation risk $\sqrt{d}\eps^{1/p-1/2}$ over distributions with bounded covariance. Consequently, obtaining accurate estimates of $\Wp$ from contaminated data is infeasible in high dimensions when $\eps = \Omega(1)$, which further motivates exploring robustness under~slicing.

From the computational standpoint, the average-sliced distance $\SWp$ is typically computed using Monte Carlo (MC) integration \cite{kolouri2019generalized,nadjahi2020statistical}. The accuracy of this approach strongly depends on the variance of the function $\theta\mapsto \Wp^p(\ptheta_{\sharp}\mu,\ptheta_{\sharp} \nu)$ when $\theta$ is uniformly distributed on $\smash{\unitsph}$, which may scale badly with $d$. A bound on the MC integration error in terms of this variance was provided in \cite{nadjahi2020statistical} but without further analysis to control it by basic properties of the population distribution or characterize its dependence on $d$. Accordingly, the accuracy of the MC-based approach for computing $\SWp$ stands unresolved. Recently, \cite{nadjahi2021fast} used the conditional central limit theorem \cite{reeves2017conditional} to derive a Gaussian~approximation of $\SWtwo$ that can be computed in closed form. The accuracy of this approximation may improve as $d\to\infty$, contingent on certain weak dependence assumptions on the data distribution. A popular approach for computing the max-sliced distance is the heuristic alternating optimization procedure from \cite{kolouri2019generalized,deshpande2019max}, which, however, lacks formal convergence guarantees. More recently, computational aspects of the so-called ``projection-robust'' Wasserstein distance, which considers projections to $k$-dimensional subspaces, were explored in \cite{paty2019subspace,lin2020projection,huang2021riemannian}.\footnote{Despite the name ``projection-robust'', these works do not explore robust estimation.} As maximization of projected distance is a non-convex and non-smooth optimization problems, these works considered convex relaxations \cite{paty2019subspace} or entropic regularization \cite{lin2020projection,huang2021riemannian} to prove approximate convergence to a stationary point.

\subsection{Contributions}

The goal of this paper is to close the aforementioned gaps %
by (i) deriving fast empirical convergence rates for sliced distances with explicit dimension dependence; (ii) characterizing minimax optimal robust estimation rates with improved dependence on dimension; and (iii) providing formal guarantees for frequently used methods for computing both the average- and max-sliced $\Wp$. Focusing on log-concave distributions, we show that both average- and max-sliced empirical distances converge as $n^{-1/\max\{2,p\}}$, which is sharp as it matches lower bounds from \cite{bobkov2019one}. Furthermore, we characterize the constant in terms of $d$ and elementary properties of the population distribution (e.g., mean, moments, covariance matrix). Our derivation leverages the machinery of \cite{bobkov2019one} for analyzing empirical convergence of Wasserstein distances between log-concave measures on $\RR$. To that end, we show that log-concavity is preserved under projections and derive lower bounds on the Cheeger constant of the projected distribution. Our results elucidate scaling rates of $d$ with $n$ for which (high-dimensional) empirical convergence holds true, thereby addressing the scalability of empirical estimates question.

For robustness guarantees, we formalize minimax risk for robust estimation under sliced $\Wp$ with total variation (TV) contamination and prove that $\MSWp$ enjoys a dimension-free risk of $\eps^{1/p-1/q}$ when clean distributions have bounded $q$th moments for $q > p$ and the corruption level is at most $\eps$. $\SWp$ admits a strictly smaller risk which scales at the same rate when $q = O(1)$. In contrast, the comparable risk for classic $\Wp$ in this setting acquires an extra $\sqrt{d}$ factor. Using the framework of generalized resilience \cite{zho2019resilience}, we extend these guarantees to the finite-sample setting with adversarial corruptions, obtaining matching rates up to an added empirical approximation term. Furthermore, when $p=1$, we prove an exact equivalence between standard mean resilience \cite{steinhardt2018resilience} and resilience w.r.t.\ $\MSWone$, allowing one to lift many statistical and algorithmic guarantees for robust mean estimation to the sliced $\Wone$ setting.

Lastly, we provide formal guarantees for popular methods for computing $\SWp$ and $\MSWp$, which were until now lacking. %
Our analysis relies on showing that $\smash{w_p:\theta\mapsto\mathsf{W}_p ( \proj^{\theta}_\sharp \mu, \proj^{\theta}_\sharp \nu)}$ and its $p$th power are Lipschitz continuous on $\smash{\unitsph}$ and deriving sharp bounds on their Lipschitz constants. %
Having that, we analyze the MC estimator for the average-sliced distance, and use concentration of Lipschitz functions on the unit sphere to bound the variance of $w_p^p$. The obtained bound reveals that higher dimension can in fact shrink the MC error when the covariance matrices have bounded operator norms. We numerically verify this surprising observation on synthetic examples. %

For the max-sliced distance, we analyze the heuristic algorithm from \cite{deshpande2019max,kolouri2019generalized}, which utilizes alternating subgradient-based optimization. We observe that in addition to being Lipschitz continuous, the optimization objective for $p=2$ is weakly convex with easily computable gradients. This lets us cast the algorithm from \cite{deshpande2019max,kolouri2019generalized} under the proximal stochastic subgradient optimization framework of \cite{davis2018stochastic}, from which we obtain local solutions for $w_2(\theta)$ with $O(\epsilon^{-4})$ computational complexity. %
An empirical comparison with the more advanced approaches of \cite{lin2020projection,huang2021riemannian} for computing the projection-robust Wasserstein distance (with $k=1$ to match the sliced framework) based on Riemannian optimization reveals that our subgradient-based method is significantly faster in terms of iteration complexity and computation time.  %
We also consider global optimization by showing that $\MSWp$ computation matches the framework of~\cite{malherbe2017global} for Lipschitz function optimization over convex domains. Adapting their LIPO algorithm to our problem, we obtain a provably consistent algorithm for computing $\MSWp$. However, the number of function evaluations that LIPO requires grows exponentially with dimension, which renders the locally optimal subgradient method preferable when dimension is large.

\section{Background and Preliminaries}
\label{sec: background}

\paragraph{Notation.} We use $\| \cdot \|$ for the Euclidean norm in $\RR^d$. The operator norm for matrices is $\| \cdot \|_{\op}$. The unit sphere in $\RR^d$ is denoted by $\unitsph$, %
while $\BB^d$ is the unit ball. Let $\calP(\R^d)$ denote the space of Borel probability measures on $\RR^d$ equipped with the TV metric $\|\mu - \nu\|_\tv = \frac{1}{2}|\mu - \nu|(\R^d)$, and set $\calP_p(\R^d):= \{ \mu \in \calP(\R^d) : \int \| x \|^p d\mu(x) < \infty \}$ for $1 \le p <\infty$. The support of $\mu \in\calP(\R^d)$ is denoted as $\supp(\mu)$, and we write $\mu \leq \nu$ for setwise inequality. %
For for a  measurable map $f$, the pushforward of $\mu$ under $f$ is denoted as $f_{\sharp}\mu = \mu \circ f^{-1}$, i.e., if $X \sim \mu$ then $f(X) \sim f_{\sharp}\mu$. %
For two numbers $a$ and $b$, we use the notation $a \wedge b = \min \{ a,b \}$ and $a \vee b = \max \{a,b \}$.  The distance between a set $S$ and a point $x$ in a metric $(\cX, d)$ space is defined as $\mathrm{dist}(x, S):= \inf_{y \in S} d(x, y) $.%

Some of our results assume log-concavity of the population distribution. A probability measure $\mu \in \calP(\R^d)$ is \textit{log-concave} if for every nonempty compact sets $A,B \subset \R^d$ and $\lambda \in [0,1]$, we have $\mu\big(\lambda A +(1-\lambda)B\big) \ge \mu(A)^\lambda \mu(B)^{1-\lambda}$. A probability density function $f$ on $\R^d$ is called \textit{log-concave} if for every $x,y \in \R^d$ and $\lambda \in [0,1]$, it satisfies $f\big(\lambda x+(1-\lambda)y\big) \ge f(x)^{\lambda}f(y)^{1-\lambda}$. Any non-degenerate distribution is log-concave if and only if it has a log-concave density \cite[Theorem~1.1]{borell1974}. For $\beta \in (0,2]$, let $\psi_{\beta}(t) = e^{t^\beta}-1$ for $t \ge 0$, and recall that the corresponding Orlicz (quasi-)norm of a real-valued random variable $X$ is defined as $\| X \|_{\psi_{\beta}}:= \inf \{ c>0 : \E[\psi_{\beta}(|X|/c)] \le 1 \}$. A Borel probability measure $\mu\in\cP(\RR^d)$ is called \textit{sub-Gaussian} if $\| \| X \| \|_{\psi_2} < \infty$ for $X \sim \mu$.

\paragraph{Classic and sliced Wasserstein distances.} %
For $1 \le p < \infty$, the $p$-Wasserstein distance between $\mu,\nu \in \calP_p(\R^d)$ is $\Wp(\mu,\nu):= \inf_{\pi \in \Pi(\mu,\nu)} \big [ \int_{\R^d \times \R^d} \|x-y\|^p \, d \pi(x,y) \big]^{1/p}$, where  $\Pi(\mu,\nu)$ is the set of couplings of $\mu$ and $\nu$. $\Wp$ is a metric on $\calP_p(\R^d)$ and metrizes weak convergence plus convergence of $p$th moments. %
While for $d>1$ the  definition of $\Wp$ generally amounts to an infinite-dimensional optimization problem, the expression simplifies when distributions are supported in $\R$. This motivates the notion of the average- and max-sliced Wasserstein distances from \eqref{EQ:SW_MSW}. Both sliced distances are also metrics on $\calP_p(\R^d)$ that induce the same topology as $\Wp$ \cite{bonnotte2013unidimensional,nadjahi2019asymptotic,bayraktar2021,nadjahi2020statistical}.

To present the simple one-dimensional formulae for $\Wp$, for $\mu \in \calP(\R^d)$ and $\theta \in \unitsph$, let 
$F_{\mu}(t;\theta):= \mu \big ( \{ x \in \R^d : \theta^{\intercal} x \le t \} \big )$ be the distribution function of $\ptheta_\sharp \mu$, and %
$F_{\mu}^{-1}(\tau;\theta) = \inf \{ t\in\RR : F_{\mu}(t;\theta) \ge \tau \}$, for $\tau \in (0,1)$, be the quantile function. The $\Wp$ between measures on $\R$ amounts to the $L^p$ distance between their quantile functions:
$\Wp^p(\ptheta_\sharp \mu, \ptheta_\sharp \nu) = \int_{0}^1 \big|F_{\mu}^{-1}(\tau;\theta) - F_{\nu}^{-1}(\tau;\theta)\big|^p d\tau$. For $p=1$, the expression further simplifies to %
$\mathsf{W}_1 (\ptheta_\sharp \mu, \ptheta_\sharp \nu) = \int_{\R} \big|F_{\mu}(t;\theta) - F_{\nu}(t;\theta)\big| \, dt$. 

Sliced Wasserstein distances between empirical distributions can be computed via order statistics. Let $\hat{\mu}_n:=n^{-1}\sum_{i=1}^n\delta_{X_i}$ and $\hat{\nu}_n:=n^{-1}\sum_{i=1}^n\delta_{Y_i}$ be the empirical distributions of samples $X_1,\ldots,X_n$ and $Y_1,\ldots,Y_n$. 
For each $\theta \in \unitsph$, denote $X_i (\theta) = \theta^{\intercal}X_i$, and let $X_{(1)}(\theta) \le \dots \le X_{(n)}(\theta)$ be the order statistics; define $Y_{(1)}(\theta) \le \cdots \le Y_{(n)}(\theta)$ analogously. By Lemma 4.2 in~\cite{bobkov2019one}, we have $\Wp^p(\ptheta_\sharp \hat{\mu}_n, \ptheta_\sharp \hat{\nu}_n) = n^{-1}\sum_{i=1}^n \big|X_{(i)}(\theta) - Y_{(i)}(\theta)\big|^p$. The sliced distances $\SWp$ and $\MSWp$ are computed by integrating or maximizing the above over $\theta\in\unitsph$.

\section{Empirical Convergence Rates}\label{sec: empirical convergence}

We study empirical convergence rates %
of sliced Wasserstein distances for log-concave distributions.  %
The next result gives sharp one-sample rates with explicit dependence on the effective dimension.

\begin{theorem}[Empirical convergence rates]%
\label{thm: SWp rate}
Let $1 \le p < \infty$ and $n \ge 2$. Suppose that $\mu \in \calP (\R^d)$ is log-concave with covariance matrix $\Sigma$ and set $k = \mathrm{rank} (\Sigma)$. Then,
\begin{subequations}
\begin{align}
&\E\big[ \SWp(\hat{\mu}_n,\mu) \big] \lesssim_p \frac{\|\Sigma \|_{\op}^{1/2} \sqrt{(\log n)^{\ind_{\{p=2\}}}}}{n^{1/(2 \vee p)}},\label{eq: swp bound}\\
&\EE\big[\mspace{1mu}\MSWp(\hat{\mu}_n,\mu)\big] \lesssim_{p} \frac{\|\Sigma\|_{\op}^{1/2} k\log n }{n^{1/p}} + \frac{\|\Sigma\|_{\op}^{1/2} \sqrt{k\log n} }{n^{1/(2 \vee p)}} + \frac{\|\Sigma \|_{\op}^{1/2} \sqrt{(\log n)^{\ind_{\{p=2\}}}}}{n^{1/(2 \vee p)}}.\label{eq: mswp bound}
\end{align}
\end{subequations}
\end{theorem}

The proof of Theorem \ref{thm: SWp rate}, in Appendix \ref{APPEN: SWP rate_proof}, employs the machinery of \cite{bobkov2019one} for analyzing empirical convergence of log-concave distributions on $\RR$ based on their Cheeger constant (see Appendix \ref{APPEN:log_conc_Cheeger}). For \eqref{eq: swp bound}, we show that log-concavity is preserved under projections and lower bound the Cheeger constant of the projected distribution by $c/\|\Sigma \|_{\op}$, uniformly in $\theta\in\unitsph$. For \eqref{eq: mswp bound}, concentration and covering arguments enables approximating the expected max-sliced distance by $\sup_{\theta\in\unitsph}\EE\big[\Wp(\ptheta_{\sharp}\hat{\mu}_n,\ptheta_{\sharp}\mu)\big]$, for which the aforementioned (uniform in $\theta$) bounds are applicable. %

\begin{remark}[Lower bounds]
The rate in \eqref{eq: swp bound} is sharp up to log factors over the~log-concave class. Corollary 6.14 in \cite{bobkov2019one} implies that $\EE[\SWp(\empmu, \mu)]^p \geq c_p (\Tr(\Sigma)/d)^{p/2}/ \big(n (\log n)^{p/2}\big )$, for $\mu = \cN(0, \Sigma) \in \cP(\R^d)$ and any $p>2$. For $p = 2$, a similar computation~yields $\EE[\SWtwo(\empmu, \mu)]^2 \gtrsim  (\Tr(\Sigma)/d) \log \log n / n$, while for $p \in [1,2)$, 
$ \EE[\SWp(\empmu, \mu)] \geq \EE[\SWone(\empmu, \mu)] \gtrsim \sqrt{\tau(\Sigma)^2/n}$ where $\tau(\Sigma) = \frac{1}{d}\sum_{i=1}^d \sqrt{\lambda_i(\Sigma)}$ is the average of root eigenvalues of $\Sigma$. Since $\MSWp(\empmu, \mu) \gtrsim \SWp(\empmu, \mu)$, this also yields a lower bound for $\MSWp$, while \cite{niles2019estimation} gives a $\sqrt{d/n}$ lower bound under the $T_p$ inequality.
\end{remark}

\begin{remark}[Comparison with \cite{niles2019estimation,lin2021projection}]
In \cite{niles2019estimation}, empirical rates for  $\MSWp$ were derived under a~high- level $T_{p'}(\sigma^2)$ assumption on $\mu$. Our rate of decay from \eqref{eq: mswp bound} is faster, while replacing their entropy-transport inequality condition with log-concavity. The bounds for $\MSWp$ in \cite[Theorem 3.6]{lin2021projection} assume the projection Poincar\'e inequality and $M_q:= (\mu \|x\|^q)^{1/q} < \infty$ for $q>p$, and matches \eqref{eq: swp bound} as $q \to \infty$ in terms of the dependence on $n$. However, their dependence on $d$ is implicit through $M_q$ which typically grows prohibitively with $q$ and $d$. 
Our log-concavity assumption is strictly stronger than the Poincar\'e inequality, but yields a bound in terms of $\|\Sigma\|_{\op}$ which, for example, is constant in $d$ when~$\Sigma = I_d$.\footnote{A recent preprint \cite{Bartl2022structure}, that was posted on arXiv after this paper was submitted, shows that an improved estimate holds with high probability (compared to the convergence in expectation studied herein) for isotropic log-concave random vectors; cf. Equation (1.13) therein.} Finally, we note that the bound for $\MSWp$ in \eqref{eq: mswp bound} adapts to the effective dimensionality~$k$ of the data, contrasting previously available bounds that depend on the ambient dimension $d$.

\end{remark}

\begin{remark}[Concentration bounds]
Combining the expectation bounds from Theorem~\ref{thm: SWp rate} with \cite[Theorem 3.8]{lin2021projection} yields concentration bounds for empirical sliced distances. These are presented in Appendix \ref{APPEN: SWp concentration} and are later used to derive formal guarantees for computing~$\MSWp$.
\end{remark}

When $p>2$, the rates in Theorem \ref{thm: SWp rate} are slower than parametric. Nevertheless, in the two-sample case with $\mu\neq \nu$, parametric rates are attainable uniformly in $p$ for compactly supported distributions.%

\begin{proposition}[Parametric rates under the alternative]
\label{thm: swp_rate_alt}
Let $1\leq p <\infty$, and suppose that $\mu,\nu$ have compact supports with $\diam\mspace{-3mu}\big(\supp(\mu)\big)\vee\diam\mspace{-3mu}\big(\supp(\nu)\big)\leq R$. Then, 
\[\EE\big [ \big| \SWp^p(\empmu, \empnu) - \SWp^p(\mu, \nu) \big| \big ] \lesssim_{p,R} n^{-1/2} \ \ \ \ \mbox{and}\ \ \ \ \EE\big [ \big| \MSWp^p(\empmu, \empnu) - \MSWp^p(\mu, \nu) \big| \big ] \lesssim_{p,R} dn^{-1/2}.
\]
If further $\mu \neq \nu$, then the same (parametric) rate also holds for empirical $\SWp$ and $\MSWp$. 
\end{proposition}

Proposition~\ref{thm: swp_rate_alt} is proven in Appendix~\ref{APPEN: SWP_alt_rate_proof} using a comparison inequality between $\mathsf{W}_p$ and $\wass$ and elementary bounds for $\wass$ using its integral representation and KR duality.%

\begin{remark}[Comparison to \cite{manole2022minimax}]
Theorem 2 of \cite{manole2022minimax} establishes a bound of $(\log n /n)^{1/2}$ on the two-sample average-sliced Wasserstein distance, but under bounded moment assumptions instead of compact support.
\end{remark}

\section{Robust Estimation}
\label{sec:robust-estimation}

We examine robustness of sliced Wasserstein distances to outliers, showing that slicing enables dimension-free risk bounds that avoid $\mathrm{poly}(d)$ factors present for classic $\Wp$ (cf. \cite{nietert2022robust}). We consider TV corruptions, where an unknown ``clean'' distribution $\mu$ is contaminated to obtain $\tilde{\mu}$ with $\|\mu - \tilde{\mu}\|_\tv \leq \eps$. Upon observing $\tilde{\mu}$, the goal is to return a distribution $T(\tilde{\mu})$ such that the error $\sD\big(T(\tilde{\mu}),\mu\big)$ is small, where $\sD \in \{\SWp,\MSWp\}$. Without further assumptions, this error can be unbounded, so we require that $\mu$ belongs to a family $\cG \subset \cP(\R^d)$ encoding standard moment bounds. We consider the minimax risk for robust estimation under $\sD$ with TV contamination, defined by
\begin{equation*}
    R(\sD,\cG,\eps) = \inf_{T:\cP(\R^d) \to \cP(\R^d)} \sup_{(\mu,\tilde{\mu}) \in \cG \times  \cP(\R^d);~ \|\tilde{\mu} - \mu\|_\tv \leq \eps} \sD\big(T(\tilde{\mu}), \mu\big).
\end{equation*}
Fix $q > p$ and let $\cG_q(\sigma)\mspace{-3mu}:=\mspace{-3mu}\big\{\mu\mspace{-3mu}\in\mspace{-3mu} \cP_q(\RR^d)\mspace{-3mu}:\sup_{\theta\in\unitsph}\mu |\theta^\intercal(x - \mu x)|^q \leq \sigma^q\big\}$ contain all  distributions whose projections have bounded central $q$th moments. In particular, $\cG_2(\sigma) = \{ \mu \in \cP_2(\R^d) : \|\Sigma_\mu\|_\mathrm{op} \leq \sigma \}$. The next theorem characterizes minimax robust estimation risk over this class.

\begin{theorem}[Population-limit robust estimation]
\label{thm:robustness}
Fix $1 \leq p < q$, $\sigma \geq 0$, and $0 \leq \eps \leq 0.49$.\footnote{The upper bound on $\eps$ of 0.49 can be substituted with an any constant bounded away from 1/2.} We have
$R\big(\SWp,\cG_q(\sigma),\eps\big) \asymp \sigma \sqrt{(1 \lor d/q)(1 \land p/d)} \,\eps^{1/p-1/q} \ \ \ \ \mbox{and}\ \ \ \ R\big(\MSWp,\cG_q(\sigma),\eps\big) \asymp \sigma \eps^{1/p-1/q}.$
\end{theorem}

Note that the $\sqrt{(1 \lor d/q)(1 \land p/d)}$ prefactor in the first bound is always less than 1. The proof in Appendix~\ref{prf:robustness} controls the risk via $\sup_{\mu,\nu \in \cG_q(\sigma), \|\mu - \nu\|_\tv \leq \eps} \sD(\mu,\nu)$, a modulus of continuity that captures the sensitivity of $\sD$ to small perturbations that preserve membership to the clean family. We employ techniques based on generalized resilience \cite{zho2019resilience,steinhardt2018resilience} to relate this modulus to similar quantities arising in the robust estimation of $p$th moment tensors, giving the above rates. The procedure that achieves these rates projects the observed contaminated distribution onto the corresponding family of clean distributions in TV norm.

\begin{remark}[Comparison to \cite{nietert2022robust}]
A related framework \cite{nietert2022robust} considers robust estimation of $\Wp$ under input measure contamination. They obtain a rate of $\sigma \sqrt{d} \eps^{1/p-1/2}$ using similar methods under the weaker Huber $\eps$-contamination model when $q=2$ (see Corollary 1 therein). Evidently, slicing eliminates a $\sqrt{d}$ factor from the minimax estimation risk. In Appendix~\ref{prf:robustness}, we interpolate between these regimes, proving that $k$-dimensional sliced distances admit risks bounded by $\sigma \sqrt{1 \lor k/q} \, \eps^{1/p-1/q}$.
\end{remark}

Theorem \ref{thm:robustness} characterizes population-limit robust estimation, i.e., when data is abundant. The next result, proven in Appendix~\ref{prf:finite-sample-robustness}, extends to the finite-sample regime. For a radius $R > 0$, we write $\mu_R$ to denote the distribution of $X \sim \mu$ conditioned on $\|X - \mu x\| \leq R$.

\begin{proposition}[Finite-sample robust estimation]
\label{prop:finite-sample-robustness}
Fix $1 \leq p < q, \sigma \geq 0$, and $0 < \eps \leq 0.49$, and let $\sD \in \{\SWp,\MSWp\}$. Then there exists a radius $R \asymp \sqrt{d/\eps}$ and a procedure which, given $n \geq (R^p + \eps^{-2}) \, d \log (d/\eps)$ samples with at least $(1-\eps)n$ drawn i.i.d.\ from any $\mu \in \cG_q(\sigma)$, returns $\nu \in \cP(\R^d)$ such that $\sD(\nu,\mu) \lesssim R\big(\sD,\cG_q(\sigma),\eps\big) + \E\big[\sD\big((\widehat{\mu}_R)_n,\mu_R\big)\big]$ with probability at least 0.99\footnote{See Appendix \ref{prf:finite-sample-robustness} for precise high-probability bounds and extension to the strong contamination model.}.
\end{proposition}

Evidently, the finite-sample error bound comprises the population-limit robust estimation risk (which is necessary) plus the empirical estimation error associated with the truncated distribution $\mu_R$. The lower bounds on $n$ ensures that the empirical distribution $(\widehat{\mu}_R)_n$ satisfies the same generalized resilience property appearing in the population-limit analysis. %
The truncated empirical convergence term can typically be bounded by the corresponding untruncated version. For example, when $\mu \in \cG_2(\sigma)$ is log-concave and $\sD = \MSWone$, we can bound this term by $O\big(\sigma d \log n / n + \sigma \sqrt{d \log n/ n}\big)$, which follows from \Cref{thm: SWp rate} and the fact that $\mu_R$ is also log-concave with $\|\Sigma_{\mu_R}\|_\mathrm{op} \leq \|\Sigma_{\mu}\|_\mathrm{op} \leq \sigma$ for any $R > 0$.

When $p=1$, we prove in Appendix~\ref{prf:resilience} a precise connection to resilience, a sufficient condition for robust mean estimation, which may be of independent interest.

\begin{proposition}[Connection to mean resilience]\label{prop:resilience}
For $0 \leq \eps < 1$,  $\mu \in \cP_1(\RR^d)$ is $(\rho,\eps)$-resilient, i.e. $\|\mu x\!-\!\nu x\|\!\leq\!\rho$ for all $\nu\!\leq\!\frac{1}{1-\eps}\mu$, if and only if $\MSWone(\mu,\nu)\!\leq\!\Theta(\rho)$ for all $\nu\!\leq\!\frac{1}{1-\eps} \mu$.
\end{proposition}

This suggests borrowing from the existing family of robust mean estimation algorithms, primarily developed for the bounded covariance setting ($q=2$). In \Cref{prf:finite-sample-cov-alg}, we inspect an efficient spectral reweighting procedure and apply it for both $\MSWp$ and $\SWp$ when $1 \leq p < 2$.

\begin{proposition}[Efficient computation via spectral reweighting]
\label{prop:finite-sample-cov-alg}
If $1 \leq p < q = 2$ and $0 \leq \eps \leq 1/12$, the guarantee of \Cref{prop:finite-sample-robustness} is achieved by an $\widetilde{O}(n d^2)$-time spectral reweighting algorithm.
\end{proposition}

\section{Formal Computational Guarantees}

The computational tractability of empirical sliced Wasserstein distances relies on the simplified expressions for $\Wp$ between distribution on $\RR$. However, even then, evaluating $\SWp$ and $\MSWp$ requires computing the average or the maximum of one-dimensional distances over projection directions $\theta\in\unitsph$. This section provides formal guarantees for two such popular computational methods: MC integration for $\SWp$ and alternating subgradient-based optimization for $\MSWp$. Our analysis relies on the observation that $w_p(\theta):=\mathsf{W}_p ( \proj^{\theta}_\sharp \mu, \proj^{\theta}_\sharp \nu)$ and its $p$th power are Lipschitz functions on $\unitsph$. %

\begin{lemma}[Lipschitz continuity]
\label{lem: wp_theta_lipschitz}
The functions $w_p$ and $w_p^p$ are Lipschitz with constants bounded by $L^p_{\mu,\nu} \mspace{-3mu}=\mspace{-6mu} \sup\limits_{\theta \in \unitsph} \mspace{-4mu}\big[ (\mu |\theta^\intercal x|^p)^{1/p} \mspace{-1mu}+\mspace{-1mu} (\nu |\theta^\intercal x|^p)^{1/p} \big ] $ and $M^p_{\mu,\nu}\mspace{-3mu}=\mspace{-1mu} 3p 2^p \mspace{-6mu}\sup\limits_{\theta \in \unitsph}\mspace{-3mu} (\mu |\theta^\intercal x|^{p}\mspace{-1mu} +\mspace{-1mu} \nu |\theta^\intercal x|^{p})$, respectively.
\end{lemma}

Lemma~\ref{lem: wp_theta_lipschitz} (proven in Appendix~\ref{APPEN:wp_theta_lipschitz_proof}) sharpens the Lipschitz constants derived in \cite[Lemma 2]{niles2019estimation}, which correspond to bounding $|\theta^\intercal x|$ by $\|x\|$ in the above expressions. The projected moments $(\mu|\theta^\intercal x|^p)^{1/p}$ typically has a milder dependence on $d$ than $(\mu\|x\|^p)^{1/p}$, which is crucial for the subsequent analysis. %

\subsection{Average-Slicing: Monte Carlo Integration}

The typical approach for computing the integral over the unit sphere in $\SWp$ is MC averaging. Fix $\mu,\nu\in\cP_p(\RR^d)$ and let $\hat{\mu}_n$ and $\hat{\nu}_n$ be the associated empirical measures. Take $\Theta\sim \mathrm{Unif}(\unitsph)$ and consider i.i.d. copies thereof $\Theta_1,\ldots,\Theta_m$. The MC based estimate of $\SWp^p$ is given by
\[
\WMC:=\frac{1}{m}\sum_{j=1}^m \Wp^p(\mathfrak{p}_\sharp^{\Theta_j}\hat{\mu}_n,\mathfrak{p}_\sharp^{\Theta_j}\hat{\nu}_n)=\frac{1}{mn}\sum_{j=1}^m \sum_{i=1}^n \big|X_{(i)}(\Theta_j) - Y_{(i)}(\Theta_j)\big|^p,
\]
where $X_{(1)}(\theta) \le \dots \le X_{(n)}(\theta)$ is the order statistics, which is readily evaluated using sorting algorithms with $O(n\log n)$ average/worst-case complexity (e.g., \texttt{quick\_sort} or \texttt{merge\_sort}).

The next result bounds the effective error of $\WMC$ in approximating the population distance $\SWp^p(\mu,\nu)$.

\begin{proposition}[Monte Carlo error bound]
\label{prop: monte carlo}
Let $1\leq p<\infty$, and assume $\mu,\nu\in\cP_p(\RR^d)$ are log-concave with covariance matrices $\Sigma_\mu$ and $\Sigma_\nu$, respectively. The MC estimate above satisfies%
\begin{align*}
\EE\Big[\Big|\WMC\mspace{-2mu}-\mspace{-2mu}\SWp^p(\mu,\nu)\Big|\Big]
\mspace{-2.5mu}\lesssim_p\mspace{-2.5mu}  \frac{ \|\mu x\mspace{-2mu}-\mspace{-2mu}\nu x\|^p \mspace{-3mu}+\mspace{-3mu} \|\Sigma_\mu\|_{\op}^{p/2}\mspace{-3mu}+\mspace{-3mu} \|\Sigma_\nu\|_{\op}^{p/2}}{\sqrt{md}} \mspace{-2mu}+\mspace{-2mu}  \frac{\big( \|\Sigma_\nu\|_{\op}^{p/2} \mspace{-3mu}+\mspace{-3mu} \|\Sigma_\mu\|_{\op}^{p/2} \big ) (\log n)^{\ind_{\mspace{-2mu}\{\mspace{-2mu}p=2\mspace{-2mu}\}}}}{n^{(p\wedge 2)/2}}
\end{align*}
where the hidden constant depends only on $p$. %
\end{proposition}

Proposition~\ref{prop: monte carlo} is proven in Appendix~\ref{APPEN: monte carlo proof} by separately bounding the MC and the empirical approximation errors. For the former, we use the Lipschitzness of $w_p^p$ on $\unitsph$ to show that it concentrates about its median. This enables controlling the variance of $\frac{1}{m}\sum_{j=1}^m w_p^p(\Theta_j)$, which, in turn, bounds~the~MC error. For the empirical approximation error, we reduce the analysis to one-sample empirical convergence under $\Wp^p$ for measures on $\R$ and obtain explicit rates by drawing upon the results of~\cite{bobkov2019one}.  

\begin{remark}[Comparison to \cite{nadjahi2020statistical}]
Error bounds for the MC estimate $\widehat{\underline{\mathsf{W}}}_\mathsf{MC}$ were also provided in \cite{nadjahi2020statistical}, but their results differ from ours in two key ways: they use implicit empirical approximation~bounds and leave their MC error in terms of $\Var\big(w_p^p(\Theta)\big)$ without further analysis. Proposition \ref{prop: monte carlo} provides an explicit convergence rates and bounds the said variance in terms of basic characteristics of the population distributions, providing precise rates in $n, m, d,$ and~$p$. 
\end{remark}

\begin{remark}[Blessing of dimensionality]
A cruder approximation of the Lipschitz constant that stems from \cite[Lemma 2]{niles2019estimation} would yield $\mu\|x\|^p + \nu \|x\|^p$ as the numerator of the first term. However, such a bound can have a significantly worse dimension dependence. Indeed, if, for instance, $\mu$ and $\nu$ are both mean zero log-concave with identity covariance matrices, then $\mu\|x\|^p + \nu \|x\|^p$ is $O_d(d^{p/2})$ while the numerator in our bound is $O_d(1)$. For such $\mu$ and $\nu$, the bound decays to 0 as $d \to \infty$.
\end{remark}

\subsection{Max-Slicing: Subgradient Methods and the LIPO Algorithm}

Maximization of projected Wasserstein distance is a non-convex and non-smooth optimization problem. Therefore, past works that studied $k$-dimensional subspace projections relied on convex relaxations \cite{paty2019subspace} or entropic regularization \cite{lin2020projection,huang2021riemannian} to prove approximate convergence to a stationary point. We show that regularization is not needed in the one-dimensional case of $\MSWp$ by proving an $O(\epsilon^{-4})$ computational complexity bound for convergence to stationarity of the simple subgradient-based optimization routine from \cite{kolouri2019generalized,deshpande2019max}. %
We also note that global solutions are attainable via generic algorithms for optimizing Lipschitz functions, but with rates that deteriorate exponentially with $d$. %

\begin{algorithm}[t]
\caption{Projected subgradient method for $\tilde{w}_2^2$}\label{alg:subgradient}
\begin{algorithmic}
	\State\textbf{Input:} $\theta_0 \in \BB^d$, a sequence $\{\alpha_t\}_{t\geq 0}\subset\R_+$, and iteration count $T$
	\For{$t=0,\ldots,T$}
	    \State Calculate $\xi_t \in \partial \tilde{w}_2^2(\theta_t)$
	    \State Set $x_{t+1}=\text{Proj}_{\BB^d}\left(x_{t} - \alpha_t \xi_t\right)$
	\EndFor
	\State Sample $t^*\in \{0,\ldots,T\}$ according to the probability distribution 	$\mathbb{P}(t^*=t)=\frac{\alpha_t}{\sum_{t=0}^T \alpha_t}.$
	\State{\bf Return} $x_{t^*}$		
\end{algorithmic}
\end{algorithm}

\paragraph{Local guarantees for subgradient methods.} First note that we may relax the $\MSWp$ optimization domain from $\unitsph$ to the unit ball $\BB^d$ without changing the value (indeed, for any $\theta \in \BB^d$, $w_p(\theta) = \|\theta\| w_p(\theta/\|\theta\|)$. Together with \cite[Lemma 4.2]{bobkov2019one}, we express the empirical max-sliced distance as: %
\[
\MSWp^p(\empmu,\empnu)= \max_{\theta \in \BB^d} \min_{\pi \in \Pi(\empmu, \empnu)}\mspace{-7mu} \EE_\pi \left [|\theta^\intercal( X - Y)|^p \right ]=
-\min_{\theta \in \BB^d} \max_{\sigma \in S_n} \left (- \frac{1}{n}\sum_{i=1}^n |\theta^\intercal (X_i - Y_{\sigma(i)})|^p \right),%
\]
where $S_n$ is the symmetric group. Here we used the fact that the optimal coupling is given by the order statistics, and hence it suffices to optimize over permutations. Denote $\rho(\sigma, \theta):=- \frac{1}{n}\sum_{i=1}^n |\theta^\intercal (X_i - Y_{\sigma(i)})|^p$ and $\tilde{w}_p^p(\theta):= \max_{\sigma\in S_n} \rho(\sigma, \theta)$. %
The subgradient of $\tilde w_p^p$ has the closed form $\partial \tilde{w}_p^p(\theta)= \text{Conv} \left ( \left\{\partial_\theta \rho(\sigma^*, \theta):\, \sigma^* \in\argmax\nolimits_{\sigma\in S_n} \hat\rho(\sigma, \theta)\right \} \right )$. We can compute an optimal $\sigma^*\in S_n$ via order statistics and evaluate the corresponding subgradient vector in $\partial_\theta \rho(\sigma^*, \theta)$. This gives direct access to subgradients of 
$\tilde{w}_p^p$ without approximation~arguments or regularization.

A heuristic description of Algorithm \ref{alg:subgradient} was given in \cite{kolouri2019generalized,deshpande2019max}, but without formal guarantees. Proposition \ref{PROP:MSWP_subgrad} below can be viewed as closing that gap by providing said guarantees. In particular, for $p=2$ the objective function $\tilde{w}_2^2$ is weakly convex \cite[Lemma 2.2]{lin2020projection} and Lipschitz (Lemma~\ref{lem: wp_theta_lipschitz}). Together with the computable subgradients, this enables applying the proximal stochastic subgradient method from \cite{davis2018stochastic}. Algorithm \ref{alg:subgradient} describes the adaptation of this method to our problem, after replacing the stochastic subgradient sampling step therein with the direct subgradient calculation described above. The following proposition provides convergence guarantees for Algorithm \ref{alg:subgradient}. %

\begin{proposition}[Computational complexity of subgradient method]
\label{PROP:MSWP_subgrad}
Fix any $\epsilon>0$ and $n,d\in\NN$ such that $d \geq (\log n)^2$. Let $\mu,\nu\in\cP_2(\RR^d)$ be log-concave with covariance matrices $\Sigma_\mu$ and $\Sigma_\nu$, respectively, and consider $M_{\mu, \nu}^2$ as defined in Lemma~\ref{lem: wp_theta_lipschitz}.
Then, there exist universal constants $c_1,c_2,c_3,c_4$ such the following holds: Algorithm~\ref{alg:subgradient} for the objective $\varphi(\theta) = \tilde w_2^2+ \delta_{\BB^d}$, where $\delta_{\BB^d} = -\infty \ind_{(\BB^d)^c}$, with step size $\alpha_t \propto \frac{1}{\sqrt{t+1}}$, outputs a point $\theta_{t^*}$ that is close to a near-stationary point $\theta^*$, in the sense that $\EE_{t^*}[\|\theta^*-\theta_{t^*}\|]\leq \frac{\epsilon}{2\rho_{\mu, \nu}}$, for $\rho_{\mu, \nu} = \|\mu x - \nu x\|^2 + c_1 d \left (\|\Sigma_\mu\|_{\op} + \|\Sigma_\nu\|_{\op} \right )$, and $\mathrm{dist}\big(0,\partial \tilde{w}_2^2(\theta^*)\big)\leq \epsilon$, within a number of~computations $N\leq C_{\mu,\nu}\epsilon^{-4} n \log n$, where $C_{\mu,\nu}:=c_2 \rho_{\mu, \nu}^2 \big ( M_{\mu,\nu}^{2} + c_3 (\|\Sigma_\mu\|_{\op} + \|\Sigma_\nu\|_{\op}) \big )^2$, with probability at least $1 - \frac{c_4}{n}$.%
\end{proposition}

Proposition~\ref{PROP:MSWP_subgrad} is proven in Appendix~\ref{APPEN:MSWP_subgrad_proof} via the  complexity bound from \cite[Corollary~2]{davis2018stochastic}. As the algorithm is tuned for the empirical objective $\tilde{w}_p^p(\theta)$, the bound depends on the random Lipschitz and weak convexity constants $M_n = 4 \sup_{\theta\in\unitsph} (\empmu |\theta^\intercal x|^2 + \empnu |\theta^\intercal x|^2)$ and $\rho_n = 2 \max_{i,j=1,\ldots,n} \|X_i - Y_j\|^2$. We use concentration bounds for $M_n$ and $\rho_n$ to obtain the deterministic bound~above. %

\begin{remark}[Comparison to past works]
Computation of projection-robust Wasserstein distances (i.e., when projections are $k$-dimensional) was studied in \cite{paty2019subspace} and \cite{lin2020projection,huang2021riemannian} using a convex relaxations and entropic regularization, respectively. A similar $O(\epsilon^{-4})$ convergence rate is proven in \cite{lin2020projection} for their regularized method. Proposition \ref{PROP:MSWP_subgrad} shows that regularization in not necessary to achieve this rate when projections are one-dimensional. %
The result of \cite{lin2020projection} was improved to $O(\epsilon^{-3})$ in \cite{huang2021riemannian} using Riemannian block coordinate descent (still with entropic regularization). While this rate is faster than in Proposition \ref{PROP:MSWP_subgrad}, our goal was to couple the simpler and abundantly used subgradient ascent approach with formal guarantees. In addition, the next section shows that empirically, our algorithm is much faster than those of \cite{lin2020projection,huang2021riemannian} for the $\MSWtwo$ in terms of complexity and computation time. %
\end{remark}

\begin{figure}[t]
\centering
\begin{subfigure}[b]{0.495\textwidth}
    \centering
    \includegraphics[scale=0.27]{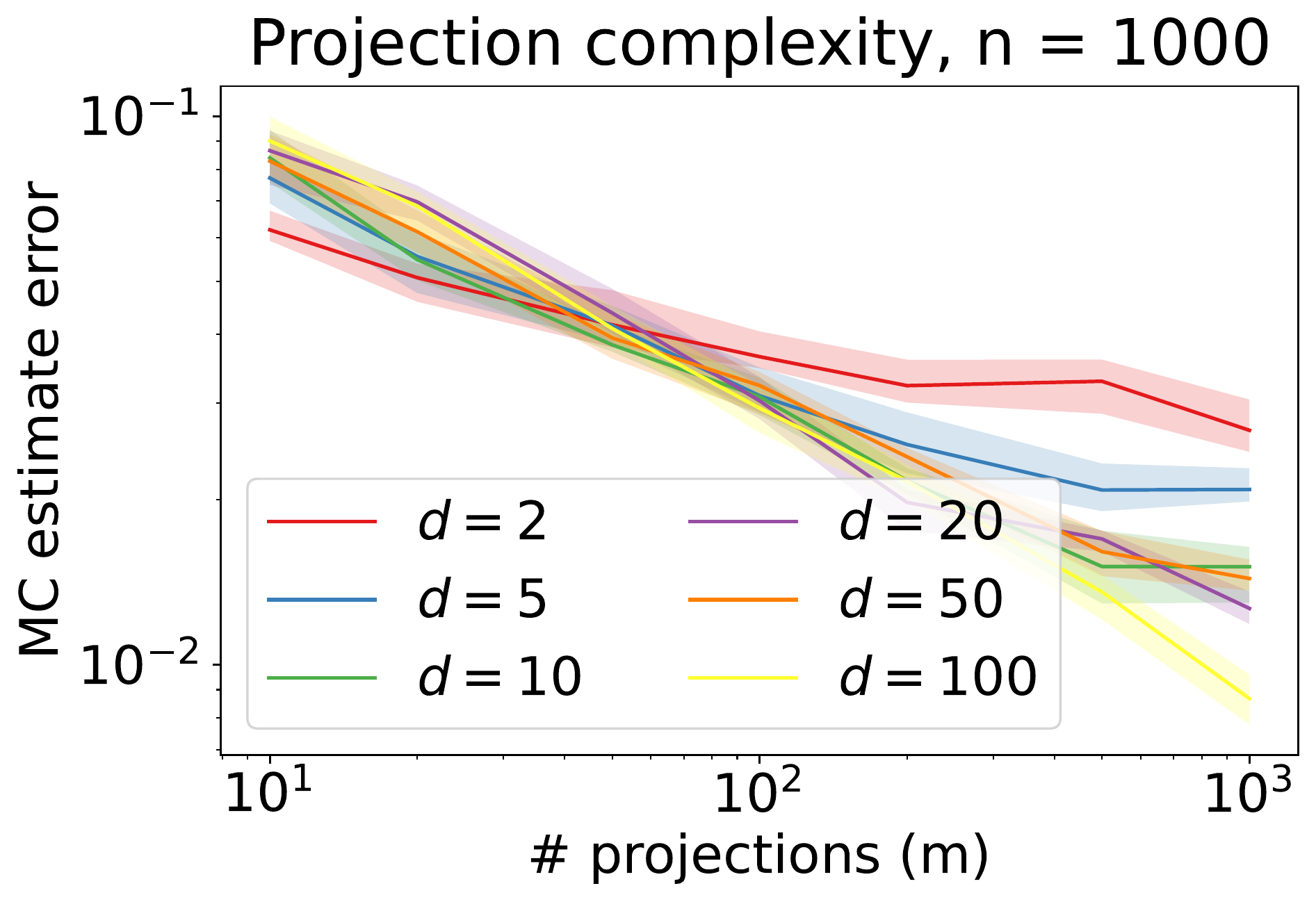}\\
    \vspace{1mm}
    \includegraphics[scale=0.27]{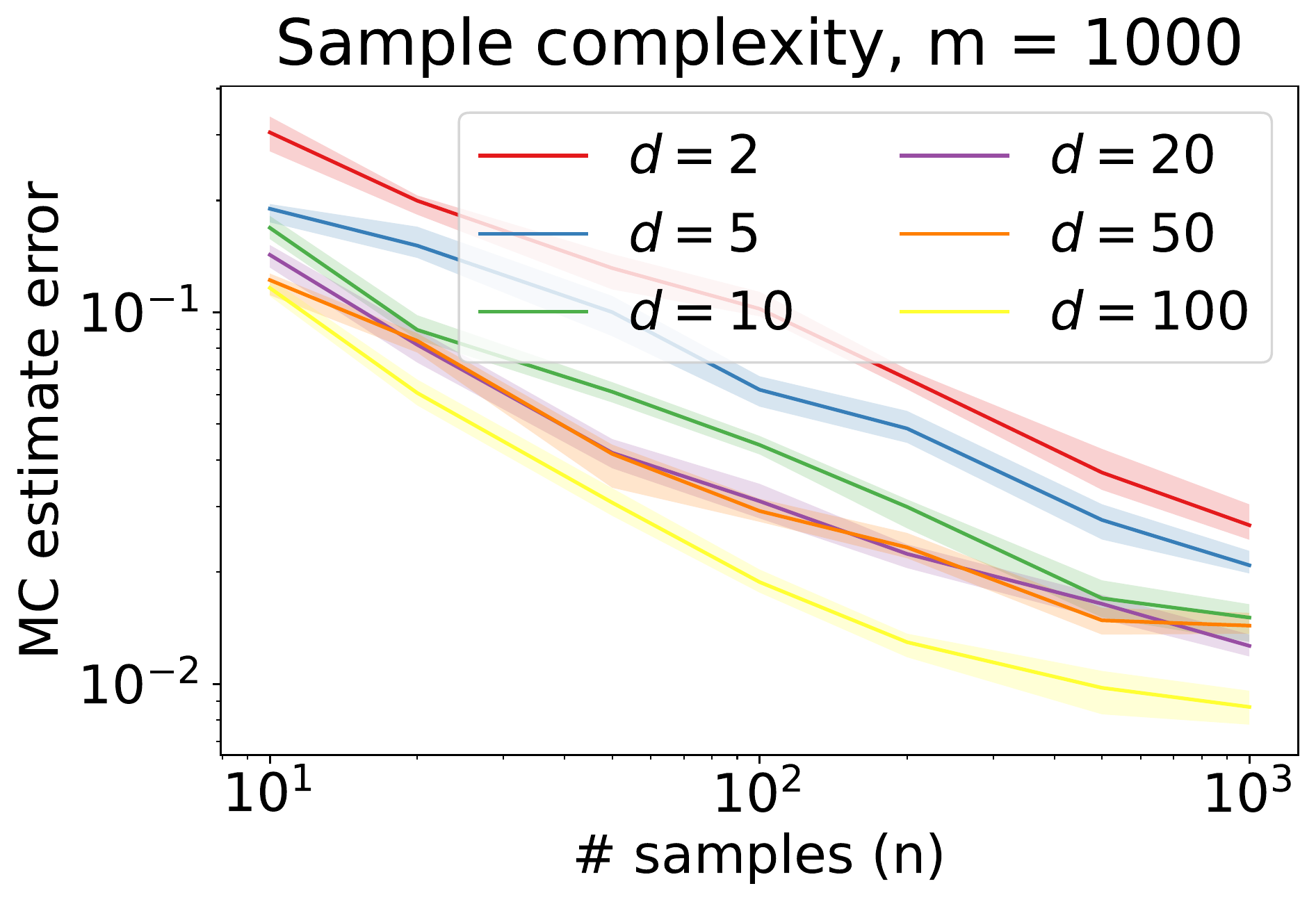}
    \caption{$\big |\WMCtwo - \SWtwo^2(\mu, \nu) \big |$ under Model (1).}
\end{subfigure}
\begin{subfigure}[b]{0.495\textwidth}
    \centering
    \includegraphics[scale=0.27]{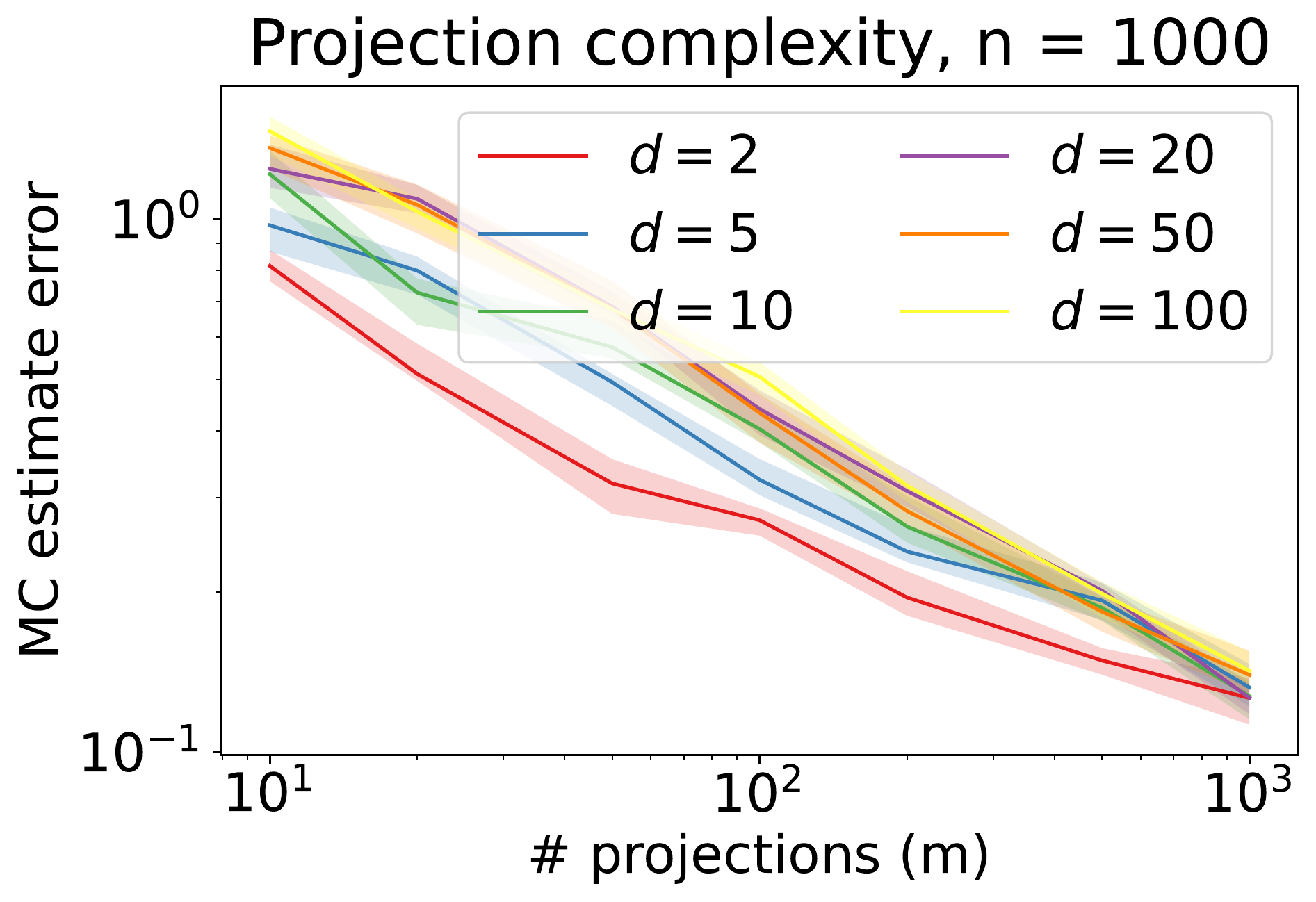}\\
    \vspace{1mm}
    \includegraphics[scale=0.27]{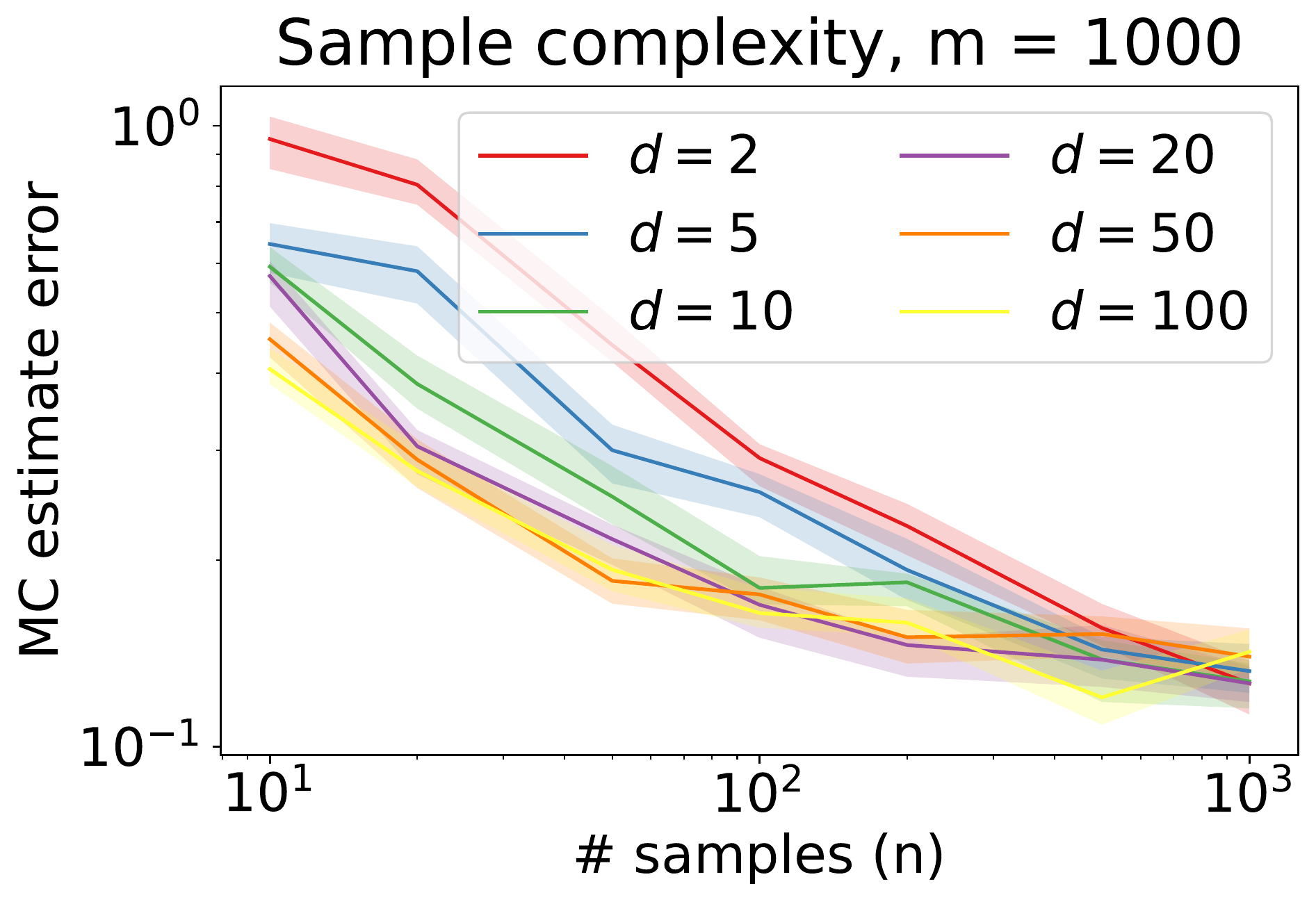}
    \caption{$\big | \WMCtwo - \SWtwo^2(\mu, \nu) \big |$ under Model (2).}
\end{subfigure}
\caption{Projection and sample complexity for $\MSWtwo$.}\label{fig:sw2_cpx}
\vspace{-2mm}
\end{figure}

\begin{remark}[The non-quadratic case]
For $p\neq 2$, we still have Lipschitzness of the objective function in $\theta$ (Lemma \ref{lem: wp_theta_lipschitz}). Recent work on finding stationary points for non-smooth, non-convex, Lipschitz functions, such as \cite{davis2021gradient}, provide convergence guarantees for these cases. These guarantees appear to be of the same $\epsilon^{-1/4}$ order (cf. \cite[Theorem 3.2]{davis2021gradient}), but we leave a full exploration for future work.
\end{remark}

\begin{remark}[Global guarantees via LIPO]
We can attain global optimality, i.e., compute $\MSWp(\empmu, \empnu)$ itself, via the LIPO algorithm \cite{malherbe2017global}. LIPO performs global optimization of Lipschitz~functions over convex domains based on function evaluations, which are readily accessible in our problem via sorting. In Appendix~\ref{APPEN:LIPO}, we adapt LIPO to the max-sliced distance, prove consistency, and derive its complexity. While this approach attains global optimality, the number of evaluation grows~exponentially with dimension. Hence, the subgradient method described above is preferable when dimension is large. 
\end{remark}

\section{Empirical Results}
\label{sec:experiments}

\paragraph{Projection and sample complexity for {\boldmath $\SWtwo$}.}
We validate the convergence rates of the MC-based estimate of $\SWtwo$ predicted by Propositions~\ref{prop: monte carlo} in the following two models: %
(1) $\mu = \cN(0, I_d)$, $\nu = 0.5 \cN(0, I_d) + 0.5 \cN(0, I_d + 0.5 \bm{1}_d \bm{1}_d^\intercal/d)$, and (2) $\mu = \cN(0, I_d)$, $\nu = \cN(2 \bm{1}_d, I_d)$, where $\bm{1}_d$ is a vector with all coordinates equal to 1. For Model (1), Proposition~\ref{prop: monte carlo} predicts a decreasing error with dimension, and inverse square root decay in number of projections and number of samples. For Model (2), on the other hand, the errors should increase with~$d$ for sufficiently large $n$. Plots of the projection and sample complexities for each model (averaged over 100 runs) are given in Figure \ref{fig:sw2_cpx} at the top of the previous page, and are in line with the above discussion and our theoretical results. Additionally, confidence bands are plotted representing top and bottom 10\% quantiles among 20 bootstrappped means from the same 100 runs. An additional experimental setup, comparing 10 component normal mixtures with different means and variances, can be found in Appendix~\ref{app:experiments}. %

\begin{figure}[t]
\centering
\begin{center}
\includegraphics[width=0.48\textwidth]{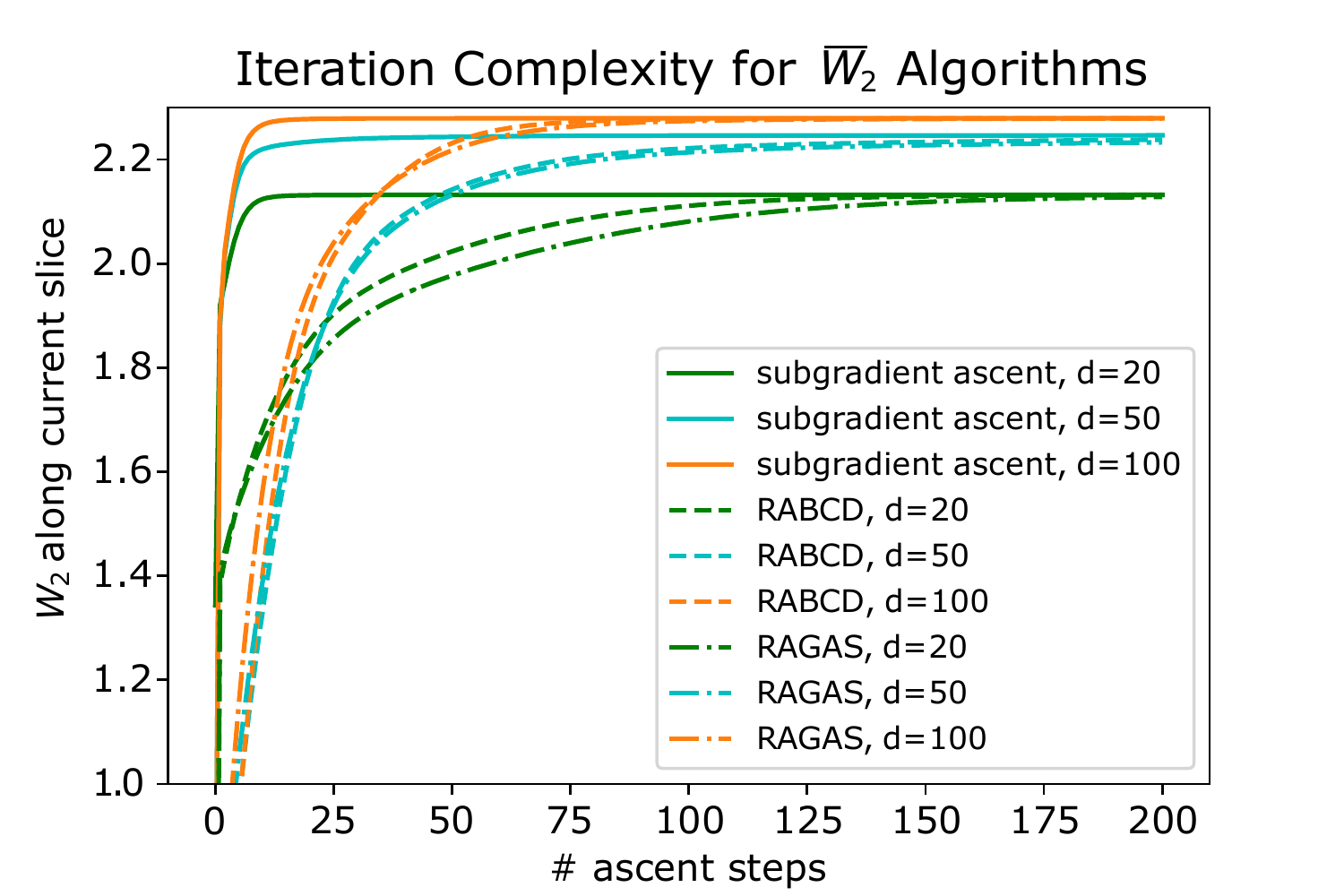}\includegraphics[width = 0.48\textwidth]{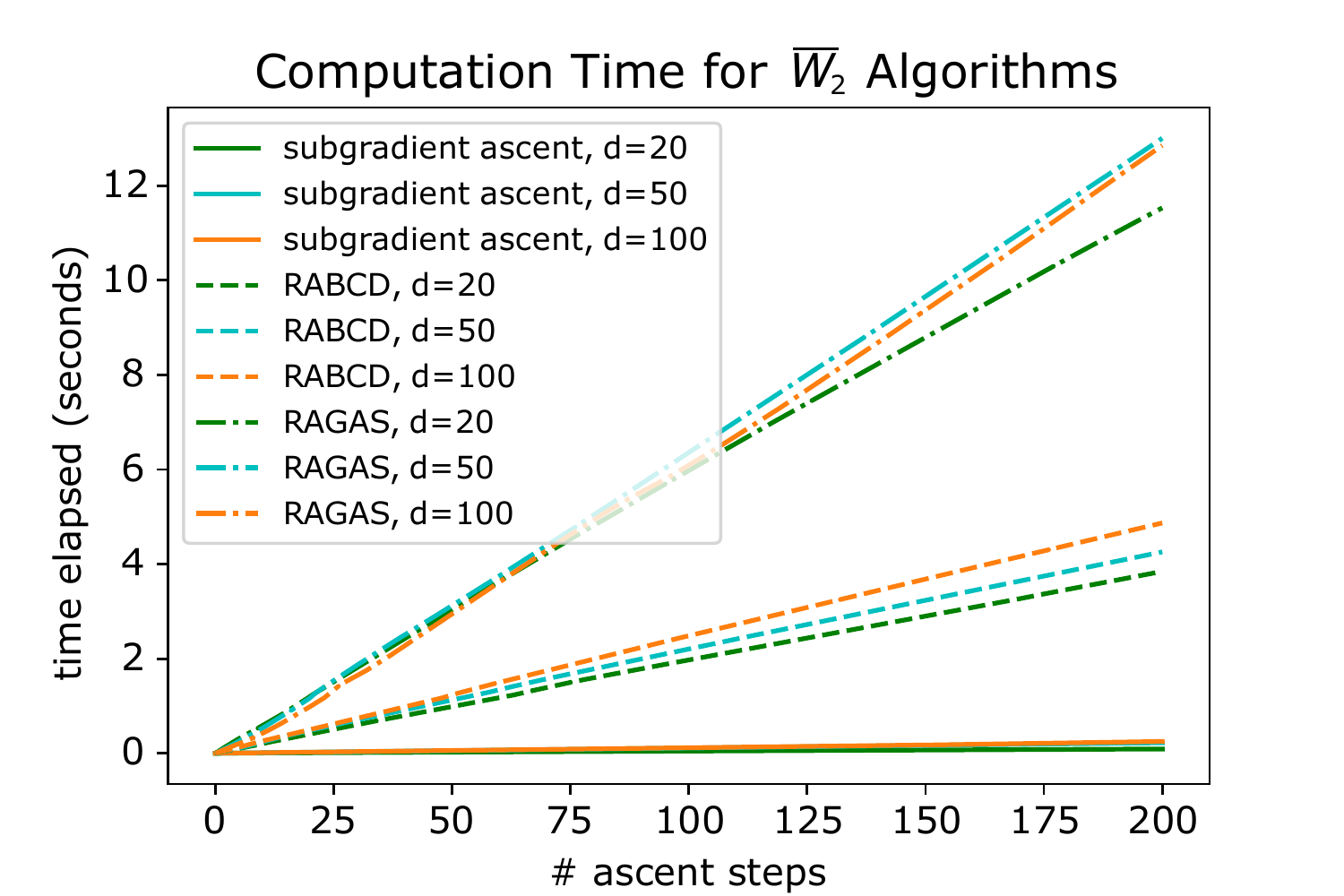}
\end{center}
\vspace{-2mm}
\caption{Errors and runtime versus step count for $\MSWtwo$ computation algorithms.}\label{fig:MSW2_computation}
\vspace{-2mm}
\label{fig:}
\end{figure}

\paragraph{Comparison of {\boldmath $\MSWtwo$} algorithms.}
We compare the performance of the subgradient-based Algorithm \ref{alg:subgradient} and the Riemannian optimization methods of \cite{lin2020projection, huang2021riemannian}. Consider the setup from \cite[Section 6.1]{huang2021riemannian}, where $\mu = \Unif([-1,1]^d)$ and $\nu = T_\sharp \mu$, with $T(x) = x + \sum_{i=1}^{10} \mathrm{sign}(x_i) e_i $, is the fragmented hypercube distribution with $k^*=10$. Figure \ref{fig:MSW2_computation} shows the errors and runtime by step count of Algorithm \ref{alg:subgradient} (with a constant step size) and the Riemannian algorithms from \cite{lin2020projection,huang2021riemannian} (abbreviated RAGAS and RABCD, respectively) for different ambient dimensions. For these algorithms, we used the code from \url{https://github.com/mhhuang95/PRW_RBCD} with their default choice of parameters; we also tried optimizing these parameters but the observed trends remained the same. Sample size is fixed at $n = 500$ and computation times are averaged over $10$ trials. 
Evidently, the subgradient ascent algorithm converges significantly faster and within fewer iterations than the other two methods, for all considered values of~$d$. Despite our $O(\epsilon^{-4})$ iteration complexity bound, which is slower than the best known rates \cite{huang2021riemannian}, this favorable empirical performance may be attributed to the fact that Algorithm \ref{alg:subgradient} relies on the cheap sorting operation, as opposed to the burdensome computation of regularization operations in \cite{lin2020projection,huang2021riemannian}. It may also be the case that our bound can be improved, which we plan to explore in future work.

\begin{figure}[t]
\centering
\begin{center}
    \includegraphics[width=0.48\textwidth]{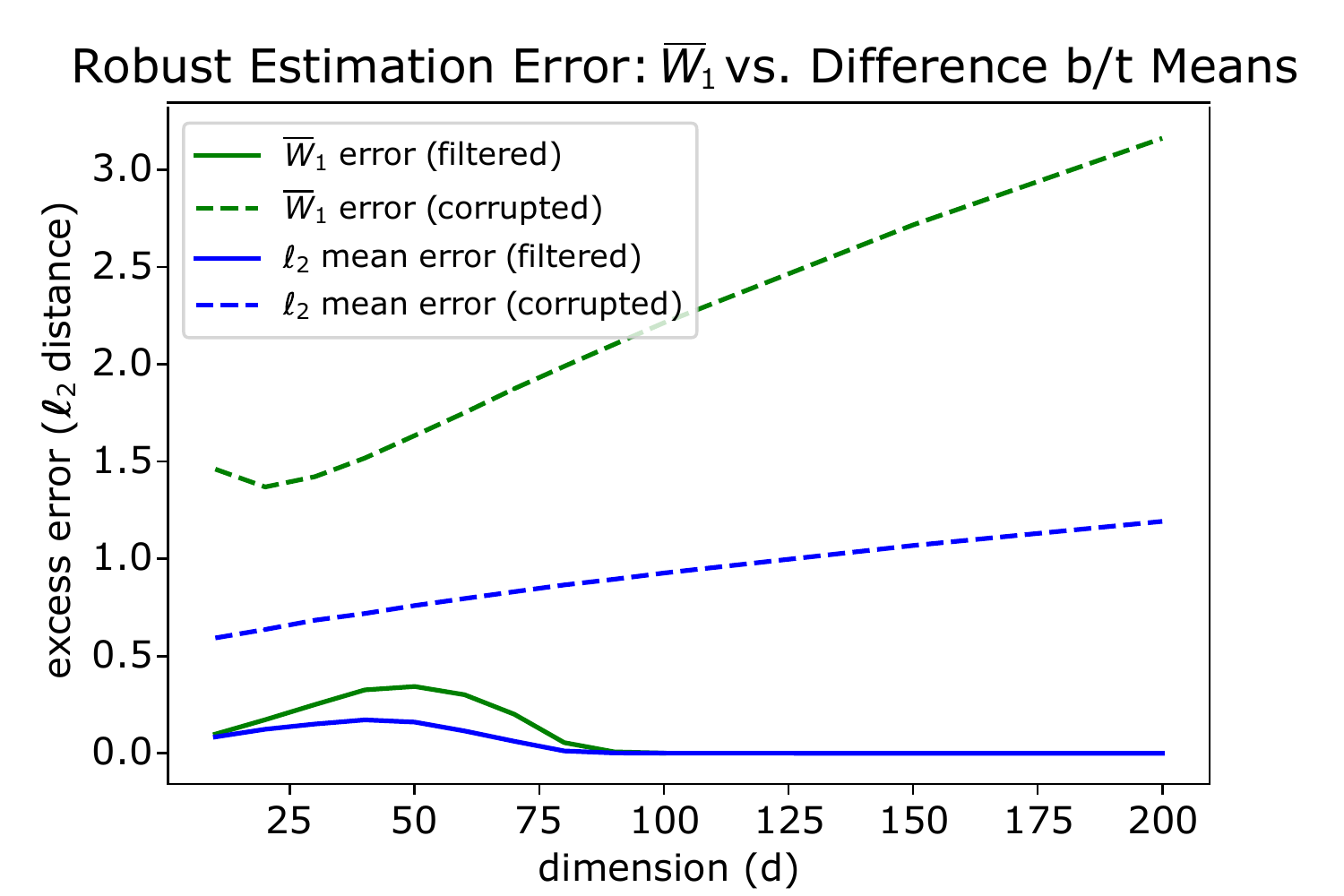}
    \includegraphics[width=0.48\textwidth]{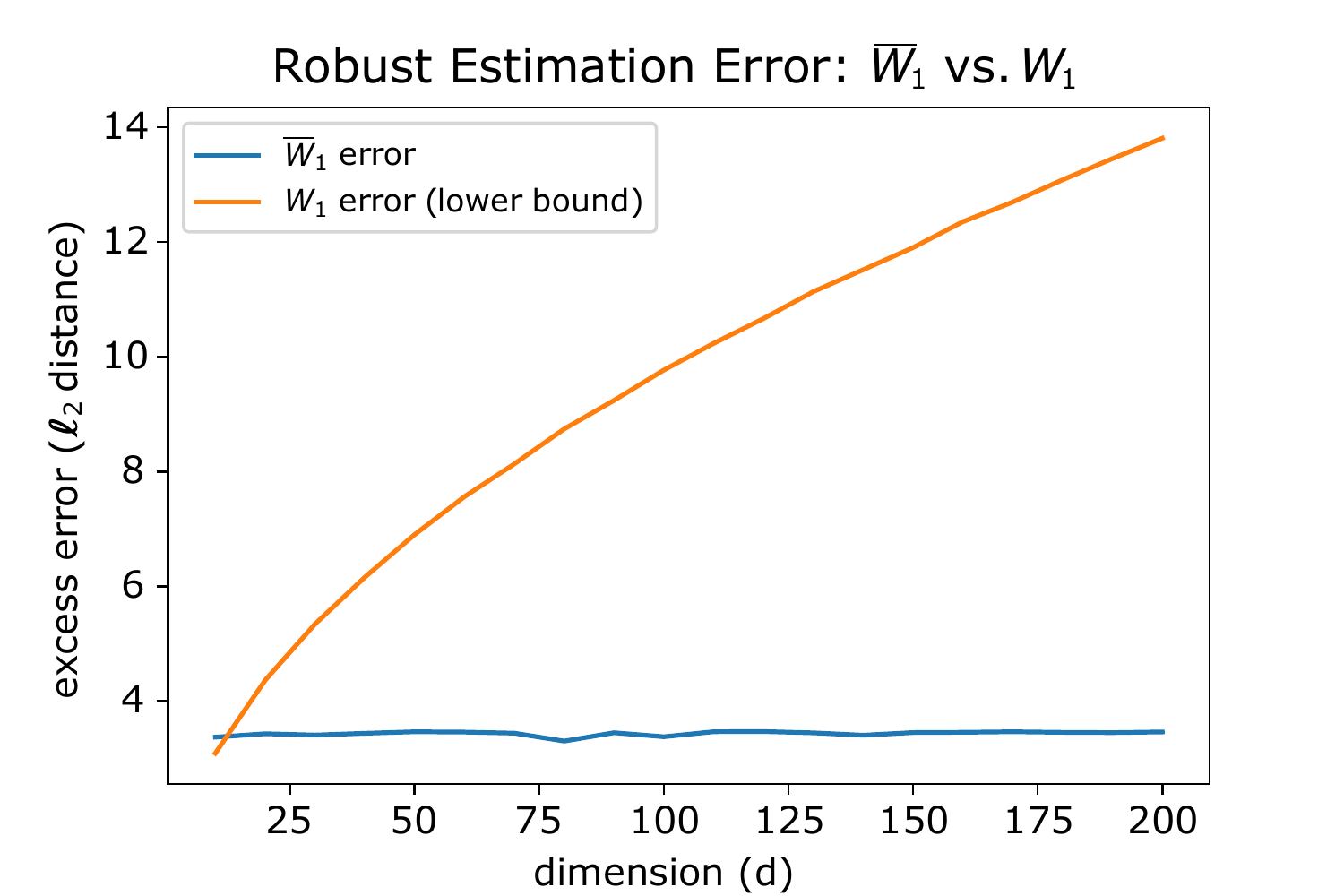}
\end{center}
\caption{Robust estimation errors for the iterative filtering estimate in two scenarios: (left) comparing $\MSWone$ with difference between means and (right) comparing $\MSWone$ to $\Wone$. Mean and $\MSWone$ errors are bounded as $d \to \infty$, while $\Wone$ error scales like $\sqrt{d}$.}\label{fig:robust-estimation}
\end{figure}

\paragraph{Robust estimation.}
To support Proposition \ref{prop:resilience}, we perform robust estimation via a standard iterative filtering algorithm developed for mean estimation \cite{diakonikolas2017robust}. For each $d \in \{10,20,\dots,200\}$, we take $n = 10d\eps^{-2}$ samples, with $(1-\eps)n$ drawn i.i.d.\ from $\cN(0,I_d)$ and $\eps n$ from a product noise distribution used in \cite{diakonikolas2017robust}, with $\eps = 0.1$. For each~$d$, iterative filtering returns a candidate subset of clean samples, and Figure \ref{fig:robust-estimation} (left) compares these subsets to the true clean samples both in $\MSWone$ (estimated via projected subgradient ascent) and in $\ell_2$ distance between means. Note that the error in the latter never exceeds that in the former by more than a factor of 2 (Proposition \ref{prop:resilience} implies that mean and $\SWone$ risk are equivalent up to constant factors). In Figure~\ref{fig:robust-estimation} (right), we set $\mu = (1-\eps)\delta_0 + \eps \Unif(\sqrt{d/\eps} \,\unitsph)$ with null contamination, and take $n,d$ ranging as before. In this case, since many samples are 0, we can efficiently compute a lower bound on classic $\Wone$ between the filtered and clean samples in high dimensions. As predicted by Theorem \ref{thm:robustness}, we observe the $\sqrt{d}$ separation in estimation error between $\Wone$ and $\MSWone$. In this case, only errors for the filtered samples are plotted, since the unfiltered samples have no contamination. See Appendix~\ref{app:experiments} for additional experiments on generative modeling with contaminated datasets, along with full details and code for all experiments.

\section{Summary and Concluding Remarks}

This paper provided a quantitative study of the scalability of sliced Wasserstein distances to high dimensions. Three key aspects were covered:
\begin{itemize}[leftmargin=.4cm]
    \item \textbf{Empirical convergence rates:} We established sharp, dimension-free rates for $\SWp$ and $\MSWp$ with explicitly characterized dimension-dependent constants. Our bounds reveal the interplay between the number of samples $n$, dimension $d$, and the order of the distance~$p$. 
    \item \textbf{Robust estimation:} The minimax optimal robust estimation rate of $\SWp$ and $\MSWp$, under contamination level $\eps$, was characterized as $O(\sigma\eps^{1/p-1/q})$. This rate is dimension-free and improves~upon corresponding results for $\Wp$ by a $\sqrt{d}$ factor. We showed that robust estimation of $\MSWone$ is equivalent to robust mean estimation, which enables lifting statistical/algorithmic results from means to $\MSWone$.
    \item \textbf{Computational guarantees:} The error of a MC-based estimator for $\SWp$ was derived, showing that it can improve as $d\to\infty$, depending on the growth-rate of the mean and the operator~norm of the covariance matrix. For $\MSWp$, we analyzed the subgradient-based local optimization algorithm from \cite{kolouri2019generalized,deshpande2019max}, and proved $O(\epsilon^{-4})$ complexity using Lipschitzness and weak convexity of the objective.
\end{itemize}
In all three aspects, the benefit of slicing in terms of dependence on dimension was clearly evident, thus providing rigorous justification the perceived scalability of sliced distance. Going forward, we plan to pursue improved complexity bounds for the subgradient ascent algorithm for computing $\MSWtwo$, as our empirical results suggest it converges faster than Proposition \ref{PROP:MSWP_subgrad} predicts. We are also interested in understanding conditions on $\mu,\nu$ under which faster global guarantees for $\MSWp$ can be provided, e.g., by precluding the existence of nontrivial local optima for $\Wp (\proj^{\theta}_\sharp \empmu, \proj^{\theta}_\sharp \empnu)$ on $\unitsph$, or matching the conditions of \cite[Theorem 15]{malherbe2017global}, which results in polynomial and even exponential rates for LIPO. %
Extensions of our results to projection-robust Wasserstein distance, which considers projections to $k$-dimensional subspaces, are of interest, aiming to understand the effect of $k$ on the results.

\subsection*{Acknowledgements}
The authors thank Jason Gaitonde for helpful discussion on high-dimensional probability. S.\ Nietert was supported by a NSF Graduate Research Fellowship under Grant DGE-1650441. Z.\ Goldfeld is partially supported by NSF grants CCF-1947801,  CCF-2046018, and DMS-2210368, and the 2020 IBM Academic Award. K.\ Kato is supported by NSF grants DMS-1952306, DMS-2014636, and DMS-2210368.

\bibliographystyle{abbrv}
\bibliography{ref}

\begin{thebibliography}{10}

\bibitem{adamczak2014tail}
R.~Adamczak, R.~Lata{\l}a, A.~E. Litvak, A.~Pajor, and N.~Tomczak-Jaegermann.
\newblock Tail estimates for norms of sums of log-concave random vectors.
\newblock {\em Proceedings of the London Mathematical Society},
  108(3):600--637, 2014.

\bibitem{adamczak2010quantitative}
R.~Adamczak, A.~Litvak, A.~Pajor, and N.~Tomczak-Jaegermann.
\newblock Quantitative estimates of the convergence of the empirical covariance
  matrix in log-concave ensembles.
\newblock {\em Journal of the American Mathematical Society}, 23(2):535--561,
  2010.

\bibitem{arjovsky_wgan_2017}
M.~Arjovsky, S.~Chintala, and L.~Bottou.
\newblock {W}asserstein generative adversarial networks.
\newblock In {\em International Conference on Machine Learning (ICML)}, 2017.

\bibitem{balaji2020}
Y.~Balaji, R.~Chellappa, and S.~Feizi.
\newblock Robust optimal transport with applications in generative modeling and
  domain adaptation.
\newblock In {\em Advances in Neural Information Processing Systems (NeurIPS)},
  2020.

\bibitem{Bartl2022structure}
D.~Bartl and S.~Mendelson.
\newblock Structure preservation via the {W}asserstein distance.
\newblock {\em arXiv preprint arXiv:2209.07058}, 2022.

\bibitem{bayraktar2021}
E.~Bayraktar and G.~Guo.
\newblock {Strong equivalence between metrics of {W}asserstein type}.
\newblock {\em Electronic Communications in Probability}, 26:1--13, 2021.

\bibitem{bobkov1997}
S.~G. Bobkov and C.~Houdré.
\newblock {Isoperimetric constants for product probability measures}.
\newblock {\em The Annals of Probability}, 25(1):184 -- 205, 1997.

\bibitem{bobkov2019one}
S.~G. Bobkov and M.~Ledoux.
\newblock {\em One-dimensional empirical measures, order statistics, and
  Kantorovich transport distances}.
\newblock American Mathematical Society, 2019.

\bibitem{bonneel_barycenter_2015}
N.~Bonneel, J.~Rabin, G.~Peyr{\'e}, and H.~Pfister.
\newblock Sliced and {R}adon {W}asserstein barycenters of measures.
\newblock {\em {Journal of Mathematical Imaging and Vision}}, 1(51):22--45,
  2015.

\bibitem{bonnotte2013unidimensional}
N.~Bonnotte.
\newblock {\em Unidimensional and evolution methods for optimal
  transportation}.
\newblock PhD thesis, Paris 11, 2013.

\bibitem{borell1974}
C.~Borell.
\newblock {Convex measures on locally convex spaces}.
\newblock {\em Arkiv för Matematik}, 12(1-2):239--252, 1974.

\bibitem{carriere2017sliced}
M.~Carriere, M.~Cuturi, and S.~Oudot.
\newblock Sliced {W}asserstein kernel for persistence diagrams.
\newblock In {\em International Conference on Machine Learning (ICML)}, 2017.

\bibitem{chen2021almost}
Y.~Chen.
\newblock An almost constant lower bound of the isoperimetric coefficient in
  the {KLS} conjecture.
\newblock {\em Geometric and Functional Analysis}, 31(1):34--61, 2021.

\bibitem{davis2018stochastic}
D.~Davis and D.~Drusvyatskiy.
\newblock Stochastic subgradient method converges at the rate ${O}(k^{-1/4}) $
  on weakly convex functions.
\newblock {\em arXiv preprint arXiv:1802.02988}, 2018.

\bibitem{davis2021gradient}
D.~Davis and D.~Drusvyatskiy.
\newblock A gradient sampling method with complexity guarantees for general
  {L}ipschitz functions.
\newblock {\em arXiv preprint arXiv:2112.06969}, 2021.

\bibitem{deng2012mnist}
L.~Deng.
\newblock The {MNIST} database of handwritten digit images for machine learning
  research.
\newblock {\em IEEE Signal Processing Magazine}, 29(6):141--142, 2012.

\bibitem{deshpande2019max}
I.~Deshpande, Y.-T. Hug, R.~Sun, A.~Pyrros, N.~Siddiqui, S.~Koyejo, Z.~Zhao,
  D.~Forsyth, and A.~G. Schwing.
\newblock Max-sliced {W}asserstein distance and its use for {GANs}.
\newblock In {\em Proceedings of the IEEE/CVF Conference on Computer Vision and
  Pattern Recognition (CVPR)}, 2019.

\bibitem{deshpande2018generative}
I.~Deshpande, Z.~Zhang, and A.~G. Schwing.
\newblock Generative modeling using the sliced {W}asserstein distance.
\newblock In {\em Proceedings of the IEEE Conference on Computer Vision and
  Pattern Recognition (CVPR)}, 2018.

\bibitem{diakonikolas2017robust}
I.~Diakonikolas, G.~Kamath, D.~M. Kane, J.~Li, A.~Moitra, and A.~Stewart.
\newblock Being robust (in high dimensions) can be practical.
\newblock In {\em Proceedings of the International Conference on Machine
  Learning (ICML)}, 2017.

\bibitem{donoho1988robustness}
D.~L. Donoho and R.~C. Liu.
\newblock {The "Automatic" Robustness of Minimum Distance Functionals}.
\newblock {\em The Annals of Statistics}, 16(2):552 -- 586, 1988.

\bibitem{goldfeld2022statistical}
Z.~Goldfeld, K.~Kato, G.~Rioux, and R.~Sadhu.
\newblock Statistical inference with regularized optimal transport.
\newblock {\em arXiv preprint arXiv:2205.04283}, 2022.

\bibitem{gulrajani2017improved}
I.~Gulrajani, F.~Ahmed, M.~Arjovsky, V.~Dumoulin, and A.~C. Courville.
\newblock Improved training of {W}asserstein {GAN}s.
\newblock In {\em Advances in Neural Information Processing Systems (NeurIPS)},
  2017.

\bibitem{gupta2004beta}
A.~K. Gupta and S.~Nadarajah, editors.
\newblock {\em Handbook of beta distribution and its applications}, volume 174
  of {\em Statistics: Textbooks and Monographs}.
\newblock Marcel Dekker, Inc., New York, 2004.

\bibitem{hopkins2020robust}
S.~B. Hopkins, J.~Li, and F.~Zhang.
\newblock Robust and heavy-tailed mean estimation made simple, via regret
  minimization.
\newblock In {\em Proceedings of the International Conference on Neural
  Information Processing Systems (NeurIPS)}, 2020.

\bibitem{huang2021riemannian}
M.~Huang, S.~Ma, and L.~Lai.
\newblock A {R}iemannian block coordinate descent method for computing the
  projection robust {W}asserstein distance.
\newblock In {\em International Conference on Machine Learning}, pages
  4446--4455. PMLR, 2021.

\bibitem{kannan1995isoperimetric}
R.~Kannan, L.~Lov{\'a}sz, and M.~Simonovits.
\newblock Isoperimetric problems for convex bodies and a localization lemma.
\newblock {\em Discrete \& Computational Geometry}, 13(3):541--559, 1995.

\bibitem{kim2020robust}
I.~Kim, S.~Balakrishnan, and L.~Wasserman.
\newblock Robust multivariate nonparametric tests via projection averaging.
\newblock {\em The Annals of Statistics}, 48(6):3417--3441, 2020.

\bibitem{kolouri2019generalized}
S.~Kolouri, K.~Nadjahi, U.~Simsekli, R.~Badeau, and G.~K. Rohde.
\newblock Generalized sliced {W}asserstein distances.
\newblock {\em arXiv preprint arXiv:1902.00434}, 2019.

\bibitem{kolouri2018sliced}
S.~Kolouri, P.~E. Pope, C.~E. Martin, and G.~K. Rohde.
\newblock Sliced {W}asserstein auto-encoders.
\newblock In {\em International Conference on Learning Representations}, 2018.

\bibitem{khang2021}
K.~Le, H.~Nguyen, Q.~Nguyen, N.~Ho, T.~Pham, and H.~Bui.
\newblock On robust optimal transport: Computational complexity, low-rank
  approximation, and barycenter computation.
\newblock {\em arXiv preprint arXiv:2102.06857}, 2021.

\bibitem{ledoux1999concentration}
M.~Ledoux.
\newblock Concentration of measure and logarithmic {S}obolev inequalities.
\newblock In {\em Seminaire de probabilites XXXIII}, pages 120--216. Springer,
  1999.

\bibitem{ledoux1991probability}
M.~Ledoux and M.~Talagrand.
\newblock {\em Probability in Banach Spaces: isoperimetry and processes},
  volume~23.
\newblock Springer Science \& Business Media, 1991.

\bibitem{lee2018kannan}
Y.~T. Lee and S.~S. Vempala.
\newblock The {K}annan-{L}ov\'asz-{S}imonovits conjecture.
\newblock {\em arXiv preprint arXiv:1807.03465}, 2018.

\bibitem{lin2020projection}
T.~Lin, C.~Fan, N.~Ho, M.~Cuturi, and M.~Jordan.
\newblock Projection robust {W}asserstein distance and {R}iemannian
  optimization.
\newblock {\em Advances in neural information processing systems},
  33:9383--9397, 2020.

\bibitem{lin2021projection}
T.~Lin, Z.~Zheng, E.~Chen, M.~Cuturi, and M.~Jordan.
\newblock On projection robust optimal transport: Sample complexity and model
  misspecification.
\newblock In {\em International Conference on Artificial Intelligence and
  Statistics}, pages 262--270. PMLR, 2021.

\bibitem{malherbe2017global}
C.~Malherbe and N.~Vayatis.
\newblock Global optimization of {L}ipschitz functions.
\newblock In {\em International Conference on Machine Learning}, pages
  2314--2323. PMLR, 2017.

\bibitem{manole2022minimax}
T.~Manole, S.~Balakrishnan, and L.~Wasserman.
\newblock Minimax confidence intervals for the sliced {W}asserstein distance.
\newblock {\em Electronic Journal of Statistics}, 16(1):2252--2345, 2022.

\bibitem{marchal2017beta}
O.~Marchal and J.~Arbel.
\newblock On the sub-{G}aussianity of the beta and {D}irichlet distributions.
\newblock {\em Electron. Commun. Probab.}, 22:Paper No. 54, 14, 2017.

\bibitem{milman2009role}
E.~Milman.
\newblock On the role of convexity in isoperimetry, spectral gap and
  concentration.
\newblock {\em Inventiones mathematicae}, 177(1):1--43, 2009.

\bibitem{mukherjee2021}
D.~Mukherjee, A.~Guha, J.~Solomon, Y.~Sun, and M.~Yurochkin.
\newblock Outlier-robust optimal transport.
\newblock In {\em International Conference on Machine Learning (ICML)}, 2021.

\bibitem{nadjahi2020approximate}
K.~Nadjahi, V.~De~Bortoli, A.~Durmus, R.~Badeau, and U.~{\c{S}}im{\c{s}}ekli.
\newblock Approximate {B}ayesian computation with the sliced-{W}asserstein
  distance.
\newblock In {\em ICASSP 2020-2020 IEEE International Conference on Acoustics,
  Speech and Signal Processing (ICASSP)}, pages 5470--5474. IEEE, 2020.

\bibitem{nadjahi2020statistical}
K.~Nadjahi, A.~Durmus, L.~Chizat, S.~Kolouri, S.~Shahrampour, and U.~Simsekli.
\newblock Statistical and topological properties of sliced probability
  divergences.
\newblock {\em Advances in Neural Information Processing Systems},
  33:20802--20812, 2020.

\bibitem{nadjahi2021fast}
K.~Nadjahi, A.~Durmus, P.~E. Jacob, R.~Badeau, and U.~Simsekli.
\newblock Fast approximation of the sliced-{W}asserstein distance using
  concentration of random projections.
\newblock {\em Advances in Neural Information Processing Systems}, 34, 2021.

\bibitem{nadjahi2019asymptotic}
K.~Nadjahi, A.~Durmus, U.~{\c{S}}im{\c{s}}ekli, and R.~Badeau.
\newblock Asymptotic guarantees for learning generative models with the
  sliced-{W}asserstein distance.
\newblock {\em arXiv preprint arXiv:1906.04516}, 2019.

\bibitem{nath2020}
J.~S. Nath.
\newblock Unbalanced optimal transport using integral probability metric
  regularization.
\newblock {\em arXiv preprint arXiv:2011.05001}, 2020.

\bibitem{nietert2022robust}
S.~Nietert, R.~Cummings, and Z.~Goldfeld.
\newblock Outlier-robust optimal transport: duality, structure, and statistical
  analysis.
\newblock In {\em International Conference on Artificial Intelligence and
  Statistics}, 2022.

\bibitem{niles2019estimation}
J.~Niles-Weed and P.~Rigollet.
\newblock Estimation of {W}asserstein distances in the spiked transport model.
\newblock {\em arXiv preprint arXiv:1909.07513}, 2019.

\bibitem{paty2019subspace}
F.-P. Paty and M.~Cuturi.
\newblock Subspace robust {W}asserstein distances.
\newblock In {\em International Conference on Machine Learning}, pages
  5072--5081. PMLR, 2019.

\bibitem{rabin2011wasserstein}
J.~Rabin, G.~Peyr{\'e}, J.~Delon, and M.~Bernot.
\newblock Wasserstein barycenter and its application to texture mixing.
\newblock In {\em International Conference on Scale Space and Variational
  Methods in Computer Vision}, pages 435--446. Springer, 2011.

\bibitem{rakotomamonjy2021differentially}
A.~Rakotomamonjy and R.~Liva.
\newblock Differentially private sliced {W}asserstein distance.
\newblock In {\em International Conference on Machine Learning}, pages
  8810--8820. PMLR, 2021.

\bibitem{reeves2017conditional}
G.~Reeves.
\newblock Conditional central limit theorems for gaussian projections.
\newblock In {\em 2017 IEEE International Symposium on Information Theory
  (ISIT)}, pages 3045--3049. IEEE, 2017.

\bibitem{spruill2007spheres}
M.~C. Spruill.
\newblock Asymptotic distribution of coordinates on high dimensional spheres.
\newblock {\em Electron. Comm. Probab.}, 12:234--247, 2007.

\bibitem{staerman21}
G.~Staerman, P.~Laforgue, P.~Mozharovskyi, and F.~d'Alch{\'{e}}{-}Buc.
\newblock When {OT} meets {MoM:} robust estimation of {W}asserstein distance.
\newblock In {\em International Conference on Artificial Intelligence and
  Statistics (AISTATS)}, 2021.

\bibitem{steinhardt2018robust}
J.~Steinhardt.
\newblock {\em Robust learning: Information theory and algorithms}.
\newblock Stanford University, 2018.

\bibitem{steinhardt2018resilience}
J.~Steinhardt, M.~Charikar, and G.~Valiant.
\newblock Resilience: {A} criterion for learning in the presence of arbitrary
  outliers.
\newblock In A.~R. Karlin, editor, {\em 9th Innovations in Theoretical Computer
  Science Conference, {ITCS} 2018, January 11-14, 2018, Cambridge, MA, {USA}},
  volume~94, pages 45:1--45:21, 2018.

\bibitem{van1996weak}
A.~W. van~der Vaart and J.~A. Wellner.
\newblock Weak convergence.
\newblock In {\em Weak convergence and empirical processes}, pages 16--28.
  Springer, 1996.

\bibitem{wainwright2019high}
M.~J. Wainwright.
\newblock {\em High-dimensional statistics: A non-asymptotic viewpoint},
  volume~48.
\newblock Cambridge University Press, 2019.

\bibitem{wu2019sliced}
J.~Wu, Z.~Huang, D.~Acharya, W.~Li, J.~Thoma, D.~P. Paudel, and L.~V. Gool.
\newblock Sliced {W}asserstein generative models.
\newblock In {\em Proceedings of the IEEE/CVF Conference on Computer Vision and
  Pattern Recognition}, pages 3713--3722, 2019.

\bibitem{xi2022distributional}
J.~Xi and J.~Niles-Weed.
\newblock Distributional convergence of the sliced {W}asserstein process.
\newblock {\em arXiv preprint arXiv:2206.00156}, 2022.

\bibitem{xu2022central}
X.~Xu and Z.~Huang.
\newblock Central limit theorem for the sliced 1-{W}asserstein distance and the
  max-sliced 1-{W}asserstein distance.
\newblock {\em arXiv preprint arXiv:2205.14624}, 2022.

\bibitem{zho2019resilience}
B.~Zhu, J.~Jiao, and J.~Steinhardt.
\newblock Generalized resilience and robust statistics.
\newblock {\em arXiv preprint arXiv:1909.08755}, 2019.

\end{thebibliography}

\newpage

\appendix

\section{Cheeger Constant}\label{APPEN:log_conc_Cheeger}

Our empirical convergence rate analysis for the proof of Theorem \ref{thm: SWp rate} relies on controlling the Cheeger (isoperimetric) constant of the projected distributions. This section collects basic definitions and facts about Cheeger constants.

For  $\mu \in \calP(\R^d)$, define the boundary measure of a Borel subset $A \subset \R^d$ as
\[
\mu^+(\partial A) := \liminf_{\epsilon \downarrow 0}
\frac{\mu(A^\epsilon) - \mu(A)}{\epsilon},
\]
where $A^\epsilon = \{x \in \R^d: d(x,A) \leq \epsilon\}$ is the $\epsilon$-blowup of $A$, with $d(x,A) := \inf \{  \| x - y \| : y \in A \}$.  The \textit{Cheeger constant} $h(\mu)$ of $\mu$ is  defined as
\[
h(\mu) := \inf_{A\subset \R^d}
\frac{\mu^+(\partial A)}{\min\{\mu(A),\mu(A^c)\}},
\]
which serves as a measure of bottleneckedness for $\mu$. Indeed, a small $h(\mu)$ indicates the existence of a measurable $A\subset \R^d$ whose boundary measure is much smaller than the measure of $A$ and $A^c$ themselves. If $\mu$ has density $f$, then we also write $h(f) = h(\mu)$. 

\medskip

If $d=1$ and $\mu$ has density $f$ with distribution function $F$, then the Cheeger constant admits the simplified expression \cite[Theorem 1.3]{bobkov1997}
\[
h(\mu) = \essinf_{x \in \R}\frac{f(x)}{\min \{ F(x),1-F(x) \}}.
\]
Furthermore, if $F$ is strictly increasing around $x$, then for $t = F(x)$, we have 
\[
\frac{f(x)}{\min \{ F(x),1-F(x) \}} = \frac{f(F^{-1}(t))}{\min \{t,1-t\}}%
\]
The numerator on the right-hand side (RHS) is denoted by $I(t):=f(F^{-1}(t))$; lower bounding this function plays a key role in our empirical convergence rate analysis. The main observation in that regard is that if $f$ is log-concave in $d=1$, then $\{ x: 0 < F(x) < 1 \}$ is an interval and $f$ is positive on the interval, which implies $I(t) \ge h(f)\min \{ t,1-t\}$ for $t \in (0,1)$. 

\medskip
Consequently, lower bounding $I(t)$ reduces to controlling $h(f)$ from below.
In general, it is known from \cite{kannan1995isoperimetric} that if $f$ is a log-concave density on $\R^d$ with covariance matrix $\Sigma$, then there exists a constant $c_d > 0$ that depends only on $d$ such that
\be
\label{eq: cheeger_lower_bound}
h(f) \ge \frac{c_{d}}{\| \Sigma \|_{\op}^{1/2}}. 
\ee
The KLS conjecture \cite{kannan1995isoperimetric,lee2018kannan} states that $c_d$ can be chosen to be independent of $d$. The best available result up to date is due to \cite{chen2021almost}, which shows that $c_d = 1/d^{o_d(1)}$ as $d \to \infty$.

The proof of the concentration inequalities in Proposition~\ref{prop: SWp concentration} below requires another property of log-concave distributions, namely the fact that they satisfy Poincar\'e inequalities. A probability measure $\mu \in \cP(\R^d)$ is said to satisfy a Poincar\'e inequality with constant $M_\mu > 0$ if 
\be
\label{eq: poincare}
\Var_\mu(f) \leq M_\mu \EE[\|\nabla f\|^2]
\ee
for any function $f: \R^d \to \R$ such that both sides of the above display are finite.
The Maz'ya-Cheeger theorem (Theorem 1.1 in \cite{milman2009role}) yields that $1/M_\mu \geq h(\mu)/2 > 0$, so that any (nondegenerate) log-concave distribution automatically satisfies a Poincar\'{e} inequality.

\section{Concentration Inequalities }\label{APPEN: SWp concentration}

We present concentration bounds for the empirical sliced distances as a corollary of Theorem~\ref{thm: SWp rate}. This result is utilized to provide global guarantees for computing $\MSWp$ via the LIPO algorithm \cite{malherbe2017global} (cf. Proposition \ref{prop: MSWp_LIPO_convergence} in Appendix \ref{APPEN:LIPO}). %

\begin{proposition}[Concentration inequalities]%
\label{prop: SWp concentration}
Let $1 \le p < \infty$ and $n \ge 2$, and assume that $\mu \in \cP(\R^d)$ is log-concave with non-singular covariance matrix $\Sigma$. Then, for any $t > 0$, 
\begin{subequations}
\begin{align}
&\PP \Bigg ( \SWp(\empmu, \mu) \geq \frac{ \big(\|\Sigma\|_{\op} (\log n)^{\ind_{\{p = 2\}}} \big)^{1/2}}{ n^{1/(2 \vee p)} } + t \Bigg ) \leq 2 \exp \left ( -K\min \left\{ n^{1/p} t, n^{2/(2 \vee p)} t^2 \right\} \right ), \label{eq: SWp concentration}\\
&\PP \Bigg ( \MSWp(\empmu, \mu) \geq \alpha_{n,\mu} + t \Bigg ) \leq 2 \exp \left ( -K \min \left \{ n^{1/p} t, n^{2/(2 \vee p)} t^2 \right \} \right ),\label{eq: MSWp concentration}
\end{align}
\end{subequations}
where $K \lesssim d^{o_d(1)} \max\{\|\Sigma\|_{\op}^{1/2},\,\|\Sigma\|_{\op} \}$ and $\alpha_{n,\mu}$ is defined by the RHS of \eqref{eq: mswp bound} with $k=d$.
\end{proposition}

The proof of Proposition \ref{prop: SWp concentration} combines the expectation bounds from Theorem \ref{thm: SWp rate} with the concentration inequality for empirical sliced Wasserstein distances from \cite[Theorem 3.8]{lin2021projection}. The latter result holds under a Poincar\'e inequality assumption on the population distribution, which is always satisfied for log-concave measures (cf. \cite[Theorem 1.1]{milman2009role}), and is hence applicable for our setting.

The proof proceeds by lower bounding the Poincar\'e constant of $\mu$, and then using a concentration result with expectation centering in \cite{lin2021projection} that relies on the Poincar\'e constant, combined with our expectation bounds (Theorem~\ref{thm: SWp rate}).  By assumption, $\mu$ is log-concave with covariance matrix $\Sigma_\mu$. This, in particular, implies that \eqref{eq: cheeger_lower_bound} holds, with $c_d = d^{o_d(1)}$ (cf. Theorem 1 in \cite{chen2021almost}). Combined with the Maz'ya-Cheeger inequality (Theorem 1.1 in \cite{milman2009role}), this gives the following bound for the Poincar\'e constant $M_\mu$ of $\mu$:
\[
\frac{1}{M_\mu} \geq \frac{h(\mu)}{2} \geq \frac{1}{2 d^{o_d(1)} \|\Sigma\|_{\op}}.
\]
Now, by Theorem 3.8 in \cite{lin2021projection}, we have 
\[
\PP\big(|\rho(\empmu, \mu) - \EE[\rho(\empmu, \mu)]| \geq t\big) \leq 2 \exp \left ( -K \min \left\{n^{1/p} t,\, n^{2/(2 \vee p)} t^2 \right\} \right ), \ t > 0,
\]
where $\rho = \SWp$ or $\MSWp$, and $K$ depends only on $M_\mu$. A careful review of the proof of Theorem 3.8 and intermediate results in \cite{ledoux1999concentration} yields that $1/\min \{ 2 \sqrt{M_\mu}, 6e^5 M_\mu \}$ is a valid choice of $K$ in the above display, so that $K \lesssim d^{o_d(1)} \max\{\|\Sigma\|_{\op}^{1/2},\,\|\Sigma\|_{\op} \}$.
Plugging \eqref{eq: swp bound} and \eqref{eq: mswp bound}  into the above display completes the proof.

\section{Global Guarantees for Max-Sliced {\boldmath $\Wp$} Computation via LIPO}\label{APPEN:LIPO}
We can compute $\MSWp(\empmu, \empnu)=\max_{\theta\in\BB^d}\Wp (\proj^{\theta}_\sharp \empmu, \proj^{\theta}_\sharp \empnu )$ itself via the LIPO algorithm~\cite{malherbe2017global}, which performs global optimization of Lipschitz functions over convex domains based on function evaluations. LIPO sequentially chooses the next evaluation point only if it can increase the function value, based on the Lipschitz condition. Setting $\hat{w}_p(\theta):=\Wp \big (\proj^{\theta}_\sharp \empmu, \proj^{\theta}_\sharp \empnu \big)$ and tuning LIPO to the (empirical) Lipschitz constant $\hat L_n^p:=\sup_{\theta \in \unitsph} \big[ (\hat{\mu}_n|\theta^\intercal x|^p)^{1/p} + (\hat{\nu}_n |\theta^\intercal x|^p)^{1/p}\big]$ (see Lemma~\ref{lem: wp_theta_lipschitz}), if $\Theta_1,\ldots,\Theta_t$ are the $t$ previous evaluation points, the next evaluation will be at $\Theta_{t+1}$ provided that
\[
\min_{1\leq i\leq t} \big\{\hat{w}_p(\Theta_i) + \hat L_n^p\|\Theta_{t+1} - \Theta_i\| \big\} \geq \max_{1\leq i \leq t} \hat{w}_p(\Theta_i).\]
The output after $k$ steps is $\max_{1\leq i\leq k} \hat{w}_p(\Theta_i)$. See \cite[Figure~1]{malherbe2017global} for the full pseudo-algorithm. We have the following global guarantee for the performance of LIPO. %

\begin{proposition}[LIPO error bound]
\label{prop: MSWp_LIPO_convergence}
Let $1\leq p<\infty$ and assume that $\mu,\nu\in\cP_p(\R^d)$ are log-concave with non-singular covariance matrices $\Sigma_\mu$ and $\Sigma_\nu$, respectively. Let $\Theta_1,\ldots\Theta_k$ be a sequence of points generated by the LIPO for computing $\max_{\theta \in \unitsph} \hat{w}_p(\theta)$. Then for any $t > 0$ and $n \geq C_p d^{p/2}$, we have
\[
\PP \left(\Big|\,\MSWp(\mu,\nu) - \max_{1 \leq i \leq k} \hat{w}_p(\Theta_i)\Big| \leq 2 L_{\mu, \nu}  \left ( \frac{\log (1/\delta)}{k} \right )^{1/d}+ \alpha_n +2t\right) \geq 1 - \delta - \beta - \gamma_n(t)
\]
where $\alpha_n = \alpha_{n,\mu} + \alpha_{n,\nu}$ with $\alpha_{n,\mu}$ given by the RHS of \eqref{eq: mswp bound} with $k=d$, $\alpha_{n, \nu}$ defined analogously, 
\begin{align*}
    L_{\mu, \nu}&= (\|\Sigma_\mu\|_{\op}^{1/2} + \|\Sigma_\nu\|_{\op}^{1/2}) \left ( (2p)^{1/p} \vee 2  + \frac{1}{2}  \right ) + \|\mu x\|  + \|\nu x \|,\\
    \beta &= \exp(-c_p \sqrt{d}), \quad \text{and} \\
    \gamma_n(t)&=2 \exp \left ( -K_\mu \min \left ( n^{1/p} t, n^{2/(2 \vee p)} t^2 \right ) \right ) + 2 \exp \left ( -K_\nu \min \left ( n^{1/p} t, n^{2/(2 \vee p)} t^2 \right ) \right ),
\end{align*}
with $K_\mu \lesssim d^{o_d(1)} \max\{\|\Sigma_\mu\|_{\op}^{1/2},\,\|\Sigma_\mu\|_{\op} \}$ and $K_\nu \lesssim d^{o_d(1)} \max\{\|\Sigma_\nu\|_{\op}^{1/2},\,\|\Sigma_\nu \|_{\op} \}$. 

\end{proposition}

The proof of Proposition \ref{prop: MSWp_LIPO_convergence} is given in Appendix~\ref{subsection: MSWp_LIPO_convergence_proof}. The analysis separately bounds the empirical approximation error of the max-sliced objective and the error due to LIPO. The empirical error is treated using the concentration inequality from Proposition \ref{prop: SWp concentration}. For the LIPO analysis, we first argue that the (random) Lipschitz constant $\hat L_n^p$ concentrated about its mean and bound the latter by the population Lipschitz constant $L_{\mu,\nu}^p$. With this deterministic bound, the result follows from \cite[Corollary 13]{malherbe2017global}. Evidently, while Proposition \ref{prop: MSWp_LIPO_convergence} provides a global optimality guarantee, the resulting rate depends exponentially on dimension, which is too conservative in high-dimensional settings. %

\section{Proofs of Results in the Main Text}

\textbf{Additional notation}: We use $N(\epsilon, \cF, d)$ to denote the $\epsilon$-covering number of a function class or set $\cF$ with respect to (w.r.t.) a metric $d$ on $\cF$, and $N_{[\,]}(\epsilon, \cF, d)$ denotes the corresponding bracketing number.

\subsection{Proof of Theorem \ref{thm: SWp rate}}\label{APPEN: SWP rate_proof}

The proof relies on \cite[Theorem 6.6]{bobkov2019one}, restated below, that bounds empirical convergence rates for $\Wp$ between distributions on $\RR$. 

\begin{lemma}[Theorem 6.6 in \cite{bobkov2019one}]
\label{prop: wp bound}
Fix $1 \le p < \infty$ and $n \ge 2$. Let $\mu \in \calP(\R)$ have log-concave density $f$ with distribution function $F$. Set $I(t) = f\big(F^{-1}(t)\big)$ for $t \in (0,1)$, where $F^{-1}$ is the quantile function of $F$. Then, 
\begin{equation}
 \EE\big[\Wp^p(\hat{\mu}_n,\mu)\big] \leq \left(\frac{C p^2}{n}\right)^{p/2} \int_{1/(n+1)}^{n/(n+1)}\frac{\big(t(1-t)\big)^{p/2}}{I^p(t)}\, \, dt  ,\label{EQ:Wp_bound_geq2}
\end{equation}
where $C$ is a universal constant. 
\end{lemma}

We will apply Lemma \ref{prop: wp bound} to $\Wp(\ptheta_{\sharp}\hat{\mu}_n,\ptheta_{\sharp}\mu)$ and bound the corresponding $I$-function from below  \textit{uniformly} over the projection parameter $\theta \in \unitsph$. Recall that the distribution function of $\ptheta_{\sharp}\mu$ is denoted by $F_{\mu}(\cdot;\theta)$, which we abbreviate as $F_\theta$ throughout this proof and denote the corresponding density by $f_\theta$. We first observe that since $\mu\in\cP(\RR^d)$ is log-concave, then so is $\ptheta_{\sharp}\mu$ for any $\theta\in\unitsph$. 

Let $h_\theta:=h(\mathfrak{p}^\theta_\sharp \mu)$ denote the Cheeger constant of the projected distribution. From the discussion in Appendix \ref{APPEN:log_conc_Cheeger}, we know that $1/\left ( f_{\theta}\big(F_{\theta}^{-1}(t)\big) \right ) \ge h_{\theta} \min \{ t,1-t\}$ for $t \in (0,1)$. Given that, the proof for the $\SWp$ case is relatively straightforward from Lemma \ref{prop: wp bound}. Bounding $\E[\MSWp(\hat{\mu}_n,\mu)]$, however, requires extra work to treat the supremum over $\theta$ that appears inside the expectation. 

\paragraph{$\bm{\SWp}$ case.}
Suppose that $\theta \in \unitsph$ is such that $\theta^\intercal \Sigma \theta = 0$. Then, $\ptheta_\sharp \mu$ degenerates to a point mass, so that $\Wp^p\big(\mathfrak{p}^\theta_\sharp \hat{\mu}_n,\mathfrak{p}^\theta_\sharp \mu\big) = 0$. 

Suppose $\theta^\intercal \Sigma \theta > 0$. Then, $\ptheta_\sharp \mu$ is nondegerate log-concave, so it has a log-concave density. 
Observe that $h_\theta \gtrsim 1/(\theta^\intercal \Sigma\theta)^{1/2}$. If $1 \le p < 2$, then 
\[
\int_{0}^1 \frac{\big(t(1-t)\big)^{p/2}}{t^p \wedge (1-t)^p} \, dt  < \infty,
\]
so that by  Lemma \ref{prop: wp bound}, we have
\[
\EE\big[\Wp^p\big(\mathfrak{p}^\theta_\sharp \hat{\mu}_n,\mathfrak{p}^\theta_\sharp \mu\big)\big] \lesssim  \left(\frac{\theta^\intercal \Sigma \theta}{n}\right)^{p/2}.
\]

If $2 \le p < \infty$, then by Lemma \ref{prop: wp bound}, dividing $\int_{1/(n+1)}^{n/(n+1)}$ into $\int_{1/(n+1)}^{1/2} + \int_{1/2}^{n/(n+1)}$ and using the symmetry, we have
\begin{align*}
\EE\big[\Wp^p\big(\mathfrak{p}^\theta_\sharp \hat{\mu}_n, \mathfrak{p}^\theta_\sharp \mu\big)\big] 
&\lesssim \left(\frac{\theta^\intercal \Sigma \theta}{n}\right)^{p/2} \int_{1/(n+1)}^{1/2} t^{-p/2} \, dt .
\end{align*}
Here
\begin{align*}
\int_{1/(n+1)}^{1/2} t^{-p/2} \, dt 
&=\begin{cases}
\big[\log(n+1) - \log 2\big]&\ p=2\\
\frac{1}{p/2 - 1} \left[(n+1)^{p/2 - 1} - 2^{p/2 - 1}\right]&\ p>2
\end{cases}
,
\end{align*}
so that 
\[
\EE\big[\Wp^p\big(\mathfrak{p}^\theta_\sharp \hat{\mu}_n, \mathfrak{p}^\theta_\sharp \mu\big)\big]  \lesssim \frac{(\theta^\intercal \Sigma \theta)^{p/2}(\log n)^{\ind_{\{p=2\}}}}{n}.
\]
The result follows by noting that $\theta^\intercal \Sigma \theta \leq \|\Sigma\|_{\op}$, integrating the display over $\theta \in \unitsph$ and applying Fubini's theorem. 
\begin{remark}[Better bound for $\SWtwo$]
The above calculation actually yields the slightly better bound 
\[
\EE[\SWtwo^2(\mu,\nu)] \lesssim \frac{k \|\Sigma\|_{\op} \log n }{nd}
\]
for $p = 2$ by using the spectral decomposition $\Sigma = \sum_{i=1}^k \lambda_i a_i a_i^\intercal$, as follows:
\[
\int_{\unitsph} \theta^\intercal \Sigma \theta \, d\sigma(\theta) = \sum_{i=1}^k \int_{\unitsph} (a_i^\intercal \theta)^2\, d\sigma(\theta) = \frac{1}{d} \sum_{i=1}^k \lambda_i \leq \frac{k \|\Sigma\|_{\op}}{d}.
\]
\end{remark}
\paragraph{$\bm{\MSWp}$ case.} 
We divide the proof into two steps. In Step 1, we will prove the claim of the theorem when $k=d$, i.e., $\Sigma$ is of full rank. In Step 2, we reduce the general case to the $d=k$ case.

\underline{Step 1}. Assume $k=\mathrm{rank}(\Sigma)=d$. 
The main idea is to approximate $\EE\big[\MSWp(\hat{\mu}_n,\mu)\big]=\EE\big[\sup_{\theta\in\unitsph}\Wp\big(\mathfrak{p}^\theta_\sharp \hat{\mu}_n, \mathfrak{p}^\theta_\sharp \mu\big)\big]$ by the maximum expected projected distance (roughly speaking, switch the expectation and the supremum). To that end we will employ a covering argument of the unit sphere along with Lipschitz continuity of $\Wp\big(\mathfrak{p}^\theta_\sharp \hat{\mu}_n, \mathfrak{p}^\theta_\sharp \mu\big)$ w.r.t. the samples and $\theta$. These technical results are collected in the following lemma.

\begin{lemma}
\label{lem: lip project}
The following hold:
\begin{enumerate}
    \item[(i)] For any $\epsilon \in (0,1)$, we have $N(\epsilon, \unitsph, \|\cdot\|) \leq (5/\epsilon)^{d}$. 
    \item[(ii)] For any $\gamma \in \calP_p(\R)$, the map $u \mapsto \Wp(n^{-1}\sum_{i=1}^n \delta_{u_i},\gamma)$ with $u= (u_1,\dots,u_n)$ is $n^{-1/(2\vee p)}$-Lipschitz. Further, it is partially differentiable a.e. w.r.t. each $u_i$, and its partial derivative w.r.t. $u_i$ is bounded by $n^{-1/p}$
    \item[(iii)] The map $\theta \mapsto \Wp(\ptheta_{\sharp}\hat{\mu}_n,\ptheta_{\sharp}\mu)$ is $L$-Lipschitz with $L = (\sup_\theta \, \empmu |\theta^\intercal x|^p)^{1/p} + (\sup_\theta \, \mu |\theta^\intercal x|^p)^{1/p}$. If $\mu \in \cP(\R^d)$ is centered and log-concave with non-singular covariance matrix $\Sigma$, then
    \[
    \E[L] \le 2p \|\Sigma\|_{\op}^{1/2} \sqrt{d}
    \]
\end{enumerate}
\end{lemma}
\begin{proof}[Proof of Lemma \ref{lem: lip project}]
(i) Follows from an elementary volumetric argument, which is omitted for brevity.

(ii) By the triangle inequality and definition of $\Wp$, we have 
\[\left|\Wp\left(n^{-1}\sum\nolimits_{i=1}^n \delta_{u_i},\gamma\right) - \Wp\left(n^{-1}\sum\nolimits_{i=1}^n \delta_{u_i'},\gamma\right)\right| \leq  \left(n^{-1}\sum\nolimits_{i=1}^n |u_i - u_i'|^p\right)^{1/p}.
\] 
To bound the RHS by $n^{-1/(2 \vee p)}\| u-u' \|$ we apply Jensen's inequality when $p \le 2$, and using the fact that $\sum_{i=1}^n a_i^{p/2} \le (\sum_{i=1}^n a_i)^{p/2}$ when $p \ge 2$.

The second statement follows from the fact that when coordinates other than $u_i$ are kept fixed, the RHS of the above display is bounded by $\|u_i - u'_i\|$.

(iii) A simpler version is proven in \cite[Lemma 2]{niles2019estimation}, but we include the argument for completeness. 
Applying Lemma~\ref{lem: wp_theta_lipschitz} with $\nu = \empmu$, we obtain Lipschitz continuity with constant $L:=\big(\sup_\theta \mu |\theta ^\intercal x|^p\big)^{1/p}+\big(\sup_\theta \hat{\mu}_n |\theta^\intercal x|^p\big)^{1/p} \leq (\mu \|x\|^p)^{1/p} + (\empmu \|x\|^p)^{1/p}$ so that $\EE[L] \leq 2 (\mu \|x\|^p)^{1/p}$

Since $\mu$ is centered and log-concave with covariance matrix $\Sigma$, in particular
\[
(\mu \|x\|^p)^{1/p} \leq p(\EE[\|X_1\|^2])^{1/2} = p \sqrt{\Tr(\Cov(X_1))} \leq p \sqrt{d} \|\Sigma\|_{\op}^{1/2}.
\]
See, for example, Remark 1 after Theorem 3.1 in \cite{adamczak2014tail}. 
\end{proof}

We are ready to prove the empirical convergence rate of the max-sliced distance. For the remainder of the proof we will assume, without loss of generality, that $\mu$ has mean 0, since $\MSWp(t^v_\sharp \mu, t^v_\sharp \nu) = \MSWp(\mu, \nu)$ for any location shift $t^v :x \mapsto x + v$ and any probability measures $\mu$ and $\nu$, and any location shifted log-concave distribution is also log-concave with the same covariance matrix.  Let $\tilde{w}_p(\theta)= \Wp(\ptheta_{\sharp}\hat{\mu}_n,\ptheta_{\sharp}\mu)$ and observe that
\[
\EE \big [ \MSWp(\hat{\mu}_n,\mu) \big ] \le \sup_{\theta \in \unitsph} \E[\tilde{w}_p(\theta)] + \E\left[\sup_{\theta\in\unitsph} \big(\tilde{w}_p(\theta) - \EE\big[\tilde{w}_p(\theta)\big]\big)  \right ].
\]
From the proof for the average-sliced case, we have 
\begin{equation}
\sup_{\theta \in \unitsph} \E\big[\tilde{w}_p(\theta)\big]  \lesssim_p \frac{\big(\|\Sigma \|_{\op}(\log n)^{\ind_{\{p=2\}}}\big)^{1/2}}{n^{1/(2 \vee p)}}.\label{EQ:MSWp_supE_bound}
\end{equation}

Let $\theta_1,\dots,\theta_{N_\epsilon}$ be a minimal $\epsilon$-net of $\unitsph$, where $N_\epsilon = N(\epsilon, \unitsph, \|\cdot\|)$. Using Lemma \ref{lem: lip project}, we have 
\begin{equation}
\EE\left[\sup_{\theta\in\unitsph} \big(\tilde{w}_p(\theta) - \EE\big[\tilde{w}_p(\theta)\big]\big)\right] \leq \inf_{\epsilon > 0} \EE\left[\max_{1 \le j \le N_\epsilon} \big(\tilde{w}_p(\theta_j) - \EE\big[\tilde{w}_p(\theta_j)\big]\big) + 2\epsilon L\right],\label{EQ:MSWp_discritization_bounds}
\end{equation}
where $L$ is a random variable with $\E[L] \le c_p \|\Sigma\|_{\op}^{1/2}d^{-1/2}$.
  
To control the maximum inside the expectation on the RHS of \eqref{EQ:MSWp_discritization_bounds}, we use an approach based on maximal inequalities for sub-exponential random variables, similar to Theorem 3.5 in \cite{lin2021projection}. Briefly, we will first show that for each $\theta$, $\tilde w_p(\theta)$ is a Lipschitz function of the projected observations $(\theta^{\intercal}X_1,\dots,\theta^{\intercal}X_n)$, with bounded gradient in each coordinate, which will imply sub-exponential concentration for each $\tilde w_p(\theta)$. The term $\max_{1 \le j \le N_\epsilon} \big(\tilde{w}_p(\theta_j) - \EE\big[\tilde{w}_p(\theta_j)\big]\big)$ will then be bounded via a maximal inequality as a direct consequence of this concentration (cf. Exercise 2.8 in \cite{wainwright2019high}).

We will use the following refined concentration inequality for Lipschitz functions of random variables satisfying a Poincar\'e inequality, stated in \cite{ledoux1999concentration}. Since explicit constants are not derived there, a proof is provided in Appendix~\ref{APPEN: poincare_concentration_proof},

\begin{lemma}[Concentration from Poincar\'e inequality; Corollary 4.6 in \cite{ledoux1999concentration}]
\label{lem: poincare_concentration}
Let $\mu \in \cP(\R^d)$ satisfy the Poincar\'e inequality \eqref{eq: poincare} with constant $M_\mu$ and $f: \R^{nd} \to \R$ be $\alpha$-Lipschitz. For $x_1,\ldots,x_n\in\RR^d$, define the functions 
\[
f_i(\cdot|x_1,\ldots,x_{i-1},x_{i+1},\ldots,x_n):=f(x_1,\ldots,x_{i-1},\cdot,x_{i+1},\ldots,x_n),\quad i=1,\ldots,n,
\]
and assume  that $\max_{1\leq i \leq n} \sup_{x\in\RR^d} \|\nabla f_i(x|X_1,\ldots,X_{i-1},X_{i+1},\ldots,X_n)\big\| \leq \beta$ a.s. Then,
\[
\mu^{\otimes n}(f \geq \mu^{\otimes n} f + t) \leq \exp \Bigg ( -\min \left \{ \frac{t}{2\beta M_\mu^{1/2}}, \frac{t^2}{6e^5 \alpha^2 M_\mu} \right \} \Bigg ), \ t > 0.
\]
\end{lemma}

The random vector $(\theta^{\intercal}X_1,\dots,\theta^{\intercal}X_n)$ in $\R^{n}$ has i.i.d. coordinates with law $(\ptheta_{\sharp}\mu)^{\otimes n}$. The distribution $\ptheta_{\sharp}\mu$ is log-concave with variance $\theta^\intercal \Sigma \theta > 0$, which is bounded above by $\|\Sigma\|_{\op}$. This yields that $h_\theta := h(\ptheta_{\sharp}\mu) \gtrsim \|\Sigma\|_{\op}^{-1}$ for each $\theta \in \unitsph$. 
By item (ii) in Lemma~\ref{lem: lip project}, the partial derivatives of $\tilde w_p(\theta)$ w.r.t. $\theta^\intercal X_i$, denoted $\nabla_i \tilde w_p(\theta)$, satisfy $\max_i \|\nabla_i \tilde w_p(\theta)\| \leq n^{-1/p}$ a.s., and $\tilde w_p(\theta)$ is $n^{-1/(2 \vee p)}$-Lipschitz in $(\theta^{\intercal}X_1,\dots,\theta^{\intercal}X_n)$. By the Maz'ya-Cheeger Theorem (cf. Theorem 1.1 in \cite{milman2009role}), $M_{\ptheta_\sharp \mu}^{-1} \geq h_\theta/2 \geq \|\Sigma_\mu\|_{\op}^{-1}/2$. Combining these facts and applying Lemma~\ref{lem: poincare_concentration}, we have
\begin{align*}
    \PP \big ( \tilde{w}_p(\theta) - \EE[\tilde{w}_p(\theta)] > t \big ) &\leq \exp \Bigg ( -\min \left ( \frac{t}{\sqrt{2} n^{-1/p} \|\Sigma\|_{\op}^{1/2}}, \frac{t^2}{3e^5 n^{-2/(2 \vee p)} \|\Sigma\|_{\op}} \right ) \Bigg )\\
    &\leq \exp \Bigg ( -\frac{t^2}{\sqrt{2} n^{-1/p} \|\Sigma\|_{\op}^{1/2} t + 3e^5 n^{-2/(2 \vee p)} \|\Sigma\|_{\op}} \Bigg ), \ t > 0.
\end{align*}

A simple union bound then gives
\[
\PP \Big ( \max_{1 \le j \le N_\epsilon} \big(\tilde{w}_p(\theta_j) - \EE\big[\tilde{w}_p(\theta_j)\big]\big) > t \Big ) \leq N_\epsilon \exp \Bigg ( -\frac{t^2}{\sqrt{2} n^{-1/p} \|\Sigma\|_{\op}^{1/2} t + 3e^5 n^{-2/(2 \vee p)} \|\Sigma\|_{\op}} \Bigg ),
\]
which, by an expectation bound for sub-exponential random variables (cf. Exercise 2.8 in \cite{wainwright2019high}), yields
\begin{align*}
&\EE \left [ \max_{1 \le j \le N_\epsilon} \big(\tilde{w}_p(\theta_j) - \EE\big[\tilde{w}_p(\theta_j)\big]\big) \right ] \\
&\leq \sqrt{6e^5 n^{-2/(2 \vee p)} \|\Sigma\|_{\op}} (\sqrt{\pi}  + \sqrt{\log N_\epsilon}) + 2\sqrt{2} n^{-1/p} \|\Sigma\|_{\op}^{1/2}(1 + \log N_\epsilon)\\
&\lesssim \|\Sigma\|_{\op}^{1/2} n^{-1/(2 \vee p)} (1 + \sqrt{\log N_\epsilon}) + \|\Sigma\|_{\op}^{1/2} n^{-1/p} (1 + \log N_\epsilon).\numberthis\label{eq:wp_maximal_ineq}
\end{align*}

By Lemma 3 (i), $\log N_\epsilon \leq d \log(5/\epsilon)$. Thus, setting $\epsilon = n^{-1/2}$ and plugging \eqref{eq:wp_maximal_ineq} into \eqref{EQ:MSWp_discritization_bounds}, we have
\[
\EE\big[\MSWp(\hat{\mu}_n,\mu)\big] \lesssim_p \|\Sigma\|_{\op}^{1/2} \left ( \frac{(\log n)^{\ind_{\{p=2\}}}}{n^{1/(2\vee p)}} +  \frac{(1 + \sqrt{d \log n})}{n^{1/(2 \vee p)}} +  \frac{ (1 + d \log n)}{n^{1/p}} + \frac{d^{1/2}} {n^{1/2}} \right ).
\]
The last term on the RHS of the above display is of smaller order in $n$ and $d$ than the other two terms. Further, for $n \geq 2$, we have $(1 + \sqrt{d \log n}) \lesssim \sqrt{d \log n}$ and $(1 + d \log n) \lesssim d \log n$. This leads to the bound stated in Theorem~\ref{thm: SWp rate} when $k=d$. 

\underline{Step 2}. Suppose now that $1 \le k < d$. Again assume without loss of generality that the mean of $\mu$ is zero. Observe that $\MSWp$ is invariant under common orthogonal transformations, i.e., for any $d \times d$ orthogonal matrix $Q$, $\MSWp(Q_{\sharp}\mu,Q_{\sharp}\nu) = \MSWp(\mu,\nu)$. With this in mind, we see that we may assume without loss of generality that $\Sigma$ is diagonal whose first $k$ diagonal entries are nonzero. Then, for $X = (X_1,\dots,X_d)^\intercal \sim \mu$ and $\theta = (\theta^1,\dots,\theta^d)^\intercal \in \unitsph$, $\theta^\intercal X = \sum_{j=1}^k \theta^j  X_j$ a.s. Thus, we have
\[
\sup_{\theta \in \unitsph} \Wp(\ptheta_{\sharp}\hat{\mu}_n,\ptheta_{\sharp}\mu)=\sup_{\substack{\theta =  (\theta^1,\dots,\theta^d)^\intercal \in \unitsph \\ \theta^{k+1}=\cdots=\theta^d = 0}} \Wp(\ptheta_{\sharp}\hat{\mu}_n,\ptheta_{\sharp}\mu) \quad \text{a.s.} 
\]
The bound stated in Theorem~\ref{thm: SWp rate} follows by the argument in Step 1 with $d$ replaced by $k$. 
\qed

\subsection{Proof of Proposition~\ref{thm: swp_rate_alt}}
\label{APPEN: SWP_alt_rate_proof}

\paragraph{Upper bound for $\SWp^p$.} Let $F_\mu(t, \theta) = \PP(\theta^\intercal X \leq t)$ for $X \sim \mu$, and analogously define $F_\nu(t, \theta)$. Then,
\begin{align*}
    \EE\big [ \big| \SWp^p(\empmu, \empnu) - \SWp^p(\mu, \nu) \big| \big ] &\leq \EE \left [ \int_{\unitsph} \Big | \Wp^p(\ptheta_\sharp \empmu, \ptheta_\sharp \empnu) - \Wp^p(\ptheta_\sharp \mu, \ptheta_\sharp \nu) \Big | \ d\sigma(\theta) \right ]\\
    &\leq C_{p,R}\ \EE \left [ \int_{\unitsph} \Big ( \wass(\ptheta_\sharp \empmu, \ptheta_\sharp \mu) + \wass(\ptheta_\sharp \empnu, \ptheta_\sharp \nu) \Big )\ d\sigma(\theta) \right ]\\
    &\le \frac{C_{p,R}}{\sqrt{n}} \Bigg ( \int_{\unitsph} \int_{-\infty}^\infty \sqrt{F_{\mu}(t,\theta)(1-F_{\mu}(t,\theta))}\ dt\,d\sigma(\theta) \\
    &\qquad+ \int_{\unitsph} \int_{-\infty}^\infty \sqrt{F_{\nu}(t,\theta)(1-F_{\nu}(t,\theta))}\ dt\,d\sigma(\theta) \Bigg ) \\
    &\leq \frac{R C_{p,R}}{\sqrt{n}},
\end{align*}
where the second inequality follows from a comparison between $\Wp$ and $\wass$ for compactly supported distributions (Lemma 4 in \cite{goldfeld2022statistical}), the third from the integral representation of $\wass$, and the final inequality from truncating the inner integrals to $[-R, R]$ and observing that $p(1-p) \leq 1/4$ for $p \in [0,1]$. 

\paragraph{Upper bound for $\MSWp^p$. } As earlier, observe that
\begin{align*}
    \EE\big [ \big| \MSWp^p(\empmu, \empnu) - \MSWp^p(\mu, \nu) \big| \big ] &\leq \EE \left [ \sup_{\theta \in \unitsph} \Big | \Wp^p(\ptheta_\sharp \empmu, \ptheta_\sharp \empnu) - \Wp^p(\ptheta_\sharp \mu, \ptheta_\sharp \nu) \Big | \right ]\\
    &\leq C_{p,R}\ \EE \left [ \sup_{\theta \in \unitsph} \Big ( \wass(\ptheta_\sharp \empmu, \ptheta_\sharp \mu) + \wass(\ptheta_\sharp \empnu, \ptheta_\sharp \nu) \Big )\right ]\\
    &\leq C_{p,R}\,\Bigg ( \mspace{-2mu}\EE \left [ \sup_{\theta \in \unitsph}  \wass(\ptheta_\sharp \empmu, \ptheta_\sharp \mu) \right ] \mspace{-5mu}+\mspace{-5mu} \EE \left [ \sup_{\theta \in \unitsph}  \wass(\ptheta_\sharp \empnu, \ptheta_\sharp \nu) \right ] \mspace{-2mu}\Bigg ).
\end{align*}
Now, $\EE \left [ \sup_{\theta \in \unitsph}  \wass(\ptheta_\sharp \empmu, \ptheta_\sharp \mu) \right ]$ admits the following dual representation via KR duality:
\be
\label{eq:mswone_emp_process}
\EE \left [ \sup_{\theta \in \unitsph}  \wass(\ptheta_\sharp \empmu, \ptheta_\sharp \mu) \right ] = \EE \Bigg [ \sup_{f \in \mathsf{Lip}_{1,0}(\R),\ \theta \in \unitsph} (\empmu - \mu)(f \circ \ptheta) \Bigg ],
\ee
where $\mathsf{Lip}_{1,0}(\R) = \{f:\R \to \R \,:\, |f(x) - f(y)| \leq |x - y| \ \forall x,y \in \R, f(0) = 0 \}$. 
By Lemma 8 in \cite{goldfeld2022statistical}, the function class $\cG = \{f \circ \ptheta: \theta \in \unitsph, f \in \mathsf{Lip}_{1,0}(\R)\}$ is $\mu$-Donsker, and we have
\[
\log N_{[\,]}(\epsilon, \cG, L^2(\mu)) \lesssim R^{5/3} \epsilon^{-3/2} + 5d \log (R/\epsilon),
\]
which, by the global maximal inequality (Theorem 2.14.2 in \cite{van1996weak}), gives,
\begin{equation}
\label{eq:MSWp-bounded-rate}
\EE \Bigg [ \sup_{f \in \mathsf{Lip}_{1,0}(\R),\ \theta \in \unitsph} (\empmu - \mu)(f \circ \ptheta) \Bigg ] \lesssim \frac{R^{5/4} + d \log R}{\sqrt{n}}.
\end{equation}
Combining this with \eqref{eq:mswone_emp_process}, and repeating the same argument for $\nu$, we have the second statement.

The final statement of the theorem on $\SWp$ and $\MSWp$ follows from the first two upper bounds combined with the elementary inequality $|a - b| \leq b^{1-p}|a^p - b^p|$ for $a \geq 0, b > 0, p \geq 1$. \qed

\subsection{Proof of Lemma~\ref{lem: poincare_concentration}}
\label{APPEN: poincare_concentration_proof}
The proof of this lemma essentially recovers constants in Corollary 4.6 in \cite{ledoux1999concentration}, but a full argument is included  for completeness. With some abuse of notation, let $\|\nabla f\|_\infty = \sup_x \|\nabla f(x) \|$. By Theorem 4.5 in \cite{ledoux1999concentration}, for any $\lambda$-Lipschitz function $f$ with $\lambda \leq 2/\sqrt{M_\mu}$, we have
\begin{equation}
\mathsf{Ent_\mu}(e^f) \leq B(\lambda) \EE\big[\|\nabla f\|_\infty^2 e^f\big],\label{EQ:Ent_Lip}
\end{equation}
where $\mathsf{Ent_\mu}(f):= \EE[f \log f]$ is the entropy functional of $f$, and
\[
B(\lambda) \leq \frac{M_\mu}{2} \left ( \frac{2 + \lambda \sqrt{M_\mu}}{2 - \lambda \sqrt{M_\mu}}\right ) e^{\sqrt{5 M_\mu} \lambda}.
\]

Each function $f_i$  in the statement of the proposition is $\beta$-Lipschitz. From \eqref{EQ:Ent_Lip} together with the tensorization property of the entropy functional (cf. Proposition 2.2 in \cite{ledoux1999concentration}), we obtain %
\[
\mathsf{Ent_{\mu^{\otimes n}}}\left(\frac{\lambda f}{\beta}\right) \leq \frac{\lambda^2}{\beta^2}\sum_{i=1}^n \EE\left[\mathsf{Ent_\mu}\left(\frac{\lambda f_i}{\beta}\right)\right] \leq \frac{\lambda^2 B(\lambda)}{\beta^2} \sum_{i=1}^n \EE\big[\|\nabla f_i\|_\infty^2 e^f\big],\quad \forall \lambda \in \big(0,2/\sqrt{M_\mu}\big].
\]
Further, we have $B(\lambda) \leq \frac{3e^5 M_\mu}{2}$ for $\lambda \leq 1/\sqrt{M_\mu}$ and $\sum_{i=1}^n \|\nabla f_i \|_\infty^2 \leq \beta$ $\mu^{\otimes n}$-a.e. by assumption. Therefore
\[
\mathsf{Ent_{\mu^{\otimes n}}}\left(\frac{\lambda f}{\beta}\right)  \leq \frac{3e^5 M_\mu \alpha^2}{2\beta^2}  \lambda^2 \EE\left[e^\frac{\lambda f}{\beta}\right], \quad \forall \lambda \in \big(0,1/\sqrt{M_\mu}\big].
\]
By Corollary 2.11 in \cite{ledoux1999concentration}, this yields 
\[
\mu^{\otimes n} \left ( \frac{f}{\beta} > \mu \left ( \frac{f}{\beta}\right) + r \right ) \leq \exp \left ( - \min \left \{ \frac{r}{2\sqrt{M_\mu}}, \frac{r^2 \beta^2}{6 M_\mu e^5\alpha^2} \right \} \right ), \ r > 0,
\]
from which the result follows by replacing $r$ with $r/\beta$. \qed

\subsection{Proof of Theorem \ref{thm:robustness}}\label{prf:robustness}

Denoting by $\land$ the setwise minimum of two measures, we first recall some useful facts. Throughout we write $\Theta \sim \Unif(\unitsph)$ for a random direction on the sphere sampled independently of any other randomness.
\begin{fact}
\label{fact:shared-mass}
For $\mu,\nu \in \cP(\R^d)$ and $p \geq 1$, we have $\Wp(\mu,\nu) \leq \Wp(\mu - \mu \land \nu, \nu - \mu \land \nu)$.
\end{fact}
This follows by infimizing over transport plans which leave the shared mass $\mu \land \nu$ unmoved.
\begin{fact}
\label{fact:homogeneity}
For $\mu,\nu \in \cP(\R^d)$, $p \geq 1$, and $c \geq 0$, we have $\Wp(c\mu,c\nu) = c^{1/p}\Wp(\mu,\nu)$.
\end{fact}
It is easy to check that these properties extend to $\SWp$ and $\MSWp$. We also employ the following.

\begin{lemma}
\label{lem:sphere-coordinate}
Fixing $p \geq 1$ and $\mu = \Unif(\unitsph)$, we have 
\begin{equation*}
    (\mu |x_1|^p)^{1/p} \asymp \sqrt{1 \land p/d}.
\end{equation*}
\end{lemma}
\begin{lemma}
\label{lem:moment-comparison}
Fixing $p \geq 1$ and $\mu \in \cP_p(\R^d)$, we have
\begin{align*}
    \mu(\|x\|^p)^{1/p} = \Wp(\mu,\delta_0) &\asymp \sqrt{1 \lor d/p}\:\, \SWp(\mu,\delta_0).
\end{align*}
\end{lemma}

We defer proofs of the previous lemmas to Appendix \ref{app:auxiliary-lemmas}. In what follows, we will refer to any $\sD:\cP(\R^d)^2 \to [0,\infty]$ as a \emph{statistical distance}, specifying additional properties as needed. Our risk bounds use the following standard lemma (see e.g. \cite{donoho1988robustness}), with a proof provided for completeness.

\begin{lemma}
\label{lem:modulus}
For any statistical distance $\sD$, corruption fraction $\eps \in [0,1]$, and clean family $\cG \subseteq \cP(\R^d)$, define the modulus of continuity
\begin{equation}
\label{eq:modulus}
    \mathfrak{m}(\sD,\cG,\eps) = \sup_{\substack{\mu,\nu \in \cG \\ \|\mu - \nu\|_\tv \leq \eps }} \sD(\mu,\nu).
\end{equation}
We then have
\begin{equation*}
    \frac{1}{2}\mathfrak{m}(\sD,\cG,\eps) \leq R(\sD,\cG,\eps) \leq \mathfrak{m}(\sD,\cG,2\eps).
\end{equation*}
\end{lemma}
\begin{proof}
For the lower bound, take any $\mu,\nu$ feasible for \eqref{eq:modulus}. Then, if the statistician observes $\eps$-contaminated measure $\tilde{\kappa} = \mu$, the clean measure could potentially be either $\kappa = \mu$ or $\kappa = \nu$. Hence any estimate $T(\tilde{\kappa})$ for the clean measure $\kappa$ must incur error at least $D(\mu,\nu)/2$ in the worst case. For the upper bound, consider $T$ which projects $\tilde{\kappa}$ onto $\cG$ in TV. Then, $\|T(\tilde{\kappa}) - \kappa\|_\tv \leq \|T(\tilde{\kappa}) - \tilde{\kappa}\|_\tv + \|\tilde{\kappa} - \kappa\|_\tv \leq 2\eps$, and so $D(T(\tilde{\kappa}),\kappa) \leq \mathfrak{m}(\sD,\cG,2\eps)$ by definition.
\end{proof}

By rescaling $\R^d$ appropriately, it is easy to check that $\mathfrak{m}(\sD, \cG_q(\sigma),\eps) = \sigma \mathfrak{m}(\sD, \cG_q(1),\eps)$ for our choices of $\sD$, so we will assume $\sigma = 1$ and write $\cG_q = \cG_q(1)$ from now on.

\subsubsection{Lower bounds}
Immediately, we can apply Lemma \ref{lem:modulus} to obtain the lower bounds of Theorem \ref{thm:robustness}.
\begin{proposition}
Fix $1 \leq p < q$ and corruption fraction $\eps \in [0,1/2]$. Then we have
\begin{align*}
    R(\SWp,\cG_q,\eps) &\gtrsim \sqrt{(1 \lor d/q)(1 \land p/d)}\: \eps^{1/p-1/q}\\
    R(\MSWp,\cG_q,\eps) &\gtrsim \eps^{1/p-1/q}.
\end{align*}
\end{proposition}
\begin{proof}
For $\MSWp$, we consider $\mu = \delta_0$ and $\nu = (1-\eps)\delta_0 + \eps\delta_y$ where $\|y\| = (2\eps)^{-1/q}$. Trivially, $\mu \in \cG_q$, and
\begin{align*}
    \sup_{\theta \in \unitsph} \nu |\theta^\intercal(x - \nu y)|^q &= \sup_{\theta \in \unitsph} (1-\eps)\eps^q|\theta^\intercal y|^q + \eps (1-\eps)^q|\theta^\intercal y|^q\\
    &= \|y\|^q \left[(1-\eps)\eps^q + \eps(1-\eps)^q\right]\\
    &\leq \frac{1}{2} \eps^{-1} 2\eps(1-\eps)^q \leq 1,
\end{align*}
so $\nu \in \cG_q$ as well. Moreover, we have
\begin{align*}
    \MSWp(\mu,\nu) = \eps^{1/p} \MSWp(\delta_0,\delta_y) = 2^{-1/q} \eps^{1/p-1/q} \geq \frac{1}{2} \eps^{1/p-1/q},
\end{align*}
and so Lemma \ref{lem:modulus} gives the desired risk bound for $\MSWp$. 

For $\SWp$, we fix $\mu = \delta_0$ and set $\nu = (1-\eps)\delta_0 + \eps \Unif(r \unitsph)$ with $r = \eps^{-1/q} \MSWq(\Unif(\unitsph),\delta_0)^{-1}$. As before $\mu,\nu \in \cG_q$, since
\begin{align*}
    \MSWq(\nu,\delta_{\nu x}) &= \eps^{1/q} \, \MSWq(\Unif(r \unitsph),\delta_0)\\
    &= r \eps^{1/q} \, \MSWq(\Unif(\unitsph),\delta_0) = 1.
\end{align*}
Furthermore, we have
\begin{align*}
    \SWp(\mu,\nu) &= \eps^{1/p}\,\MSWp(\Unif(r \unitsph), \delta_0)\\
    &= \eps^{1/p-1/q} \MSWq(\Unif(\unitsph),\delta_0)^{-1} \,\MSWp(\delta_0, \Unif(\unitsph))\\
    &\asymp \sqrt{(1 \lor d/q)(1 \land p/d)} \, \eps^{1/p - 1/q},
\end{align*}
where the last relation uses Lemma \ref{lem:sphere-coordinate}. Again, we obtain the desired risk bound via Lemma \ref{lem:modulus}.
\end{proof}

\subsubsection{Upper bounds}

Next, we introduce an important notion of \emph{(generalized) resilience} \cite{steinhardt2018resilience,zho2019resilience}. We say that a distribution $\mu \in \cP(\R^d)$ is $(\rho,\eps)$-resilient w.r.t.\ a statistical distance $\sD$ if $\sD(\mu,\nu) \leq \rho$ for all distributions $\nu \leq \frac{1}{1-\eps}\mu$ (i.e.\ for all $\eps$-deletions of $\mu$). Standard (mean) resilience refers to resilience w.r.t.\ $\sD_\mathrm{mean}(\mu,\nu) = \|\mu x - \nu x\|$. Writing $\cG^\sD_{\rho,\eps} \subset \cP(\R^d)$ for the family of $\mu \in \cP(\R^d)$ which are $(\rho,\eps)$-resilient w.r.t. $\sD$, we have the following standard result.

\begin{proposition}
\label{prop:risk-bound-via-resilience}
Fix $0 \leq \eps < 1/2$, $\rho \geq 0$, and $\sD$ satisfying the triangle inequality. Then, we have $R(\sD,\cG^\sD_{\rho,2\eps},\eps) \leq 2\rho$.
\end{proposition}
\begin{proof}
Fix $\mu,\nu \in \cG^\sD_{\rho,2\eps}$ with $\|\mu - \nu\|_\tv \leq 2\eps$. We consider the midpoint $\gamma = {\frac{1}{1-\|\mu - \nu\|_\tv} \mu \land \nu} \in \cP(\R^d)$ and compute
\begin{equation*}
    \sD(\mu,\nu) \leq \sD(\mu,\gamma) + \sD(\nu,\gamma) \leq 2\rho,
\end{equation*}
implying the desired risk bound via Lemma \ref{lem:modulus}.
\end{proof}

We will also use the following standard result for one-dimensional (mean) resilience (see e.g.\ \cite[Proposition 23]{steinhardt2018resilience}), which is a consequence of Markov's inequality.
\begin{lemma}
\label{lem:moment-mean-resilience}
Fix $0 \leq \eps \leq 1/2$ and $\mu \in \cP(\R)$ with $\mu|x - x_0|^p \leq \sigma^p$ for some $x_0 \in \R$. Then, for all distributions $\nu \leq \frac{1}{1-\eps}\mu$, we have $|\mu x - \nu x| \lesssim \sigma \eps^{1-1/p}$.
\end{lemma}

Next, for $\sD \in \{\SWp,\MSWp\}$, we show that it suffices to prove resilience with respect to the simpler distances defined by
\begin{align*}
    \underline{\sD}_p(\mu,\nu) &= \underline{\sD}_p(\mu - \nu) :=  \bigl|\E[(\mu - \nu)(|\Theta^\intercal x|^p)]\bigr| = |\SWp^p(\mu,\delta_0) - \SWp^p(\nu,\delta_0)|\\
    \text{and} \quad  
    \overline{\sD}_p(\mu,\nu) &= \overline{\sD}_p(\mu - \nu) := \sup_{\theta \in \unitsph} \bigl|(\mu - \nu)(|\theta^\intercal x|^p)\bigr|,
\end{align*}
respectively. These distances encode a certain similarity of moment tensors, with $\overline{\sD}_2(\mu,\nu) = \|\Sigma_\mu + (\mu x)(\mu x)^\intercal - \Sigma_\nu - (\nu x)(\nu x)^\intercal\|_{\mathrm{op}}$.

Recall that $\sD = \sD_\cF$ is an integral probability metric (IPM) w.r.t.\ a class $\cF$ of measurable functions on $\R^d$ if $\sD(\mu,\nu) = \sup_{f \in \cF}(\mu - \nu)(f)$. By design, we have the following.

\begin{lemma}
\label{lem:modified-dist-props}
The statistical distances $\underline{\sD}_p$ and $\overline{\sD}_p$ are IPMs w.r.t.\ the function classes $\underline{\cF}_p = \{ x \mapsto c_p s\|x\|^p : s \in \{\pm1\}\}$ and $\overline{\cF}_p = \{ x \mapsto s|\theta^\intercal x|^p : s \in \{\pm1\}, \theta \in \unitsph \}$, respectively, where $c_p = \E[|\Theta_1|^p] \asymp \sqrt{1\land p/d}$. Moreover, $\underline{\sD}_p(\mu,\delta_0) = \SWp^p(\mu,\delta_0)$ and $\overline{\sD}_p(\mu,\delta_0) = \MSWp^p(\mu,\delta_0)$.
\end{lemma}
\begin{proof}
For $\underline{\sD}_p$, we compute
\begin{align*}
    \underline{\sD}_p(\mu,\nu) &= \bigl|\E[(\mu - \nu)(|\Theta^\intercal x|^p)]\bigr|\\
    &= \bigl|(\mu - \nu)(\E|\Theta^\intercal x|^p)\bigr|\\
    &= \bigl|(\mu - \nu)(c_p \|x\|^p)\bigr|\\
    &= \sup_{s \in \{\pm 1\}} (\mu - \nu)(c_p s\|x\|^p).
\end{align*}
Likewise, for $\overline{\sD}_p$, we check
\begin{align*}
    \overline{\sD}_p(\mu,\nu) &= \sup_{\theta \in \unitsph} \bigl|(\mu - \nu)(|\theta^\intercal x|^p)\bigr|\\
    &= \sup_{s \in \{\pm 1\}, \theta \in \unitsph} (\mu - \nu)(s|\theta^\intercal x|^p).
\end{align*}
Computations when $\nu = \delta_0$ are trivial, since there is a single coupling between $\mu$ and $\nu$.
\end{proof}

The third property is particularly relevant to resilience.
\begin{lemma}
\label{lem:resiliency-large-eps}
Let $\sD = \sD_\cF$ be an IPM. Then $\mu$ is $(\rho,\eps)$-resilient w.r.t.\ $\sD$ if and only if $\mu$ is $(\eps(1-\eps)^{-1}\rho,1-\eps)$-resilient w.r.t. $\sD$.
\end{lemma}
\begin{proof}
Writing $\mu = (1-\eps)\nu + \eps \alpha$ for some $\nu,\alpha \in \cP(\R^d)$, we have
\begin{align*}
    \sD(\nu,\mu) &= \sD'(\eps^{-1}[\mu - (1-\eps)\alpha] - \mu)\\
    &= \frac{1-\eps}{\eps} \sD(\mu, \alpha) && \text{(homogeneity)}.\qedhere
\end{align*}
\end{proof}

We now formally translate resilience w.r.t.\ $\underline{\sD}_p$ and $\overline{\sD}_p$ to that which we desire.

\begin{proposition}
\label{prop:resiliency-comparison}
Fix $0 \leq \eps < 1$, $\rho \geq 0$, and $(\sD,\sD') \in \{(\MSWp,\overline{\sD}_p),(\SWp,\underline{\sD}_p)\}$. If $\mu \in \cP(\R^d)$ with $\mu x = 0$ is $(\rho,\eps)$-resilient w.r.t.\ $\sD'$, then $\mu$ is $(2\rho^{1/p} + 2\eps^{1/p}\sD(\mu,\delta_{0}),\eps)$-resilient w.r.t.\ $\sD$.
\end{proposition}
\begin{proof}
Fixing such $\mu$ and taking $\nu \leq \frac{1}{1-\eps}\mu$, write $\mu = (1-\eps)\nu + \eps \alpha$ for some $\alpha \in \cP(\R^d)$ and write $\tau = \eps \land (1-\eps)$, so that $\nu,\alpha \leq \tau^{-1}\mu$. Then, we bound
\begin{align*}
    \sD(\mu,\nu)^p &= \sD((1-\eps)\nu + \eps \alpha, \nu)^p\\
    &\leq \eps \, \sD(\alpha,\nu)^p && \text{(Facts \ref{fact:shared-mass} and \ref{fact:homogeneity})}\\
    &\leq 2^p \eps \sup_{\kappa \leq \tau^{-1}\mu} \sD(\kappa,\delta_{0})^p && \text{(triangle inequality for $\sD$)}\\
    &= 2^p\eps \sup_{\kappa \leq \tau^{-1}\mu} \sD'(\kappa,\delta_{0}) && \text{(Lemma \ref{lem:modified-dist-props})}\\
    &\leq 2^p\eps \sup_{\kappa \leq \tau^{-1}\mu} \sD'(\kappa,\mu) + 2^p\eps\,\sD(\mu,\delta_{0})^p && \text{(triangle inequality for $\sD'$)}\\
    &\leq 2^p \rho + 2^p\eps\,\sD(\mu,\delta_{0})^p, && \text{(Lemma \ref{lem:resiliency-large-eps})}
\end{align*}
giving the desired bound after taking $p$th roots.
\end{proof}

Equipped with this result, we are prepared to prove the upper bounds of Theorem \ref{thm:robustness}. Given $\mu \in \cG_q$, we must provide bounds on $\sD(\mu,\delta_0)$ as well as the resilience of $\mu$ w.r.t.\ $\sD'$.

\begin{lemma}
\label{lem:moment-transfer}
Fixing $1 \leq p < q$ and $\mu \in \cG_q$ with $\mu x = 0$, we have
\begin{align*}
    \SWp(\mu,\delta_0) &\lesssim \sqrt{(1 \lor d/q)(1 \land p/d)}\\
    \MSWp(\mu,\delta_0) &\lesssim 1.
\end{align*}
\end{lemma}
\begin{lemma}
\label{lem:moment-resilience}
Fix $1 \leq p < q$, corruption fraction $0 \leq \eps \leq 1/2$, and $\mu \in \cG_q$ with $\mu x = 0$. Then, $\mu$ is ${(C\sqrt{(1 \lor d/q)(1 \land p/d)}}\eps^{1/p-1/q}, \eps)$-resilient w.r.t.\ $\underline{\sD}_p^{1/p}$ and $(C\eps^{1/p-1/q},\eps)$-resilient w.r.t.\ $\overline{\sD}_p^{1/p}$, for some absolute constant $C > 0$.
\end{lemma}

Together, these give the desired risk bounds.

\begin{proposition}
\label{prop:robustness-upper-bd}
Fix $1 \leq p < q$ and corruption fraction $0 \leq \eps \leq 0.49$. Then we have
\begin{align*}
    R(\SWp,\cG_q,\eps) &\lesssim \sqrt{(1 \lor d/q)(1 \land p/d)}\: \eps^{1/p-1/q}\\
    R(\MSWp,\cG_q,\eps) &\lesssim \eps^{1/p-1/q}.
\end{align*}
\end{proposition}
\begin{proof}
Fixing $\mu \in \cG_q$, it suffices by Proposition \ref{prop:risk-bound-via-resilience} to prove that $\mu$ is $(C\eps^{1/p-1/q},\eps)$-resilient w.r.t.\ $\MSWp$ and $(C\sqrt{(1 \lor d/q)(1 \land p/d)}\: \eps^{1/p-1/q},\eps)$-resilient w.r.t.\ $\SWp$ for all $0 \leq \eps \leq 0.98$, where $C > 0$ is some absolute constant. Since these distances are translation invariant, we can assume without loss of generality that $\mu x = 0$. By Lemmas \ref{lem:moment-transfer} and \ref{lem:moment-resilience}, we know that $\mu$ is $(C\eps^{1/p-1/q},\eps)$-resilient w.r.t.\ $\overline{\sD}_p^{1/p}$ and $(C\sqrt{(1 \lor d/q)(1 \land p/d)}\: \eps^{1/p-1/q},\eps)$-resilient w.r.t.\ $\underline{\sD}_p^{1/p}$ for all $0 \leq \eps \leq 1/2$ and some absolute constant $C > 0$. For $1/2 \leq \eps \leq 0.98$, the same resiliency bounds are implied by Lemma \ref{lem:resiliency-large-eps}, since $1-\eps \geq 0.02 \geq \eps/49$. Finally, we apply Proposition \ref{prop:resiliency-comparison} to obtain the desired risk bounds.
\end{proof}

We now prove the preceding lemmas.
\begin{proof}[Proof of Lemma \ref{lem:moment-transfer}]
Fixing $\mu \in \cG_q$ with $\mu x = 0$, we bound
\begin{align}
\label{eq:avg-moment-bound}
    \SWp(\mu,\delta_{\mu x}) &\asymp \sqrt{1 \land p/d} \: \Wp(\mu,\delta_{\mu x}) && \text{(Lemma \ref{lem:moment-comparison})}\nonumber\\
    &\leq \sqrt{1 \land p/d} \: \Wq(\mu,\delta_{\mu x}) && \text{($q > p$)}\nonumber\\
    &= \sqrt{1 \land p/d} \: \frac{\Wq(\mu,\delta_{\mu x})}{\SWq(\mu,\delta_{\mu x})} \SWq(\mu,\delta_{\mu x}) \nonumber\\
    &\asymp \sqrt{(1 \land p/d)(1 \lor d/q)} \: \SWq(\mu,\delta_{\mu x}) && \text{(Lemma \ref{lem:moment-comparison})}\nonumber\\
    &\leq \sqrt{(1 \land p/d)(1 \lor d/q)} && \text{$\mu \in \cG_q$}.
\end{align}
Similarly, we obtain
\begin{equation*}
    \MSWp(\mu,\delta_{\mu x}) \leq \MSWq(\mu,\delta_{\mu x}) \leq 1.\qedhere
\end{equation*}
\end{proof}

\begin{proof}[Proof of Lemma \ref{lem:moment-resilience}]
By Lemma~\ref{lem:modified-dist-props}, $\underline{\sD}_p$ and $\overline{\sD}_p$ are IPMs with respect to the stated function classes $\underline{\cF}_p$ and $\overline{\cF}_p$, respectively. Note that if $\sD = \sD_\cF$ is an IPM for any symmetric $\cF = -\cF$, then $\mu$ is $(\rho,\eps)$-resilient w.r.t.\ $\sD$ if and only if $f_\sharp \mu$ is $(\rho,\eps)$-resilient (in mean) for all $f \in \cF$.

Now, fix $\mu \in \cG_q$ with $\mu x = 0$. For $\underline{\sD}_p$, we observe that
\begin{align*}
    \mu((\|x\|^p)^{q/p}) &= \mu(\|x\|^q)\\
    &\lesssim C^q (1\lor d/q)^{q/2} \sup_{\theta \in \unitsph} \mu(|\theta^\intercal x|^q) && \text{(Lemma \ref{lem:moment-comparison})}\\
    &\leq C^q (1\lor d/q)^{q/2} && \text{($\mu \in \cG_q$)}\\
    &= \left[C^p (1\lor d/q)^{p/2}\right]^{q/p},
\end{align*}
for some absolute constant $C > 0$. For $f \in \underline{\cF}_p$, we then have that $f_\sharp \mu$ has $q/p$-th moments bounded by $O(C^p (1\land p/d)^{p/2}(1\lor d/q)^{p/2})$, and is thus $(O(C^p (1\land p/d)^{p/2}(1\lor d/q)^{p/2} \eps^{1-p/q}),\eps)$-resilient, by Lemma \ref{lem:moment-mean-resilience}. Taking $p$th roots gives the claim.

For $\overline{\sD}_p$, note that for $\theta \in \unitsph$, we have
\begin{align*}
    \mu((|\theta^\intercal x|^p)^{q/p}) &= \mu(|\theta^\intercal x|^q) \leq 1.
\end{align*}
For $f \in \overline{\cF}_p$, we then have that $f_\sharp \mu$ has $q/p$-th moments bounded by 1 and is thus $O(\eps^{1-p/q},\eps)$-resilient. Taking $p$th roots gives the claim.
\end{proof}

\subsubsection{Higher-dimensional slicing}
We now extend Proposition \ref{prop:robustness-upper-bd} to the $k$-dimensional sliced distances defined by
\begin{equation*}
     \SWpk(\mu,\nu) := \left[\int_{\Gr}\mspace{-12mu} \Wp^p(\pE_\sharp\mu,\pE_\sharp\nu)d\sigma_k(E)\right]^{1/p}\ \text{and} \ \ \ \ \MSWpk(\mu,\nu) := \sup_{E \in \Gr}\mspace{-12mu} \Wp(\pE_\sharp\mu,\pE_\sharp\nu),
\end{equation*}
where $\Gr$ is the Grassmannian of $k$-dimensional linear subspaces of $\R^d$, $\sigma_k$ is its standard Haar measure, and $\pE$ is the orthogonal projection onto $E \in \Gr$. These coincide with $\SWp$ and $\MSWp$ when $k=1$, and both equal $\Wp$ when $k=d$. We focus here on $\MSWp$, with $\SWp$ inheriting the same risk bound, although stronger guarantees can be obtained in a similar manner to the proof of Proposition \ref{prop:robustness-upper-bd}. First, we extend $\overline{\sD}_p$ to this regime as
\begin{equation*}
    \overline{\sD}_{p,k} := \sup_{\substack{U \in \R^{d \times k} \\ U^\intercal U = I_k}} \bigl|(\mu - \nu)(\|U^\intercal x\|^p)\bigr|,
\end{equation*}
and observe that all of the properties from Lemma \ref{lem:modified-dist-props} still hold. Moreover, for $\mu \in \cG_q$ with $\mu x = 0$, we obtain the needed analog of Lemma \ref{lem:moment-transfer}, bounding
\begin{align*}
    \MSWpk(\mu,\delta_0)^{1/p} &= \sup_{\substack{U \in \R^{d \times k} \\ U^\intercal U = I_k}} \mu(\|U^\intercal x\|^p)^{1/p}\\
    &\leq \sup_{\substack{U \in \R^{d \times k} \\ U^\intercal U = I_k}} \mu(\|U^\intercal x\|^q)^{1/q} && \text{($q > p$)}\\
    &\leq \sqrt{1 \lor k/q} \sup_{\substack{U \in \R^{d \times k} \\ U^\intercal U = I_k}} \sup_{\theta \in \mathbb{S}^{k-1}} \mu(|\theta^\intercal U^\intercal x|^q)^{1/q} && \text{(Lemma \ref{lem:moment-comparison})}\\
    &\leq \sqrt{1 \lor k/q} \sup_{\theta \in \unitsph} \mu(|\theta^\intercal x|^q)^{1/q} && \text{($\mathbb{S}^{k-1} \subset \mathbb{S}^{d-1}$)}\\
    &\leq \sqrt{1 \lor k/q} && \text{($\mu \in \cG_q$)}.
\end{align*}
In the same way, we can extract this factor of $\sqrt{1 \lor k/q}$ for the resiliency of $\mu$ w.r.t.\ $\overline{\sD}_{p,k}^{1/p}$ to prove the needed analog of Lemma \ref{lem:moment-resilience}. Combining these results gives that $R(\SWpk,\cG_q,\eps) \leq R(\MSWpk,\cG_q,\eps) \lesssim \sqrt{1 \lor k/q} \, R(\MSWp,\cG_q,\eps) \asymp \sqrt{1 \lor k/q}\, \eps^{1/p-1/q}$ for $0 \leq \eps \leq 0.49$, as desired. %

\subsubsection{Proofs of auxiliary lemmas}
\label{app:auxiliary-lemmas}

\begin{proof}[Proof of Lemma \ref{lem:sphere-coordinate}]
Let $\Theta \sim \Unif(\unitsph)$. When $d=1$, we have $\E[|\Theta_1|^p] = 1$. Otherwise, we use that the probability density function of $\Theta_1$ at $s \in [-1,1]$ is proportional to $(1-s^2)^{\frac{d-3}{2}}$ \cite{spruill2007spheres}. Equivalently, $(\Theta_1 + 1)/2 \sim \Beta\left(\frac{d-1}{2},\frac{d-1}{2}\right)$. We will first prove the desired statement for even integer $p = 2m$, where
\begin{align*}
    \E\left[|\Theta_1|^{2m}\right] &= \frac{(2 m) !}{2^{2 m} m !} \frac{\Gamma(d-1) \Gamma(\frac{d-1}{2}+m)}{\Gamma(\frac{d-1}{2}) \Gamma(d-1+2m)}
\end{align*}
(see, e.g., \cite{marchal2017beta}). Simplifying, we obtain
\begin{align*}
    \E\left[|\Theta_1|^{2m}\right] &= \frac{\Gamma(2m+1)}{2^{2 m} \Gamma(m+1)} \frac{\Gamma(d-1) \Gamma(\frac{d-1}{2}+m)}{\Gamma(\frac{d-1}{2}) \Gamma(d-1+2m)}\\
    &= \frac{\Gamma(d/2)\Gamma(m+1/2)}{\sqrt{\pi}2^{2m}\Gamma(m+d/2)}.
\end{align*}
Employing Stirling's formula, we compute
\begin{align*}
    2^{2m} \, \E\left[|\Theta_1|^{2m}\right] &\asymp \frac{\Gamma(d/2)\Gamma(m+1/2)}{\Gamma(m+d/2)}\\
    &\asymp \frac{(d/2)^{d/2-1/2} e^{-d/2} (m+1/2)^m e^{-m-1/2}}{(m+d/2)^{m+d/2-1/2} e^{-m-d/2}}\\
    &\asymp \frac{(d/2)^{d/2-1/2} (m+1/2)^m }{(m+d/2)^{m+d/2-1/2}}\\
    &= \left(\frac{d/2}{m+d/2}\right)^{\frac{d-1}{2}} \left(\frac{m+1/2}{m+d/2}\right)^m.
\end{align*}
Consequently, we have
\begin{align*}
    \E\left[|\Theta_1|^{2m}\right]^{1/2m} &\asymp \left(1 + \frac{m}{d/2}\right)^{-\frac{d-1}{4m}} \sqrt{\frac{m+1/2}{m+d/2}}\\
    &\asymp \sqrt{\frac{m+1/2}{m+d/2}}\\
    &\asymp 1 \land \sqrt{m/d},
\end{align*}
as desired. When $p \geq 2$ is not an even integer, we use that $\E\left[|\Theta_1|^{p}\right]$ is monotonically increasing in $p$ to obtain matching bounds by rounding $p$ up and down to the nearest even integers. To obtain the needed lower bound when $p \in [1,2)$, we derive
\begin{align*}
    \E\left[|\Theta_1|\right] = \frac{4\Gamma(d-1)}{\Gamma\left(\frac{d-1}{2}\right)^2}\frac{\left(\frac{d-1}{2}\right)^{d-1}}{(d-1)^{d}},
\end{align*}
using the formula for the mean absolute deviation of the beta distribution \cite{gupta2004beta}. Applying Stirling's formula once more, we obtain
\begin{align*}
    \E\left[|\Theta_1|\right] \asymp \frac{(d-1)^{d-3/2}\left(\frac{d-1}{2}\right)^{d-1}}{(\frac{d-1}{2})^{d-2}(d-1)^{d}} = \frac{(d-1)/2}{(d-1)^{3/2}} \asymp d^{-1/2},
\end{align*}
as desired.
\end{proof}

\begin{proof}[Proof of Lemma \ref{lem:moment-comparison}]
Taking $X \sim \mu$ and $\Theta \sim \Unif(\unitsph)$, we use rotational symmetry of the sphere to compute
\begin{equation*}
    \E\left[|\Theta^\intercal X|^p\right] = \E\left[|\Theta_1|^p\right] \, \E\left[\|X\|^p\right],
\end{equation*}
giving the first equality via Lemma \ref{lem:sphere-coordinate}. The inequality follows by comparing an average to a supremum, and the inequality is tight when these coincide, i.e.\ when $\mu$ is rotationally symmetric about 0.
\end{proof}

\subsection{Proof of \Cref{prop:finite-sample-robustness}}
\label{prf:finite-sample-robustness}

The high-level structure of our proof follows a standard template for finite-sample robust mean and covariance estimation (see, e.g., \cite{steinhardt2018resilience,zho2019resilience}). 
We first prove \Cref{prop:finite-sample-robustness} under bounded support and then extend our result to the general setting. Throughout, we write $\BB_r := \{ x \in \R^d : \|x\| \leq r\}$ for the Euclidean ball of radius $r \geq 0$.

\paragraph{Bounded Support:}
For ease of presentation, we slightly extend our notion of resilience in a standard way. We say that $\mu \in \cP(\R^d)$ is $(\rho,\eps)$-resilient w.r.t.\ $\sD$ about $\kappa \in \cP(\R^d)$ if $\sD(\nu,\kappa) \leq \rho$ for all $\nu \leq \frac{1}{1-\eps}\mu$. Namely, we will consider the resilience of an empirical measure $\kappa = \hat{\mu}_n$ about its population measure $\mu$. If $\sD = \sD_\cF$ is an IPM for symmetric $\cF = -\cF$, note that $\mu$ is $(\rho,\eps)$-resilient w.r.t.\ $\sD$ about $\kappa$ if and only if $f_\sharp \mu$ is $(\rho,\eps)$-resilient (in mean) about \ $f_\sharp \kappa$ for all $f \in \cF$.

We first recall and derive some basic results for finite-sample resilience. The following lemma is a simplification of \cite[Proposition 4]{steinhardt2018resilience}, specified to the 1-dimensional case.
\begin{lemma}[1-dimensional finite-sample resilience]
\label{lem:1d-finite-sample-resilience}
Suppose that $\mu \in \cP([-M,M])$ is $(\rho,\eps)$-resilient in mean for $\eps \leq 0.999$. Then, with probability at least $1-\delta$, the empirical distribution $\hat{\mu}_n$ is $(\rho',\eps)$-resilient in mean about $\mu$ with $\rho' = O\left(\rho\left(1 + \sqrt{\frac{\log(1/\delta)}{\eps^2 n}}\right) + \frac{M\log(1/\delta)}{n}\right)$.
\end{lemma}
The result is stated in \cite{steinhardt2018resilience} for $\eps < 1/2$, but the proof only uses that $\eps$ is bounded away from 1. We then extend this result to IPMs over uniformly bounded function classes.

\begin{proposition}[Finite-sample resilience w.r.t.\ IPMs]
\label{prop:IPM-finite-sample-resilience}
Let $\sD_\cF$ be the IPM induced by a function class $\cF = -\cF$ on $\R^d$ with $\|f\|_\infty \leq M$ for $f \in \cF$, and fix any finite subset $\cH \subseteq \cF$ such that $\sD_\cF(\mu,\nu) \leq \sD_\cH(\mu,\nu) + \rho$. Then if $\mu \in \cP(\R^d)$ is $(\rho,\eps)$-resilient w.r.t.\ $\sD_\cF$ for $\eps \leq 0.999$, we have that $\hat{\mu}_n$ is $(\rho',\eps)$-resilient w.r.t.\ $\sD_\cF$ about $\mu$ with probability $1 - \delta$, where $\rho' = O\left(\rho+\frac{\rho}{\epsilon} \sqrt{\frac{\log (|\cH|/ \delta)}{n}}+\frac{M\log (|\cH| / \delta)}{n}\right)$.
\end{proposition}
\begin{proof}
For any $\nu_n \leq (1-\eps)\hat{\mu}_n$, we have
\begin{align*}
    \sD_\cF(\nu_n,\mu) \leq \max_{f \in \cH} (\nu_n - \mu)(f) + \rho.
\end{align*}
Now, fixing any $f \in \cH$, resilience of $\mu$ w.r.t.\ $\sD_\cF$ requires that $f_\sharp \mu$ is $(\rho,\eps)$-resilient. Noting that $f_\sharp \nu_n \leq \frac{1}{1-\eps} f_\sharp \hat{\mu}_n$,
Lemma \ref{lem:1d-finite-sample-resilience} gives that
\begin{equation*}
    |(\nu_n - \mu)(f)| = |(f_\sharp \nu_n)x - (f_\sharp \mu)x| \leq O\left(\rho+\frac{\rho}{\epsilon} \sqrt{\frac{\log (|\cH|/ \delta)}{n}}+\frac{M\log (|\cH| / \delta)}{n}\right)
\end{equation*}
with probability at least $1-\delta/|\cH|$. A union bound over $f \in \cH$ gives the desired result.
\end{proof}

To apply this result, we approximate $\overline{\sD}_p$ with an IPM over a finite function class.
\begin{lemma}[Approximating $\overline{\sD}_p$]
\label{lem:net-bound}
For each $\gamma > 0$, there exists a net $\cN \subset \unitsph$ of size $(10 R^p p / \gamma)^d$ such that for all $\mu,\nu \in \cP(\BB_R)$, we have
\begin{equation*}
    \overline{\sD}_p(\mu,\nu) = \sup_{\theta \in \unitsph}|(\mu - \nu)(|\theta^\intercal x|^p)| \leq \max_{\theta \in \cN}|(\mu - \nu)(|\theta^\intercal x|^p)| + \gamma.
\end{equation*}
\end{lemma}
\begin{proof}
Let $\cN$ be a $\gamma(2 R^p p)^{-1}$-covering for $\unitsph$ in $\ell_2$ with $|\cN| \leq (10 R^p p\gamma^{-1})^d$, the existence of which is guaranteed by Lemma \ref{lem: lip project} (i). Then, taking $\theta$ to be a direction achieving the LHS supremum, and $\tilde{\theta} \in \cN$ to be its nearest neighbor in $\cN$, we have
\begin{align*}
    |(\mu - \nu)(|\theta^\intercal x|^p)| &\leq  |(\mu - \nu)(|\tilde{\theta}^\intercal x|^p)| + 2 \sup_{\kappa \in \cP(\BB_R)} \kappa||\theta^\intercal x|^p - |\tilde{\theta}^\intercal x|^p|\\
    &\leq |(\mu - \nu)(|\tilde{\theta}^\intercal x|^p)| + 2 R^p p \|\theta - \tilde{\theta}\|\\
    &\leq |(\mu - \nu)(|\tilde{\theta}^\intercal x|^p)| + \gamma,
\end{align*}
where the second inequality follows by Lipschitzness. Supremizing over $\theta$ gives the lemma.
\end{proof}

Combining, we obtain finite-sample resilience w.r.t.\ our distances of interest. Slightly abusing notation for brevity, we write $\underline{\rho}(\tau) = \underline{\rho}(\tau,p,d,q) = \sqrt{(1\lor d/q)(1 \land p/d)}\tau^{1/p-1/q}$ and $\overline{\rho}(\tau) = \overline{\rho}(\tau,p,q) = \tau^{1/p-1/q}$ for our resilience bounds for the class $\cG_q$ w.r.t.\ $\SWp$ and $\MSWp$.

\begin{lemma}[Finite-sample resilience under bounded support]
\label{lem:bounded-finite-sample-resilience}
Let $0 \leq \eps \leq 0.999$ and $q > p$. If $\mu \in \cG_q$ with $\diam(\supp(\mu)) \leq R/2$ and $n = \Omega\big((R^p + \eps^{-2})(d \log(R/\eps) + \log(1/\delta))\big)$, then $\hat{\mu}_n$ is $(O(\underline{\rho}(\eps)),\eps)$-resilient w.r.t.\ $\SWp$ and $(O(\overline{\rho}(\eps)),\eps)$-resilient w.r.t.\ $\MSWp$ with probability $1-\delta$.
\end{lemma}
\begin{proof}
Assume without loss of generality that $\mu x = 0$ and $\mu \in \cG_q \cap \cP(\BB_R)$. By Lemma \ref{lem:moment-resilience} (combined with Lemma \ref{lem:resiliency-large-eps} if $\eps \geq 1/2$), we have that $\mu$ is $(\underline{\rho}(\eps),\eps)$-resilient w.r.t.\ $\underline{\sD}_p^{1/p}$ and $(\overline{\rho}(\eps),\eps)$-resilient w.r.t.\ $\overline{\sD}_p^{1/p}$. For $\overline{\sD}_p$, observe that for $\|x\| \leq R$ and $\theta \in \unitsph$, we have $|\theta^\intercal x|^p \leq R^p$. Thus, applying Proposition \ref{prop:IPM-finite-sample-resilience} with $\cF = \overline{\cF}_p$, $M = R^p$, and $\cH$ induced by the net from Lemma \ref{lem:net-bound} with $\gamma = \overline{\rho}(\eps)^p$ gives that $\hat{\mu}_n$ is $(O(2^p\overline{\rho}(\eps)^p),\eps)$-resilient about $\mu$ w.r.t.\ $\overline{\sD}_p$ with probability at least $1-\delta/2$ whenever
\begin{align*}
    n &\geq (\overline{\rho}(\eps)^p R^p + \eps^{-2}) \log(2|\cH|/\delta)/2^p\\
    &=(\overline{\rho}(\eps)^p R^p + \eps^{-2}) \log((20 R^p p / \overline{\rho}(\eps)^p)^d/\delta) / 2^p.
\end{align*}
Plugging in our value for $\overline{\rho}(\eps)$ and applying some crude bounds shows the stated sample complexity of $n = \Omega\big((R^p + \eps^{-2})(d \log(R/\eps) + \log(1/\delta))\big)$ suffices. Since the resilience bound is centered about $\mu$ and $\mu \in \cG_q$, we deduce that $\overline{\sD}_p(\hat{\mu}_n,\delta_0) \leq \overline{\sD}_p(\hat{\mu}_n,\mu) + \overline{\sD}_p(\mu,\delta_0) \leq O(2^p \overline{\rho}(\eps)^p) + 1 = O(2^p)$. Thus, by \Cref{prop:resiliency-comparison}, we have that $\hat{\mu}_n$ is $(O(\overline{\rho}(\eps)),\eps)$-resilient w.r.t.\ $\MSWp$. An analogous argument shows that the same sample complexity suffices for $\SWp$ (of course, far fewer samples are actually needed, but we shall not focus on this distinction). Applying a union bound gives that both resilience guarantees hold with probability $1- \delta$.
\end{proof}

Finally, we use finite-sample resiliency to bound finite-sample robust estimation risk.

\begin{proposition}
\label{prop:bounded-finite-sample-robust-estimation}
Let $0 \leq \eps \leq 0.499$ and $q > p$ and $\sD \in \{ \SWp, \MSWp\}$. Then there exists an estimation procedure $\cA$ with the following guarantee: for any $\mu \in \cG_q$ with $\diam(\supp(\mu)) \leq R/2$ and $n = \Omega\big((R^p + \eps^{-2})(d \log(R/\eps) + \log(1/\delta))\big)$, after observing any random distribution $\tilde{\mu}_n$ such that $\|\tilde{\mu}_n - \hat{\mu}_n\|_\tv \leq \eps$ almost surely, $\cA$ produces $\nu$ such that $\sD(\nu,\mu) \lesssim R(\sD,\cG_q,\eps) + \sD(\hat{\mu}_n,\mu)$ with probability $1-\delta$.
\end{proposition}
\begin{proof}
For $\sD = \MSWp$, define
\begin{equation*}
    \Pi_{\eps,\rho}(\tilde{\mu}_n) = \left\{ \kappa \in \cP(\R^d) : \|\kappa - \tilde{\mu}_n\|_\tv \leq \eps \text{ and $\kappa$ is $(\rho,2\eps)$-resilient w.r.t.\ $\MSWp$}\right\}
\end{equation*}
Write $\rho_\star = \inf \{ \rho \geq 0 : \Pi_{\eps,\rho}(\tilde{\mu}_n) \neq \emptyset \}$ for the smallest resilience parameter such that this set is non-empty. Since $2\eps \leq 0.999$, we know by \Cref{lem:bounded-finite-sample-resilience} that $\hat{\mu}_n$ is $(O(\overline{\rho}(2\eps),2\eps)$-resilient w.r.t.\ $\MSWp$ (for an appropriate choice of constant in the sample complexity) with probability $1 - \delta/2$. Noting that $\|\tilde{\mu}_n - \hat{\mu}_n\|_\tv \leq \eps$, we have $\rho_\star \lesssim \overline{\rho}(2\eps) \lesssim \overline{\rho}(\eps)$ with probability $1 - \delta/2$.

Now consider any algorithm which returns $\nu \in \Pi_{\eps,2\rho_\star}(\tilde{\mu}_n)$. Then we have $\|\nu - \hat{\mu}_n\|_\tv \leq \|\nu - \tilde{\mu}_n\|_\tv + \|\tilde{\mu}_n - \hat{\mu}_n\|_\tv \leq 2\eps$. By considering their midpoint $\kappa = \frac{1}{1 - \|\nu - \hat{\mu}_n\|_\tv} \nu \land \hat{\mu}_n$ and applying $(O(\overline{\rho}(2\eps)),2\eps)$-resilience of $\nu$ and $\hat{\mu}_n$ w.r.t.\ $\MSWp$, we deduce that $\MSWp(\nu,\hat{\mu}_n) \leq \MSWp(\nu,\kappa) + \MSWp(\kappa,\hat{\mu}_n) \lesssim \overline{\rho}(\eps) \lesssim R(\MSWp,\cG_q,\eps)$ with probability $1 - \delta$. 
By the triangle inequality for $\MSWp$, we thus have $\MSWp(\nu,\mu) \lesssim R(\MSWp,\cG_q,\eps) + \MSWp(\hat{\mu}_n,\mu)$ with probability $1 - \delta$. An analogous argument gives the corresponding result for $\SWp$.
\end{proof}

\paragraph{Reduction to bounded support:} 
To prove \Cref{prop:finite-sample-robustness}, we provide a reduction from the general case to that of bounded support via Markov's inequality and a coupling argument.

\begin{lemma}[High probability norm bound]
\label{lem:high-probability-norm-bound}
If $X \sim \mu \in \cG_q$, there exists $R \lesssim \delta^{-1/q} \sqrt{1 \lor d/q} \leq \delta^{-1/q}\sqrt{d}$ such that $\|X - \mu x\| \leq R$ with probability at least $1 - \delta$.
\end{lemma}
\begin{proof}
Assume without loss of generality that $\mu x = 0$. We compute
\begin{align*}
    \mu(\|x\|^q)^{1/q} \lesssim (1 \lor d/q)^{q/2} \sup_{\theta \in \unitsph} \mu(|\theta^\intercal x|^q)^{1/q} \leq (1 \lor d/q)^{q/2},
\end{align*}
where the first inequality uses Lemma~\ref{lem:moment-comparison} and the second uses $\mu \in \cG_q$. Markov's inequality then gives the claim.
\end{proof}

\begin{lemma}[Switch of base measure]
\label{lem:base-measure-switch}
Fix $\mu \in \cP(\R^d)$ and $A \subseteq \R^d$ with $\mu(A) \geq 1-\eps$. Write $\mu_A$ for the distribution of $X \sim \mu$ conditioned on $X \in A$. Consider any random measure $\tilde{\mu}_n$ such that $\|\tilde{\mu}_n - \hat{\mu}_n\| \leq \eps'$ almost surely, where $n \geq 3 \log(1/\delta)/\eps$. Then there exists a coupling of $(\tilde{\mu}_n,\hat{\mu}_n)$ and $(\widehat{\mu_A})_n$ such that $\|\tilde{\mu}_n - (\widehat{\mu_A})_n\|_\tv \leq 2\eps + \eps'$ with probability at least $1 - \delta$.
\end{lemma}

\begin{proof}
Given $n$ i.i.d.\ samples $X_1,\dots,X_n$ from $\mu$, \Cref{lem:base-measure-switch} and a Chernoff bound give that at least $(1-2\eps)n$ of them satisfy $X_i \in A$, with probability at least $1-\delta$. Define the coupled set of samples $Y_1,\dots,Y_n$ by $Y_i = X_i$ if $X_i \in A$ and $Y_i \sim \mu_A$ i.i.d.\ otherwise, and choose $(\widehat{\mu_A})_n$ as their empirical measure (by design, the marginal distribution of $Y_1,\dots,Y_n$ coincides with $n$ samples from $\mu_A$). Under this coupling, we then have
\begin{align*}
    \|\tilde{\mu}_n - (\widehat{\mu_A})_n\|_\tv \leq \|\tilde{\mu}_n - \hat{\mu}_n\|_\tv + \|\hat{\mu}_n - (\widehat{\mu_A})_n\|_\tv \leq 2\eps + \eps'
\end{align*}
with probability at least $1-\delta$.
\end{proof}

\begin{proof}[Proof of \Cref{prop:finite-sample-robustness}]
By \Cref{lem:high-probability-norm-bound}, we have that for $X \sim \mu$, $\|X - \mu x\| \leq R \asymp \sqrt{d/\eps}$ with probability at least $1 - \eps/400$. Letting $A$ denote the ball of radius $R$ around $\mu x$ and applying \Cref{lem:base-measure-switch} with failure probability $0.001$, we can view $\tilde{\mu}_n$ as being a $\frac{201}{200}\eps$-corrupted version of $n$ i.i.d.\ samples from the conditional distribution $\mu_R$, with probability at least $0.999$. Thus, applying the procedure from \Cref{prop:bounded-finite-sample-robust-estimation} with $R \asymp \sqrt{d/\eps}$, confidence probability $0.999$ and corruption fraction $\frac{201}{200}\eps < 0.499$, we obtain $\nu$ with $\MSWp(\nu,\mu_R) \lesssim R(\sD,\cG_q,\eps) + \sD((\widehat{\mu_R})_n,\mu_R)$ with unconditional probability $0.998$. By resilience of $\mu$ and the fact that $\|\mu - \mu_R\|_\tv \leq \eps/400$, the same recovery guarantees hold with base measure $\mu$. Finally, we bound $\sD((\widehat{\mu_R})_n,\mu_R)$ by its expectation via Markov's inequality to obtain $\MSWp(\nu,\mu) \lesssim R(\sD,\cG_q,\eps) + \E[\sD((\widehat{\mu_R})_n,\mu_R)]$ with probability $0.99$.
\end{proof}

\subsection{Proof of Proposition \ref{prop:resilience}}
\label{prf:resilience}

For any $\mu,\nu \in \cP_1(\R^d)$, we have $\MSWp(\mu,\nu) \geq \| \mu x - \nu x\|$ (seen by taking $\theta$ in the direction of $\mu x - \nu x$). Hence, if $\MSWp(\mu,\nu) \leq \rho$ for all $\nu \leq \frac{1}{1-\eps}\mu$, then $\mu$ is $(\rho,\eps)$-resilient in mean. (This direction holds for all $p \geq 1$). 
For the other direction, we mirror the proof of Theorem \ref{thm:robustness}, first establishing a simple lemma.

\begin{lemma}
\label{lem:quantiles}
Fix $X \sim \mu \in \cP_1(\R)$ and define the quantiles $\tau_\eps = \sup\{ t \in \R : \Pr(X \geq t) \geq \eps \}$ and $\tilde{\tau}_\eps = \sup\{ t \in \R : \Pr(|X| \geq t) \geq \eps \}$. Then, we have
\begin{equation*}
    \E[|X| \:|\: |X| \geq \tilde{\tau}_\eps] \leq 4  \: \E[X \:|\: X \geq \tau_\eps] \lor \E[-X \:|\: X \leq \tau_{1-\eps}]
\end{equation*}
\end{lemma}
Simply put, if $|X|$ has large tails, then one of $X$ or $-X$ must have a large tail.
\begin{proof}
Writing $X_+ = X \lor 0$ and $X_- = -X \lor 0$, we bound
\begin{align*}
    &\E\bigl[|X| \:\big|\: |X| \geq \tilde{\tau}_\eps\bigr] = \E\bigl[X_+ \:\big|\: |X| \geq \tilde{\tau}_\eps\bigr] + \E\bigl[X_- \:\big|\: |X| \geq \tilde{\tau}_\eps\bigr]\\
    &\leq \E[X_+ | X \geq \tau_\eps] + \E[X_- | X \leq \tau_{1-\eps}]\\
    &\leq \E[X - (\tau_\eps \land 0) \:|\: X \geq \tau_\eps] + \E[-X + (\tau_{1-\eps} \lor 0) \:|\: X \leq \tau_{1-\eps}]\\
    &= \E[X \:|\: X \geq \tau_\eps] + \E[-X  \:|\: X \leq \tau_{1-\eps}] - (\tau_\eps \land 0) + (\tau_{1-\eps} \lor 0)\\
    &\leq \E[X \:|\: X \geq \tau_\eps] + \E[-X  \:|\: X \leq \tau_{1-\eps}] + (-\tau_{1-\eps} \lor 0) + (\tau_\eps \lor 0)\\
    &\leq \E[X \:|\: X \geq \tau_\eps] + \E[-X  \:|\: X \leq \tau_{1-\eps}] + (\E[-X  \:|\: X \leq \tau_{1-\eps}] \lor 0) + (\E[X \:|\: X \geq \tau_\eps] \lor 0).
\end{align*}
Now, it is easy to check that each summand is bounded by $\E[X \:|\: X \geq \tau_\eps] \lor \E[-X \:|\: X \leq \tau_{1-\eps}]$ (since this maximum is non-negative), giving the lemma.
\end{proof}

Continuing, we take $\mu \in \cP_1(\R^d)$ which is $(\rho,\eps)$-mean-resilient and assume without loss of generality that $\mu x = 0$. For all $\nu \leq \frac{1}{1-\eps}\mu$, we write $\mu = (1-\eps)\nu + \eps \alpha$ for $\alpha \in \cP(\R^d)$ and bound
\begin{align*}
    \MSWp(\mu,\nu) &= \MSWp((1-\eps)\nu + \eps \alpha, \nu)\\
    &\leq \eps^{1/p} \MSWp(\alpha,\nu) && \text{(Fact \ref{fact:homogeneity})}\\
    &\leq \eps^{1/p} (\MSWp(\alpha,\delta_0) + \MSWp(\nu,\delta_0)) && \text{(triangle inequality)}\\
    &\leq 2 \eps^{1/p} \sup_{\kappa \leq \frac{1}{(1-\eps) \land \eps} \mu} \MSWp(\kappa,\delta_0)\\
    &= 2 \eps \sup_{\kappa \leq \frac{1}{(1-\eps) \land \eps} \mu} \sup_{\theta \in \unitsph}\E_\kappa \left[|\theta^\intercal X|^p\right]^{1/p}\\
    &= 2 \eps \sup_{\theta \in \unitsph} \E_\kappa \left[|\theta^\intercal X|^p \:\big|\: |\theta^\intercal X| \geq \tilde{\tau}_{\eps \land (1-\eps)}(\theta) \right]^{1/p}.
\end{align*}
where $\tilde{\tau}_\eps(\theta) = \sup \{t \in \R : \PP(|\theta^\intercal X| \geq t) \geq \eps \}$ for $X \sim \mu$. (Technically, the final inequality may fail if $\mu$ has a point mass at $\tilde{\tau}_{\eps \land (1-\eps)}(\theta)$; in this case, assume that ties are broken with independent randomness so that the conditioned event has probability $\eps \land (1-\eps)$). From now on, we will use that $p=1$. Writing $\tau_\eps(\theta) = \sup \{t \in \R : \PP(\theta^\intercal X \geq t) \geq \eps \}$ and breaking ties in the same way, we apply Lemma \ref{lem:quantiles} to bound
\begin{align*}
    \MSWone(\mu,\nu) &\leq 8 \eps \sup_{\theta \in \unitsph} \E_\mu [\theta^\intercal X \:|\: \theta^\intercal X \geq \tau_{\eps \land (1-\eps)}(\theta) ]\\
    &=8 \eps \sup_{\theta \in \unitsph} \theta^\intercal \E_\mu [ X \:|\: \theta^\intercal X \geq \tau_{\eps \land (1-\eps)}(\theta)]\\
    &=8 \eps  \: \|\E_\mu [ X \:|\: \theta^\intercal X \geq \tau_{\eps \land (1-\eps)}(\theta)]\|\\
    &=8 \eps  \: \|\E_\mu X - \E_\mu [ X \:|\: \theta^\intercal X \geq \tau_{\eps \land (1-\eps)}(\theta)]\|.
\end{align*}
Now, if $\eps \geq 1/2$, we can use resilience of $\mu$ to bound 
\begin{align*}
    \MSWone(\mu,\nu) &\leq 8 \eps \rho \leq 8\rho.
\end{align*}
Otherwise, writing $E$ for the event that $\theta^\intercal X \geq \tau_\eps(\theta)$, we have
\begin{align*}
    \MSWone(\mu,\nu) &\leq 8\eps\: \| \eps\E_\mu[X|E] + (1-\eps)\E_\mu[X|E^c] - \E_\mu[X|E]\|\\
    &= 8\eps (1-\eps)\: \|\E_\mu[X|E^c] - \E_\mu[X|E]\|\\
    &= 8\eps (1-\eps)\: \|\E_\mu[X|E^c] - \eps^{-1}(\E_\mu[X] - (1-\eps)\E_\mu[X|E^c])\|\\
    &= 8 (1-\eps)\: \|\E_\mu[X] - \E_\mu[X|E^c]\|\\
    &\leq 8 (1-\eps) \rho \leq 8 \rho.
\end{align*}
Hence, $\mu$ is $(8\rho,\eps)$-resilient w.r.t.\ $\MSWone$.

Immediately, this allows $\MSWone$ to inherit a multitude of (population-limit and finite-sample) risk bounds from the robust mean estimation literature. See \cite{steinhardt2018robust} for a detailed survey of robust statistics results based on resiliency. For example, $\mu \in \cG_q$ is known to be $(O(\eps^{1-1/q}),\eps)$-mean-resilient, immediately implying Theorem \ref{thm:robustness} for $\MSWone$.

\subsection{Proof of \Cref{prop:finite-sample-cov-alg}}
\label{prf:finite-sample-cov-alg}

When $q=2$, we mirror the approach of \Cref{prop:finite-sample-robustness} but perform projection onto the space of distributions with bounded covariance, instead of onto the space of resilient distributions. We require the following standard result (see, e.g., Lemma A.18 of \cite{diakonikolas2017robust}), establishing finite-sample covariance bounds under bounded support.

\begin{lemma}
\label{lem:finite-sample-covariance-bound}
Let $\mu \in \cP(\R^d)$ with $\|\Sigma_\mu\|_\mathrm{op} \leq \sigma^2$ and $\diam(\supp(\mu)) \leq R$. Then the empirical distribution $\hat{\mu}_n$ satisfies $\|\Sigma_{\hat{\mu}_n}\|_\mathrm{op} \lesssim \sigma^2$ with probability at least 0.999 for $n \gtrsim R^2 \log(d)$.
\end{lemma}

Importantly, there are efficient filtering algorithms for projecting onto the set of distributions with bounded covariance (see, e.g., Theorem 3.1 \cite{hopkins2020robust}).

\begin{lemma}[Spectral reweighting]
\label{lem:efficient-recovery}
Let $x_1, \dots, x_n \in \R^d$ and $0 < \eps \leq 1/10$.
Suppose the discrete measure $\mu_n = \frac{1}{n} \sum_{i=1}^n \delta_{x_i}$ admits an $\eps$-deletion $\nu_n \leq \frac{1}{1-\eps}\mu_n$ such that $\|\Sigma_{\mu_n}\|_\mathrm{op} \leq \sigma^2$.
Then, given $\{x_i\}_{i=1}^n$ and $\eps$, there is an algorithm which finds $\nu \leq \frac{1}{1-3\eps}\mu_n$ such that $\|\Sigma_{\nu}\|_\mathrm{op} \lesssim \sigma^2$ with probability 0.999, in time $\smash{\widetilde{O}(n d^2)}$.
\end{lemma}

Combining, we prove the proposition. We remark that sample complexity is dominated by empirical convergence under $\sD \in \{ \SWp, \MSWp\}$ of the truncated version of a distribution with bounded second moments. This can be improved significantly in many cases of interest, for example under log-concavity of the clean distribution.

\begin{proof}[Proof of \Cref{prop:finite-sample-cov-alg}]
First, we consider the case of bounded support, when $\diam(\supp(\mu)) \leq R$, and with contamination fraction $\eps \in [0,1/10]$. We mirror the argument of \Cref{prop:bounded-finite-sample-robust-estimation}, but project onto the set of distributions with bounded covariance using spectral reweighting. Write $\tilde{\mu}_n$ for the empirical distribution of the $\eps$-contaminated samples and $\hat{\mu}_n$ for that of the clean samples, with $\hat{\mu}_n \leq \frac{1}{1-\eps}\tilde{\mu}_n$. Combining \Cref{lem:finite-sample-covariance-bound,lem:efficient-recovery}, we find that $\|\Sigma_{\hat{\mu}_n}\|_\mathrm{op} \lesssim \sigma^2$ and that the spectral reweighting algorithm returns $\nu \leq \frac{1}{1-3\eps}\hat{\mu}_n$ with $\|\Sigma_\nu\|_\mathrm{op} \lesssim \sigma^2$ in time $\widetilde{O}(nd^2)$, all with probability 0.998. By resilience of the class $\cG_2(\sigma)$ w.r.t.\ $\sD$ and Markov's inequality, we have 
\begin{equation*}
    \sD(\nu,\mu) \leq \sD(\nu,\hat{\mu}_n) + \sD(\hat{\mu}_n,\mu) \lesssim R(\sD,\cG_q(\sigma),\eps) + \E[\sD(\hat{\mu}_n,\mu)]
\end{equation*}
with probability $0.995$. For the unbounded case, we apply \Cref{lem:high-probability-norm-bound} and \Cref{lem:base-measure-switch} as in the proof of \Cref{prop:finite-sample-robustness} to reduce to $R \asymp \sqrt{d/\eps}$ and obtain the desired error bound with probability at least 0.99, so long as $0 < \eps \leq 1/12$ (any constant separated from 1/10 will do).
\end{proof}

\subsection{Proof of Lemma~\ref{lem: wp_theta_lipschitz}}\label{APPEN:wp_theta_lipschitz_proof}

We start by showing that the Lipschitz constant of $w_p$ is upper bounded by $L_{\mu,\nu}^p$. Fix $\theta_1, \theta_2 \in \unitsph$ and observe that
\begin{align*}
\big|w_p(\theta_1) - w_p(\theta_2)\big| &= \big|\Wp\big(\proj^{\theta_1}_\sharp \mu, \proj^{\theta_1}_\sharp \nu\big) - \Wp\big(\proj^{\theta_2}_\sharp \mu,\proj^{\theta_2}_\sharp \nu\big)\big| \\
&\leq \Wp\big(\proj^{\theta_1}_\sharp \mu, \proj^{\theta_2}_\sharp \mu\big) + \Wp\big(\proj^{\theta_1}_\sharp \nu, \proj^{\theta_2}_\sharp \nu\big)\\
&\leq \|\theta_1 - \theta_2\| \sup_{\theta \in \unitsph} \left ( \big(\mu |\theta^\intercal x|^p\big)^{1/p} + \big(\nu |\theta^\intercal x|^p \big)^{1/p} \right ),
\end{align*}
where the last step uses the optimal transportation cost formulation of $\Wp$. The RHS above is $L_{\mu,\nu}^p$ from the lemma, which concludes the proof of the first statement.

Next, we bound the Lipschitz constant of $w_p^p$. For $\theta_1, \theta_2 \in \unitsph$ and $i=1,2$, let $(X_i, Y_i)$ be a coupling of $\mu$ and $\nu$ so that $(\theta^\intercal_i X_i, \theta_i^\intercal Y_i)$ is optimal for $\Wp\big(\proj^{\theta_i}_\sharp \mu, \proj^{\theta_i}_\sharp \nu\big)$. These couplings are constructed as follows. For $i=1,2$, let $(U_i,V_i)$ be an optimal couplings for $\Wp\big(\proj^{\theta_i}_\sharp \mu, \proj^{\theta_i}_\sharp \nu\big)$. Take $\mathrm{P}_i\in\RR^{d\times d}$ as a unitary matrix whose first row is $\theta_i$, and let $\mathrm{P}_{i,-1}\in\RR^{(d-1)\times d}$ denote the matrix obtained by deleting the first row of $\mathrm{P}_i$. Given $u_1,u_2\in\RR$ generate the random variables $W_i(u_i)\sim \law\big(P_{i,-1}X\big|\theta_i^\intercal X=u_i\big)$, for $i=1,2$, where $X\sim \mu$ and $\cL(\cdot)$ designates the probability law of a random variable. Setting $\bar{U}_i:=\big(U_i,W_i(U_i)\big)$ for $i=1,2$, observe that $\bar{U}_i\sim \law\big(\mathrm{P}_i X\big)$ and further that $X_i:=\mathrm{P}_i^\intercal \bar{U}_i\sim \mu$. Constructing $\bar{V}_i$, for $i=1,2$, in an analogous fashion but with $\nu$ in place of $\mu$, and defining $Y_i$ similarly to $X_i$ above, we  obtain the desired $(X_i,Y_i)$ couplings. 

Then by optimality of the couplings, we have
\[
\begin{split}
w_p^p(\theta_1) - w_p^p(\theta_2) &\leq \EE\big[\big|\theta_1^\intercal (X_2 - Y_2)\big|^p - \big|\theta_2^\intercal (X_2 - Y_2)\big|^p \big],\\
w_p^p(\theta_2) - w_p^p(\theta_1) &\leq \EE\big[\big|\theta_2^\intercal (X_1 - Y_1)\big|^p - \big|\theta_1^\intercal (X_1 - Y_1)\big|^p \big].
\end{split}
\]

Combining these bounds, we obtain
\begin{align*}
|w_p^p(\theta_1) - w_p^p(\theta_2)| &\leq p\|\theta_1 - \theta_2\|\, \EE \left  [ \max_{i=1,2} \left | \frac{(\theta_1 - \theta_2)^\intercal (X_i - Y_i)}{\|\theta_1 - \theta_2\|} \right | \cdot \max_{i,j=1,2} \big|\theta_i^\intercal (X_j - Y_j)\big|^{p-1} \right ]\\
&\leq p \|\theta_1 - \theta_2\| \, \EE \left [ \max_{\substack{i=1,2 \\ j = 1,2,3}} \big|\theta_j'(X_i - Y_i)\big|^p \right ]\\
&\leq 3p2^p\|\theta_1 - \theta_2\| \, \sup_{\theta \in \unitsph} \EE \big [ \big|\theta^\intercal X_1|^p + |\theta^\intercal Y_1|^p \big ],
\end{align*}
where for the second inequality we have defined $\theta_3 := \frac{\theta_1 - \theta_2}{\|\theta_1 - \theta_2\|}$. This concludes the proof. 

\begin{remark}[Alternative Lipschitz constants]
\label{rem: wp_lipschitz_alt}
The Lipschitz constant for $w_p^p$ can be alternatively derived as
\begin{align*}
|w_p^p(\theta_1) - w_p^p(\theta_2)| &\leq p\|\theta_1 - \theta_2\|\, \EE \left  [ \max_{i=1,2} \left | \frac{(\theta_1 - \theta_2)^\intercal (X_i - Y_i)}{\|\theta_1 - \theta_2\|} \right | \cdot \max_{i,j=1,2} \big|\theta_i^\intercal (X_j - Y_j)\big|^{p-1} \right ]\\
&\leq p \|\theta_1 - \theta_2\| \, \EE \left [ \max_{\substack{i=1,2 \\ j = 1,2,3}} \big|\theta_j'(X_i - Y_i)\big|^p \right ]\\
&\lesssim_p \|\theta_1 - \theta_2\| \left ( \|\mu x - \nu x\| +  \sup_{\theta \in \unitsph} \EE \big [ \big | \theta^\intercal (X_1 - \mu x) \big |^p + \big | \theta^\intercal (Y_1 - \nu x) \big |^p \big ] \right ),
\end{align*}
where the terms corresponding to mean difference and covariance are separated.
\end{remark}

\subsection{Proof of Proposition \ref{prop: monte carlo}}\label{APPEN: monte carlo proof}

We decompose the error by introducing the Monte Carlo average for the population projected distances:
\begin{align*}
    &\EE\Big[\Big|\WMC-\SWp^p(\mu,\nu)\Big|\Big]\\
    &\leq \EE\left[\left| \WMC-  \frac{1}{m}\sum_{i=1}^m \mathsf{W}_p^p \big ( \proj^{\Theta_i}_\sharp \mu, \proj^{\Theta_i}_\sharp \nu \big ) \right | \right ] +\EE \left [ \left | \frac{1}{m} \sum_{i=1}^m \mathsf{W}_p^p \big ( \proj^{\Theta_i}_\sharp \mu, \proj^{\Theta_i}_\sharp \nu \big )  - \SWp^p(\mu, \nu) \right | \right ] \Bigg \}.\numberthis\label{EQ:MC_decomp}
\end{align*}

For the first term, using the fact that $\Theta_1,\ldots,\Theta_n$ are i.i.d., we have
\begin{equation}
\EE\left[\left | \WMC- \frac{1}{m} \sum_{i=1}^m \mathsf{W}_p^p \big ( \proj^{\Theta_i}_\sharp \mu, \proj^{\Theta_i}_\sharp \nu \big ) \right | \right ] \leq \EE\bigg\{\EE \bigg [ \Big| \mathsf{W}_p^p \big(\proj^{\Theta}_\sharp \empmu, \proj^{\Theta}_\sharp \empnu\big) - \mathsf{W}_p^p \big ( \proj^{\Theta}_\sharp \mu, \proj^{\Theta}_\sharp \nu \big ) \Big | \,\bigg | \Theta\bigg ]\bigg\}.\label{EQ:MC_A}
\end{equation}

Denote $(f \oplus g)(x,y) = f(x) + g(y)$, and let $c(x,y) = \|x - y\|^2$.  Further, define the c-conjugate of a function $f$ as $f^c(y) = \inf_x c(x,y) - f(x)$. For each $\theta\in\unitsph$, observe that
\begin{align*}
    \mathsf{W}_p^p &\big ( \proj^\theta_\sharp \empmu, \proj^\theta_\sharp \empnu \big ) - \mathsf{W}_p^p \big ( \proj^{\theta}_\sharp \mu, \proj^{\theta}_\sharp \nu \big )\\
    &\leq \sup_{\substack{(\varphi,\psi) \in L^1(\mu)\times L^1(\nu): \\ \varphi\oplus\psi \le c}} \left\{(\proj^\theta_\sharp \empmu) \varphi + (\proj^\theta_\sharp \empnu) \psi \right\} - \sup_{\substack{(f,g) \in L^1(\mu)\times L^1(\nu): \\ f\oplus g \le c}} \left \{  (\proj^\theta_\sharp \mu) f + (\proj^\theta_\sharp \nu) g \right \}\\
    &\leq \sup_{(\varphi,\psi) \in L^1(\mu)\times L^1(\nu)} \left\{ (\proj^\theta_\sharp \empmu) \varphi + (\proj^\theta_\sharp \mu) \varphi^c + (\proj^\theta_\sharp \empnu) \psi + (\proj^\theta_\sharp \nu) \psi^c \right\}\\
    &= \mathsf{W}_p^p\big(\proj^\theta_\sharp\empmu, \proj^\theta_\sharp\mu\big) + \mathsf{W}_p^p\big(\proj^\theta_\sharp\empnu, \proj^\theta_\sharp\nu\big).
\end{align*}
Repeating this argument for $\mathsf{W}_p^p \big ( \proj^{\theta}_\sharp \mu, \proj^{\theta}_\sharp \nu \big ) - \mathsf{W}_p^p \big ( \proj^\Theta_\sharp \empmu, \proj^\Theta_\sharp \empnu \big )$ we obtain
\[
\Big|\mathsf{W}_p^p \big ( \proj^{\theta}_\sharp \mu, \proj^{\theta}_\sharp \nu \big ) - \mathsf{W}_p^p \big (\proj^\theta_\sharp \empmu, \proj^\theta_\sharp \empnu \big )\Big| \leq \mathsf{W}_p^p(\proj^{\theta}_\sharp\empmu, \proj^{\theta}_\sharp\mu) + \mathsf{W}_p^p(\proj^{\theta}_\sharp\empnu, \proj^{\theta}_\sharp\nu).
\]

The proof of Theorem \ref{thm: SWp rate} implies that, for any $\theta \in \unitsph$,
\[
\begin{split}
\EE\big[\mathsf{W}_p^p\big(\proj^\theta_\sharp\empmu, \proj^\theta_\sharp\mu\big)\big] &\leq C_p \frac{(\log n)^{\ind_{\{p=2\}}} \|\Sigma_\mu\|_{\op}^{p/2}} {n^{(p\wedge 2)/2}},\\
\EE\big[\mathsf{W}_p^p\big(\proj^\theta_\sharp\empnu, \proj^\theta_\sharp\nu\big)\big] &\leq C_p \frac{(\log n)^{\ind_{\{p=2\}}} \|\Sigma_\nu\|_{\op}^{p/2}} {n^{(p\wedge 2)/2}}.
\end{split}
\]
Inserting this back into \eqref{EQ:MC_A}, we have
\begin{equation}
\EE\left[\left| \WMC-  \frac{1}{m}\sum_{i=1}^m \mathsf{W}_p^p \big ( \proj^{\Theta_i}_\sharp \mu, \proj^{\Theta_i}_\sharp \nu \big ) \right | \right ] \leq  \frac{C_p\big( \|\Sigma_\nu\|_{\op}^{p/2} + \|\Sigma_\mu\|_{\op}^{p/2} \big) (\log n)^{\ind_{\{p=2\}}}}{n^{(p\wedge 2)/2}}\label{EQ:MC_A_final}
\end{equation}

\medskip

For the second term in \eqref{EQ:MC_decomp}, recall that $w_p^p(\theta) := \mathsf{W}_p^p \big ( \proj^{\theta}_\sharp \mu, \proj^{\theta}_\sharp \nu \big )$ for $\theta\in\unitsph$, and bound
\[
    \EE \Bigg [ \Bigg | \frac{1}{m}  \sum_{i=1}^m \mathsf{W}_p^p \big ( \proj^{\Theta_i}_\sharp \mu, \proj^{\Theta_i}_\sharp \nu \big ) - \SWp^p(\mu, \nu) \Bigg | \Bigg ] \leq  \sqrt{ \frac{1}{m} \Var \big (  w_p^p(\Theta)  \big )}
\]
To control the variance we use concentration of Lipschitz functions on the unit sphere. By Remark~\ref{rem: wp_lipschitz_alt} following the proof of Lemma~\ref{lem: wp_theta_lipschitz}, $w_p^p$ is $\tilde M^p_{\mu, \nu}$-Lipschitz, with $\tilde M^p_{\mu, \nu} \lesssim_p \|\mu x - \nu x\|^p + \sup_{\theta \in \unitsph} (\mu |\theta^\intercal (x - \mu x)|^{p} + \nu |\theta^\intercal (x - \nu x)|^{p})$. Denoting the median by $\med(\cdot)$, we have for $d \geq 3$ (cf. e.g., \cite[Chapter 1]{ledoux1991probability})
\[
\PP\Big(\big|w_p^p(\Theta) - \med\big(w_p^p(\theta)\big)\big| \ge t\Big) \leq 8 \exp\left(-\frac{(d-2)t^2}{2(\tilde M^p_{\mu, \nu})^2}\right).
\]
Consequently,
\begin{align*}
    \Var \big (  w_p^p(\Theta)  \big ) &\leq \EE \Big [ \big ( w_p^p(\Theta) - \med \big ( w_p^p(\Theta) \big ) \big )^2 \Big ]\\
    &= \int_{0}^\infty \PP\Big(\big|w_p^p(\Theta) - \med\big(w_p^p(\Theta)\big)\big| \ge \sqrt{t}\Big)\,dt\\
    &\leq \frac{16(\tilde M^p_{\mu, \nu})^2}{d-2}.
\end{align*}
Alternatively, for $d \leq 2$, letting $\Theta,\Theta'$ be independent samples drawn uniformly from $\unitsph$, we have
\[
\Var\big(w_p^p(\Theta)\big)=\frac 12 \EE\big[\big| w_p^p(\Theta)-w_p^p(\Theta')\big|^2\big]\leq \frac{(\tilde M^p_{\mu, \nu})^2}{2}\EE\big[|\Theta-\Theta'|^2\big] \leq (\tilde M^p_{\mu, \nu})^2.
\]
Combining the two variance bounds, for any $d\geq 1$, we obtain
\begin{equation}
\EE \Bigg [ \Bigg | \frac{1}{m}  \sum_{i=1}^m \mathsf{W}_p^p \big ( \proj^{\Theta_i}_\sharp \mu, \proj^{\Theta_i}_\sharp \nu \big ) - \SWp^p(\mu, \nu) \Bigg | \Bigg ] \lesssim \frac{4\tilde M^p_{\mu, \nu}}{\sqrt{md}},\label{EQ:MC_B_final}  
\end{equation}
where the hidden constant is universal.

We now focus on bounding $\tilde M^p_{\mu, \nu}$, leveraging log-concavity of $\mu$ and $\nu$. We present the derivation for $\sup_\theta \mu |\theta^\intercal (x - \mu x)|$; the one corresponding to $\nu$ is analogous. To control this term we use exponential concentration for 1-Lipschitz functions of log-concave random variables. Recalling that $\mu$ and $\nu$ being log-concave implies that so are $\proj^\theta_\sharp \mu$ and $\proj^\theta_\sharp \nu$, Theorem 1.2 in \cite{milman2009role} yields
\[
\PP\big(\big|\theta^\intercal X - \mu (\theta^\intercal x)\big| > t\big) \leq e \exp(-D_\mu t),
\]
where $X \sim \mu$ and $D_\mu \geq c/\sqrt{\|\Sigma_\mu\|_{\op}}$, with a universal constant $c$. Then,
\begin{align*}
    \sup_{\theta \in \unitsph} \mu |\theta^\intercal (x - \mu x)|^p &\leq \int_{0}^\infty \PP\big(\big| \theta^\intercal x - \mu(\theta^\intercal x)\big|^p > t\big) dt \\
    &\leq \int_{0}^\infty e\exp(-D_\mu t^{1/p})\,dt \\
    &= \Gamma(p+1)D_\mu^{-p} \\
    &\leq  \Gamma(p+1) \left(\frac{\sqrt{\|\Sigma_\mu\|_{\op}}}{c} \right)^p \\
    &\leq C_p \|\Sigma_\nu\|_{\op}^{p/2},
\end{align*}
for a constant $C_p$ depending only on $p$. Similarly, we obtain
\[
\sup_{\theta \in \unitsph} \nu |\theta^\intercal (x - \nu x)|^p \leq C_p \|\Sigma_\nu\|_{\op}^{p/2},
\]
which together implies
\[
\tilde M^p_{\mu, \nu} \leq C'_p \left (  \|\mu x - \nu x\|^p + \|\Sigma_\mu\|_{\op}^{p/2} + \|\Sigma_\nu\|_{\op}^{p/2}  \right ).
\]
Inserting the above bound into \eqref{EQ:MC_B_final} and combining with \eqref{EQ:MC_A_final} yields the result.

\subsection{Proof of Proposition \ref{PROP:MSWP_subgrad} }
\label{APPEN:MSWP_subgrad_proof}
Observe that $\tilde w_2^2(\theta)$ is  $M_n$-Lipschitz by Lemma ~\ref{lem: wp_theta_lipschitz} and $\rho_n$-weakly convex by Lemma 2.2 in \cite{lin2020projection}, where $M_n = 4 \sup_\theta (\empmu |\theta^\intercal x|^2 + \empnu |\theta^\intercal x|^2)$ and $\rho_n = 2 \max_{i,j} \|X_i - Y_j\|^2$.
By equation (2.10) in \cite{davis2018stochastic}, there exists a choice of step sizes $\alpha_t = \frac{c_{\rho_n, M_n}}{\sqrt{t+1}}$, such that Algorithm~\ref{alg:subgradient} for the objective $\varphi(\theta) = \tilde w_2^2+ \delta_{\BB^d}$, where $\delta_{\BB^d} = \infty \ind_{(\BB^d)^c}$, outputs a point $\theta_{t^*}$ that is close to a near-stationary point $\theta^*$, in the sense that $\EE_{t^*}[\|\theta^*-\theta_{t^*}\|]\leq \frac{\epsilon}{2\rho_n}$ and $\mathrm{dist}\big(0,\partial \tilde{w}_2^2(\theta^*)\big)\leq \epsilon$, in number of steps 
\[
T \leq \left \lceil \frac{64 \rho_n^2 M_n^2 \big (1 \wedge \frac{M_n}{2\rho_n}\big )}{\epsilon^4}  \right \rceil.
\]

We derive high probability upper bounds on $M_n$ and $\rho_n$ to obtain a non-stochastic bound on the computational complexity of our algorithm.

We will first reduce our problem to the case where $\mu$ and $\nu$ are isotropic log-concave, where our assumptions will lead to concentration inequalities on the above quantities. Assume first that $\Sigma_\mu$ and $\Sigma_\nu$ have rank $d$. Let $T^\mu (x) = \Sigma_\mu^{-1/2} (x - \mu x)$, and define $T^\nu$ analogously. Then, $\tilde{\mu} = T^\mu_\sharp \mu$ and $\tilde \nu = T^\nu_\sharp \nu$ are isotropic log-concave. Let $\tilde \mu_n$ and $\tilde \nu_n$ be empirical measures corresponding to $\tilde \mu$ and $\tilde \nu$, obtained by applying $T_\mu$ and $T_\nu$ to samples from $\mu$ and $\nu$, respectively.
For the first, we have
\[
\begin{split}
&M_n \leq  M^2_{\mu, \nu} + 24\sup_{\theta \in \unitsph} |\empmu \theta^\intercal (x - \mu x)|^2 + 24 \sup_{\theta \in \unitsph} \empnu |\theta^\intercal (x - \nu x)|^2, \quad \text{and} \\
&\rho_n \leq 6 \Big ( \|\mu x - \nu x\|^2 + \|\Sigma_\mu\|_{\op} \max_i \|\Sigma_\mu^{-1/2}(X_i - \mu x)\|^2 + \|\Sigma_\nu\|_{\op} \max_j \|\Sigma_\nu^{-1/2}(Y_j - \nu x)\|^2 \Big ).
\end{split}
\]
Further, assuming that $\mu$ and $\nu$ are centered, we have
\begin{align*}
    \sup_{\theta \in \unitsph} &\left | \empmu |\theta^\intercal x|^2  - \mu |\theta^\intercal x|^2 \right | \\
    &\leq \sup_{\theta \in \unitsph} \left | \empmu |\theta_{\Sigma_\mu}^\intercal x|^2  - \mu |\theta^\intercal x|^2 \right | \qquad\qquad \left [ \theta_{\Sigma_\mu} = \frac{\Sigma_\mu^{-1/2} \theta}{\|\Sigma_\mu^{-1/2} \theta\|} \right ]\\
    &\leq \|\Sigma_\mu\|_{\op} \sup_{\theta \in \unitsph} \Big | \tilde \mu_n |\theta^\intercal x|^2  - \tilde\mu |\theta^\intercal x|^2 \Big |,
\end{align*}
and similarly,
\[
\sup_{\theta \in \unitsph} \left | \empnu |\theta^\intercal x|^2  - \nu |\theta^\intercal x|^2 \right | \leq \|\Sigma_\nu\|_{\op} \sup_{\theta \in \unitsph} \Big | \tilde \nu_n |\theta^\intercal x|^2  - \tilde\nu |\theta^\intercal x|^2 \Big |.
\]

For isotropic $\tilde \mu$ and $\tilde \nu$, we have (cf. Theorem 4.2 in \cite{adamczak2010quantitative})
\[
\PP \left ( \sup_{\theta \in \unitsph}  \Big |  \tilde \mu_n |\theta^\intercal x|^2  - \tilde \mu |\theta^\intercal x|^2 \Big | \leq \epsilon  \right ) \geq 1 - \exp \left ( - c n^{1/4} \epsilon \sqrt{d} \right ).
\]
Choosing $\epsilon = 1/c$ above and noting that $\tilde\mu |\theta^\intercal x|^2 = \tilde\nu |\theta^\intercal x|^2 = 1 $, we have
\be
M_n \leq M_{\mu, \nu}^2 + 4(1 + 1/c) \Big(\|\Sigma_\mu\|_{\op} + \|\Sigma_\nu\|_{\op}\Big)
\label{eq: wpp_lipschitz_bound}
\ee
with probability at least $1-\frac{2}{n}$.
Additionally, by Lemma 3.1 in \cite{adamczak2010quantitative}, if $d \geq \big (\log n \big )^2$, there exists a universal constant $C > 0$ such that
\[
\max \left \{\max_i \|\Sigma_\mu^{-1/2}(X_i - \mu x)\|^2, \max_i \|\Sigma_\mu^{-1/2}(X_i - \mu x)\|^2 \right \} \leq C d, \]
implying
\be
\label{eq: wpp_convex_bound}
\rho_n \leq 6 \Big ( \|\mu x - \nu x\|^2 + Cd \left ( \|\Sigma_\mu\|_{\op} + \|\Sigma_\nu\|_{\op} \right ) \Big )
\ee
with probability at least $1 - \frac{2}{n}$.

If $\Sigma_\mu$ and $\Sigma_\nu$ are not full rank, then the above results hold for $\mu \ast \Unif(B_d(0,\sigma))$ and $\nu \ast \Unif (B_d(0,\sigma))$ instead, which are log-concave measures with covariance matrices $\Sigma_\mu + \sigma^2 I_d/(d+1)$ and $\Sigma_\nu + \sigma^2 I_d/(d+1)$, respectively. Letting $M_n^\sigma$, $\rho_n^\sigma$ and $M_{\mu, \nu}^{2,\sigma}$ denote $M_n$, $\rho_n$ and $M_{\mu, \nu}^2$ for these perturbed measures, we observe that $|\rho_n - \rho| \leq \sigma^2$, $|M_n^\sigma - M_n| \leq 96 \sigma^2$, and $|M_{\mu, \nu}^{2,\sigma} - M_{\mu, \nu}^2| \leq 48 \sigma^2$. Choosing $\sigma^2 = \|\Sigma_\mu\|_{\op} + \|\Sigma_\nu\|_{\op}$, we see that \eqref{eq: wpp_lipschitz_bound} and \eqref{eq: wpp_convex_bound} hold for non-full dimensional $\mu$ and $\nu$ as well with adjustments to $c$ and $C$. 

Combining \eqref{eq: wpp_lipschitz_bound} and \eqref{eq: wpp_convex_bound}, and noting that sorting to obtain the optimal permutation $\sigma^*$ and computing the subdifferential $\partial \rho(\sigma^*, \theta) = \nabla_\theta \rho(\sigma^*, \theta) $ takes $O(n \log n)$ operations, we have the result. \qed

\subsection{Proof of Proposition \ref{prop: MSWp_LIPO_convergence}}
\label{subsection: MSWp_LIPO_convergence_proof}
By Lemma~\ref{lem: wp_theta_lipschitz}, we have that $\hat w_p(\theta)$ is $\hat L_n$-Lipschitz with $\hat L_n := \sup_{\theta \in \unitsph} \big \{ ( \empmu |\theta^\intercal x|^p )^{1/p} + ( \empnu |\theta^\intercal x|^p )^{1/p} \big \}$. This yields, via the LIPO convergence guarantee (Corollary 13 in \cite{malherbe2017global}), that
\be
\label{eq: LIPO_bound_emp}
\max_{\theta \in \unitsph} \hat w_p(\theta) - \max_{1 \leq i \leq k} \hat w_p(\Theta_i) \leq 2 \hat L_n \left ( \frac{\log(1/\delta)}{k} \right )^{1/d}
\ee
with probability at least $1 - \delta$.

As in the previous section, we will first reduce our problem to the case where $\mu$ and $\nu$ are isotropic log-concave. For $1 \leq p \leq 2$, $\hat L_n \leq \sup_{\theta \in \unitsph} \big \{ ( \empmu |\theta^\intercal x|^2 )^{1/2} + ( \empnu |\theta^\intercal x|^2 )^{1/2} \big \}$, so that it suffices to bound $\hat L_n$ for $p \geq 2$. We have
\begin{align*}
    \hat L_n &= \sup_{\theta \in \unitsph} \big \{ ( \empmu |\theta^\intercal x|^p )^{1/p} + ( \empnu |\theta^\intercal x|^p )^{1/p} \big \} \\
    &\leq \sup_{\theta \in \unitsph} \big \{ ( \empmu |\theta^\intercal (x - \mu x)|^p )^{1/p} + ( \empnu |\theta^\intercal (x - \nu x)|^p )^{1/p} \big \} \\
    &\qquad+ \sup_\theta |\mu (\theta^\intercal x)|  + \sup_\theta |\nu (\theta^\intercal x)|\\
    &\leq \|\Sigma_\mu\|_{\op}^{1/2} \left [ \sup_{\theta \in \unitsph} (\tilde{\mu} |\theta^\intercal x|^p)^{1/p} + \sup_{\theta \in \unitsph} \left | (\tilde \mu_n |\theta^\intercal x|^p )^{1/p} - (|\tilde \mu (\theta^\intercal x)|^p)^{1/p} \right |  \right ]\\
    &\qquad + \|\Sigma_\nu\|_{\op}^{1/2} \left [ \sup_{\theta \in\unitsph} (\tilde{\mu} |\theta^\intercal x|^p)^{1/p} +  \sup_{\theta \in \unitsph} \left | (\tilde \nu_n |\theta^\intercal x|^p )^{1/p} - (|\tilde\nu (\theta^\intercal x)|^p)^{1/p} \right | \right ]\\
    &\qquad + \sup_\theta |\mu (\theta^\intercal x)|  + \sup_\theta |\nu (\theta^\intercal x)|\\
    &\leq \|\Sigma_\mu\|_{\op}^{1/2} \left [ (2p)^{1/p}  + \sup_{\theta \in \unitsph} \left | (\tilde \mu_n |\theta^\intercal x|^p )^{1/p} - (|\tilde\mu (\theta^\intercal x)|^p)^{1/p} \right |  \right ]\\
    &\qquad + \|\Sigma_\nu\|_{\op}^{1/2} \left [ (2p)^{1/p}  +  \sup_{\theta \in \unitsph} \left | (\tilde \nu_n |\theta^\intercal x|^p )^{1/p} - (|\tilde\nu (\theta^\intercal x)|^p)^{1/p} \right | \right ]\\
    &\qquad + \sup_\theta |\mu (\theta^\intercal x)|  + \sup_\theta |\nu (\theta^\intercal x)|
\end{align*}

Now, applying \cite[Theorem 4.2]{adamczak2010quantitative} with $\epsilon = 1/2$ and $t = 1$, we get
\begin{align*}
    \PP \left ( \sup_{\theta \in \unitsph} \left | (\tilde \mu_n |\theta^\intercal x|^p )^{1/p} - (|\tilde\mu (\theta^\intercal x)|^p)^{1/p} \right | > \frac{1}{2} \right ) &\leq \PP \left ( \sup_{\theta \in \unitsph} \left | \tilde \mu_n |\theta^\intercal x|^p  - |\tilde\mu (\theta^\intercal x)|^p \right | > \frac{1}{2^p} \right )\\
    &\leq 1 - e^{-c_p  \sqrt{d}} \numberthis\label{eq: L_n concentration}
\end{align*}
under assumed constraints on $n$ in the statement. An analogous bound holds for $\nu$, which yields that
\[
\PP \left ( \hat L_n \geq (\|\Sigma_\mu\|_{\op}^{1/2} + \|\Sigma_\nu\|_{\op}^{1/2}) \left ( (2p)^{1/p}  + \frac{1}{2}  \right ) + \sup_\theta |\mu (\theta^\intercal x)|  + \sup_\theta |\nu (\theta^\intercal x)| \right ) \leq e^{-c_p \sqrt{d}}
\]
Recall that $\beta = \exp(-c_p \sqrt{d}).$
Plugging \eqref{eq: L_n concentration} back into \eqref{eq: LIPO_bound_emp}, we get
\be
\label{eq: LIPO_bound_pop}
\max_{\theta \in \unitsph} \hat w_p(\theta) - \max_{1 \leq i \leq k} \hat w_p(\Theta_i) \leq L_{\mu,\nu} \left ( \frac{\log(1/\delta)}{k} \right )^{1/d}
\ee
with probability $1 - \delta - \beta$.

Finally, we have $\max_{\theta \in \unitsph} \hat w_p(\theta) = \MSWp(\empmu, \empnu)$, and
\[
|\MSWp(\empmu, \empnu) - \MSWp(\mu, \nu)| \leq \MSWp(\empmu, \mu) + \MSWp(\empnu, \nu).
\]
By \eqref{eq: MSWp concentration} in Proposition~\ref{prop: SWp concentration}, for any $t > 0$,
\[
\PP \Bigg ( \MSWp(\empmu, \mu) \geq \alpha_{n,\mu} + t \Bigg ) \leq 2 \exp \left ( -K_\mu \min \left ( n^{1/p} t, n^{2/(2 \vee p)} t^2 \right ) \right ),
\]
\[
\PP \Bigg ( \MSWp(\empnu, \nu) \geq \alpha_{n,\nu} + t \Bigg ) \leq 2 \exp \left ( -K_\nu \min \left ( n^{1/p} t, n^{2/(2 \vee p)} t^2 \right ) \right ),
\]
where $K_\mu \lesssim d^{o_d(1)} \max\{\|\Sigma_\mu\|_{\op}^{1/2},\,\|\Sigma_\mu\|_{\op} \}$ and $K_\nu \lesssim d^{o_d(1)} \max\{\|\Sigma_\nu\|_{\op}^{1/2},\,\|\Sigma_\nu \|_{\op} \}$. Setting
\[
\gamma_n(t) = 2 \exp \left ( -K_\mu^{-1} \min \left ( n^{1/p} t, n^{2/(2 \vee p)} t^2 \right ) \right ) + 2 \exp \left ( -K_\nu^{-1} \min \left ( n^{1/p} t, n^{2/(2 \vee p)} t^2 \right ) \right ),
\]
we then have
\[
\PP \left ( |\MSWp(\empmu, \empnu) - \MSWp(\mu, \nu)| > \alpha_{n,\mu} + \alpha_{n,\nu} + 2t \right ) \leq \gamma_n(t).
\]
Combining the above display with \eqref{eq: LIPO_bound_pop}, we get the desired result.

\section{Additional Experiments and Details}
\label{app:experiments}
Code for reproducing this paper's experiments can be found at \url{https://github.com/sbnietert/sliced-Wp}. Distance computations and plots for Figure \ref{fig:sw2_cpx} were performed on a cluster machine with 8 CPU cores and 64GB RAM in approximately 6 hours. Distance computations and plots for Figures  \ref{fig:MSW2_computation} and \ref{fig:robust-estimation} were performed on a cluster machine with 4 CPU cores and 20GB RAM in approximately 30 minutes. For Figure \ref{fig:robust-estimation} (right), the lower bound on $\Wone$ is computed by only considering couplings which leave the shared mass at 0 unmoved.

\begin{figure}[b]
\centering
\includegraphics[width=0.4\textwidth]{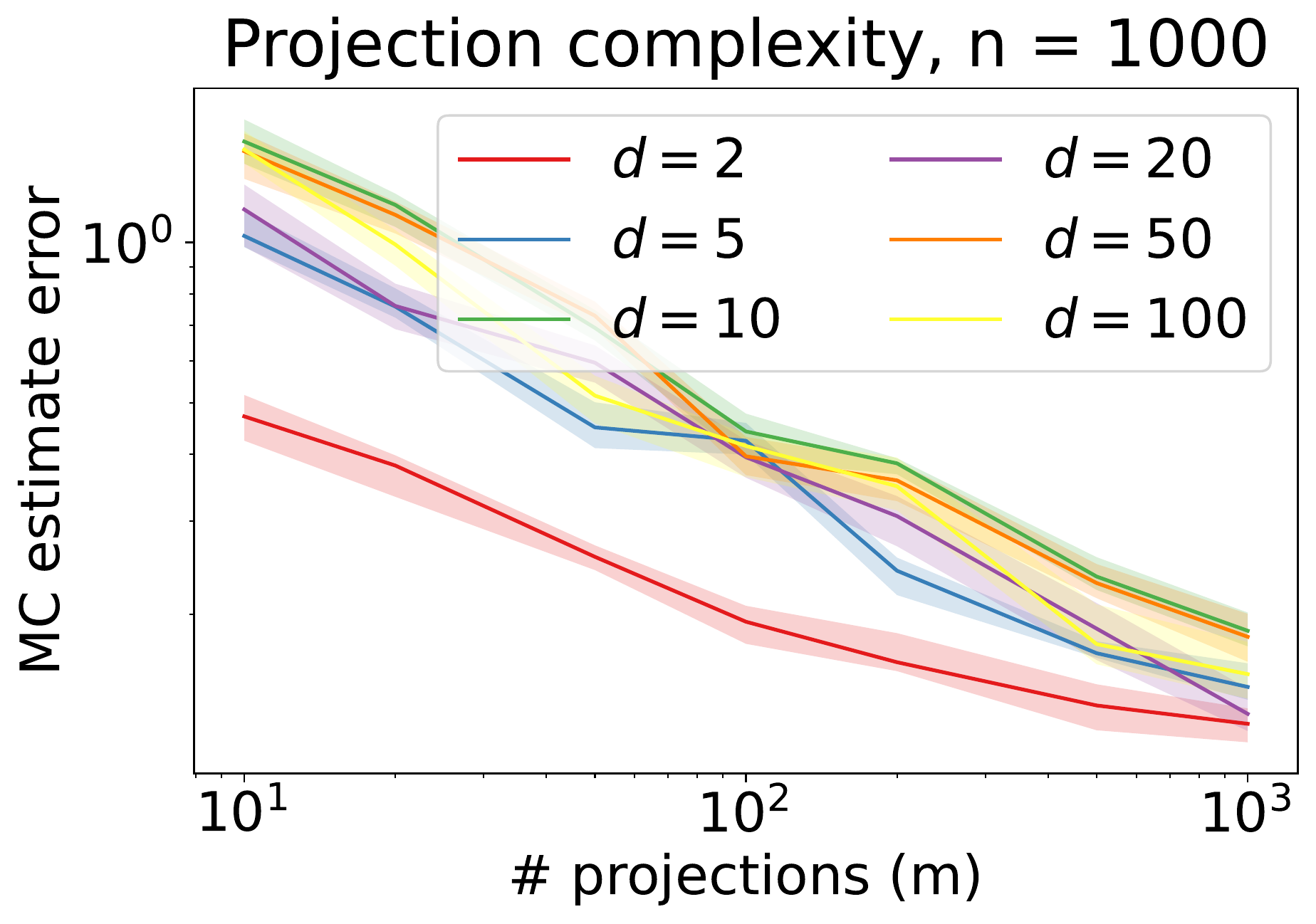}
\includegraphics[width = 0.41\textwidth]{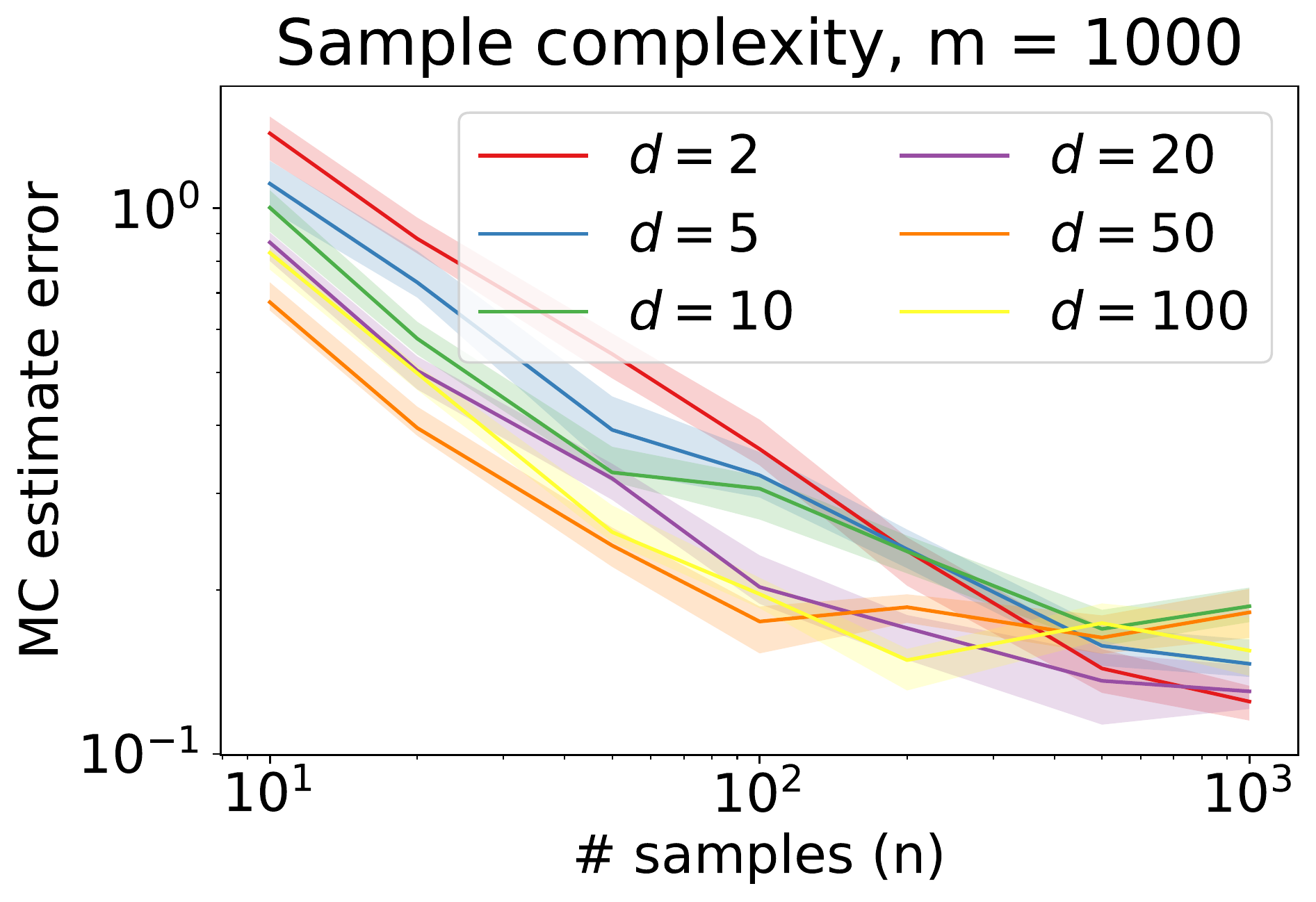}
\caption{$\big |\WMCtwo - \SWtwo^2(\mu, \nu) \big |$ under Model (3).}\label{fig:swp-exp3}
\end{figure}

As an additional experimental setup along the lines of Figure \ref{fig:sw2_cpx}, we consider Model (3): Gaussian mixtures $\mu = \frac{1}{10} \sum_{i=1}^{10} \cN(\mu_{1,i}, \Sigma_{1,i})$ and $\nu = \frac{1}{10}\cN(\mu_{2,i}, \Sigma_{2,i})$, where means $\mu_{1,i}$ and $\mu_{2, i}$ are respectively generated from $\cN(\bm{1}_d, I_d)$ and $\cN(3\,\bm{1}_d, I_d)$, and the covariance matrices of the mixtures are simulated as $\frac{1}{k} X^{\intercal} X$, where $X$ is $k \times d$ data matrix generated from $\cN(0, I_d)$ and $k$ is a uniformly sampled integer from $1$ to $d$. Conditioned on fixed random choices of $\mu$ and $\nu$, we provide the corresponding projection and sample complexity plots in \Cref{fig:swp-exp3}, with general trends matching those of \Cref{fig:sw2_cpx}. For both this experiment and Figure \ref{fig:sw2_cpx} in the main text, the population versions of the distances where no closed forms exist were calculated by setting the number of samples and Monte Carlo directions to 5000 and 2000 respectively. Computations and plots were performed on a cluster machine with 8 CPU cores and 64GB RAM in approximately 12 hours.

Finally, we consider how the robustness properties of sliced $\Wp$ may impact its application to generative modeling. Minimum distance estimation with respect to classic $\Wone$ serves as a theoretical foundation for Wasserstein GANs \cite{arjovsky_wgan_2017, gulrajani2017improved}, a successful approach for training generative models.
Later work extended this approach to average and max-sliced $\Wp$ \cite{deshpande2018generative, deshpande2019max}, albeit at a slightly less direct level (in these papers, sliced distances are computed in a feature embedding space rather than raw image space). 
\begin{wrapfigure}{r}{0.5\textwidth}
\centering
\begin{subfigure}[b]{0.245\textwidth}
    \centering
    \includegraphics[scale=0.27]{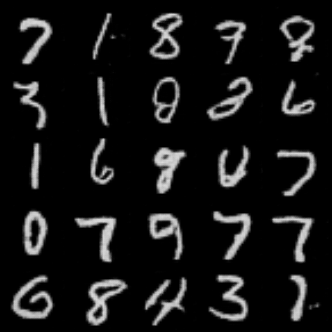}\\
    \vspace{1mm}
    \includegraphics[scale=0.27]{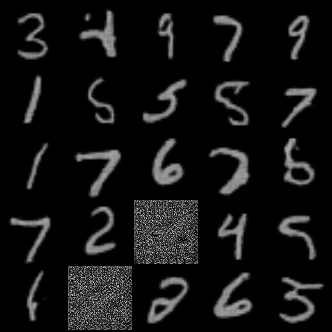}
    \caption{Sliced WGAN}
\end{subfigure}
\begin{subfigure}[b]{0.245\textwidth}
    \centering
    \includegraphics[scale=0.59]{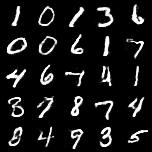}\\
    \vspace{1mm}
    \includegraphics[scale=0.59]{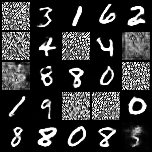}
    \caption{WGAN-GP}
\end{subfigure}
\caption{Preliminary GAN experiments with uncontaminated (top) vs. contaminated (bottom) MNIST data.}\label{fig:gans}
\end{wrapfigure}
In \Cref{fig:gans}, we display samples generated from open source implementations of the standard Wasserstein GAN with Gradient Penalty (WGAN-GP) \cite{gulrajani2017improved} and an average-sliced WGAN \cite{deshpande2018generative} trained for 20 epochs over the MNIST dataset \cite{deng2012mnist} of digit images with 10\% random noise contamination, using default parameter settings. Computations were performed on a cluster machine with 4 CPU cores, a NVIDIA Tesla T4 GPU, and 20GB RAM in roughly 12 hours. While there are differences between the produced samples, the two GAN architectures seem too distinct to draw any strong conclusions. Moreover, the robustness guarantees from \Cref{sec:robust-estimation} hold after preprocessing that appears too expensive to perform for data of this scale, so it is not surprising that the sliced WGAN reproduces random noise. Translating methods and guarantees for standard WGAN robustification (e.g., \cite{nietert2022robust}) to the sliced setting and thorough empirical comparisons are an interesting avenue for future research beyond the scope of this paper.

\end{document}